\documentclass[letterpaper,11pt]{article}
\usepackage[margin=1in]{geometry}  % set the margins to 1in on all sides

\usepackage{amssymb,amsfonts,amsmath,amstext,amsthm,mathrsfs}
\usepackage{graphics,latexsym,epsfig,psfrag,wrapfig,cite,comment,paralist}
\usepackage{xspace}
\usepackage{subcaption,bm,multirow}
\usepackage{booktabs}
\usepackage{cite,mathdots}

\usepackage{algorithm}
\usepackage{prettyref}
\usepackage{listings}
\usepackage{amsmath}
\usepackage{amsthm}
\usepackage{tikz}
\usepackage{caption}
\usepackage{array}
\usepackage{mdwmath}
\usepackage{multirow}
\usepackage{mdwtab}
\usepackage{eqparbox}
\usepackage{multicol}
\usepackage{amsfonts}
\usepackage{tikz}
\usepackage{multirow,bigstrut,threeparttable}
\usepackage{amsthm}
\usepackage{array}
\usepackage{bbm}
\usepackage{epstopdf}
\usepackage{mdwmath}
\usepackage{mdwtab}
\usepackage{eqparbox}
\usepackage{tikz}
\usepackage{latexsym}
\usepackage{cite}
\usepackage{amssymb}
\usepackage{bm}
\usepackage{graphicx}
\usepackage{mathrsfs}
\usepackage{epsfig}
\usepackage{psfrag}
\usepackage{setspace}
\usepackage[CJKbookmarks=true,
            bookmarksnumbered=true,
            bookmarksopen=true,
            colorlinks=true,
            citecolor=red,
            linkcolor=blue,
            anchorcolor=red,
            urlcolor=blue
            ]{hyperref}

\newtheorem{definition}{\textbf{Definition}}

\newtheorem{lemma}{\textbf{Lemma}}
\newtheorem{theorem}{\textbf{Theorem}}

\newtheorem{remark}{\textbf{Remark}}

\newtheorem{principle}{\textbf{Principle}}

\newcommand{\tW}{\tilde{W}}
\newcommand{\bbP}{\mathbb{P}}
\newcommand{\bbQ}{\mathbb{Q}}
\newcommand{\bbR}{\mathbb{R}}

\newcommand{\bbE}{\mathbb{E}}

\newcommand{\cN}{\mathcal{N}}

\newcommand{\cG}{\mathcal{G}}

\newcommand{\cF}{\mathcal{F}}

\DeclareMathOperator*{\argmax}{arg\,max}
\DeclareMathOperator*{\argmin}{arg\,min}

\usepackage{color}

\title{Deconstructing Generative Adversarial Networks}
\author{Banghua Zhu, Jiantao Jiao, David Tse\thanks{Banghua Zhu and Jiantao Jiao are with the Department of Electrical Engineering and Computer Sciences, University of California, Berkeley. Email: \{banghua, jiantao\}@berkeley.edu. David Tse is with the Department of Electrical Engineering, Stanford University. Email: dntse@stanford.edu.}}
\date{\today}

\begin{document}

\maketitle

\begin{abstract}

Generative Adversarial Networks (GANs) are a thriving unsupervised machine learning technique that has led to significant advances in various fields such as computer vision, natural language processing, among others. However, GANs are known to be difficult to train and usually suffer from mode collapse and the discriminator winning problem. To interpret the empirical observations of GANs and design better ones, we deconstruct the study of GANs into three components and make the following contributions. 
\begin{itemize}
    \item Formulation: we propose a perturbation view of the population target of GANs. Building on this interpretation, we show that GANs can be connected to the robust statistics framework, and propose a novel GAN architecture, termed as \emph{Cascade GANs}, to provably recover meaningful low-dimensional generator approximations when the real distribution is high-dimensional and corrupted by outliers. 
    \item  Generalization: given a population target of GANs, we design a systematic principle, {\it projection under admissible distance}, to design GANs to meet the population requirement using only finite samples. We implement our principle in three cases to achieve polynomial and sometimes near-optimal sample complexities: (1) learning an arbitrary generator under an arbitrary pseudonorm; (2) learning a Gaussian location family under total variation distance, where we utilize our principle to provide a new proof for the near-optimality of the Tukey median viewed as GANs; (3) learning a low-dimensional Gaussian approximation of a high-dimensional arbitrary distribution under Wasserstein distance. We demonstrate a fundamental trade-off in the approximation error and statistical error in GANs, and demonstrate how to apply our principle in practice with only empirical samples to \emph{predict} how many samples would be sufficient for GANs in order not to suffer from the discriminator winning problem.
    \item Optimization: we demonstrate alternating gradient descent is provably not locally asymptotically stable in optimizing the GAN formulation of PCA. We diagnose the problem as the minimax duality gap being non-zero, and propose a new GAN architecture whose duality gap is zero, where the value of the game is equal to the previous minimax value (not the maximin value). We prove the new GAN architecture is globally asymptotically stable in solving PCA under alternating gradient descent.  
\end{itemize}
 
%From the three aspects, we provide sharp and powerful tools to diagnose GANs with issues such as mode collapse and discriminator winning. We provide several indications -- such as statistical insights, novel methodologies, and initial experimental results -- to suggest that our ultimate goals of improving the performances of GANs in a principled, versatile, and easy-to-implement manner in a variety of application domains.% are achievable within the time frame of this project. 

%a principled approach to design and analyze GANs from the generalization points of view, and address the fundamental generalization and optimization issues of GANs. 

%We focus on three aspects: (1) a general formulation for clarifying the target of GAN;(2) a principled method to alleviate mode collapse and achieve statistical optimality in the design of GANs through \emph{weakening} the distance; (3) a versatile approach to improve the architecture of GANs to stabilize training. 
%Our contribution is four-fold: (1) develop a theoretical framework for analyzing the generalization properties of GANs in high-dimensions; (2) suggest principled approaches to design GANs to achieve optimal statistical properties; (3) diagnose GANs when issues such as mode collapse or discriminator winning occur; (4) provide analysis on the optimization property of GANs. 

\end{abstract}
\tableofcontents
\newpage

\section{Introduction}

\label{sec:intro}

%In the modern era of big data, although the cost of labeling and analyzing a single data sample has been decreasing rapidly, it is usually outpaced by the unrivaled fast growth of dataset size, which makes it particularly timely to design unsupervised learning algorithms that are able to discover meaningful structures of data without extensive human efforts. 

Unsupervised learning has been studied extensively in the literature. It is intimately related to the problem of distribution (density) estimation in statistics~\cite{silverman2018density}, and dimensionality reduction~\cite{van2009dimensionality}. Recently, Generative Adversarial Networks (GANs) become a thriving unsupervised machine learning technique that has led to significant advances in various fields such as computer vision, natural language processing, among others~\cite{brock2018large, young2018recent}. In computer vision and natural language processing, GANs have resulted in superior empirical performance over standard generative models for images and texts, such as variational autoencoder and deep Boltzmann machine~\cite{goodfellow2014generative}. Various ideas have been proposed to further improve the quality of the learned distribution and the stability of the training~\cite{arjovsky2017wasserstein, durugkar2016generative, huang2017stacked, im2016generative, odena2016conditional, radford2015unsupervised, salimans2016improved, tolstikhin2017adagan, xu2017attngan, bai2018approximability, feizi2017understanding, DBLP:journals/corr/abs-1811-12402, sanjabi2018solving}. Given the success of GANs, there also exist challenges that need timely solutions. In particular, GAN training has been reportedly observed to be challenging, unstable, and the problems of \emph{mode collapse} and \emph{discriminator winning} frequently appear~\cite{che2016mode, salimans2016improved, neyshabur2017stabilizing,arora2017generalization, arora2018gans, dumoulin2016adversarially}. To understand the empirical observations in GAN training and propose theoretically near-optimal GANs algorithms, we \emph{deconstruct} the design of GANs into three fundamental aspects: {\bf formulation}, {\bf generalization} and {\bf optimization}, and develop systematic approaches to analyze them.

\subsection{Main contributions}

Our main contributions in this paper can be summarized as follows.

\begin{itemize}
  \item \textbf{Formulation}: we propose a perturbation view in designing the population target of GANs. Building on this interpretation, we show that GANs can be connected to the robust statistics framework in terms of perturbation beyond the total variation distance. We also propose a novel GAN architecture, termed as \emph{Cascade GANs}, to provably recover meaningful low-dimensional generator approximations when the real distribution is high-dimensional and corrupted by outliers. We also elaborate on the implicit assumptions used in~\cite{arjovsky2017wasserstein} that motivated the Wasserstein GAN. 
  \item \textbf{Generalization}: given a population target of GANs, we design a systematic principle, \emph{projection under admissible distance}, to construct GAN algorithms that achieve polynomial and sometimes near-optimal sample complexities given only finite samples. We study three cases in detail: 
  \begin{enumerate}
      \item learning an arbitrary generator under an arbitrary pseudonorm;
      \item learning Gaussian location family under total variation distance, where we provide a new proof for the near-optimality of Tukey median viewed as GANs;
      \item learning low-dimensional Gaussian approximations of a high-dimensional arbitrary distribution under Wasserstein-2 distance. 
  \end{enumerate}
  We demonstrate a fundamental trade-off in the approximation error and statistical error in GANs, and demonstrate how to apply our principle in practice with only empirical samples to \emph{predict} how many samples would be sufficient for GANs in order not to suffer from the discriminator winning problem.
  \item \textbf{Optimization}: we demonstrate alternating gradient descent is provably not \emph{locally} \emph{asymptotically} stable in optimizing the GAN formulation of PCA. We diagnose the problem as the minimax duality gap being non-zero, and propose a new GAN architecture whose duality gap is zero, and the value of the game is equal to the previous minimax value (not the maximin value). We prove the new GAN architecture is \emph{globally asymptotically} stable in solving PCA under alternating gradient descent.  
\end{itemize}

\medskip		
\noindent {\bf Organization of paper.}
The rest of the paper is organized as follows. In Section \ref{sec:formulation}, we discuss the formulation of GANs from a perturbation view, and propose \emph{Cascade GANs} for capturing the sequential perturbation model.   In Section \ref{sec:generalization}, we study the generalization property of GANs. In Section \ref{sec:computation}, we study the optimization property of GANs. %In Section \ref{sec:related_work}, other related works on formulation and training are discussed. %And in Section \ref{sec:discussion}, the formulation of GANs and the relationship between three aspects and mode collapse are discussed.
All the proofs are deferred to the appendices.

\medskip		
\noindent {\bf Notations.}
Throughout this paper, we use capital letter $X$ for random variable, blackbold letter $\bbP$ for probability distribution. Generally, the distribution for random variable $X$ is denoted as $\bbP_X$, and the corresponding empirical distribution from $\bbP_X$ with $n$ samples is $\hat{\bbP}_X^n$. In all sections except for Section~\ref{sec:computation}, we use bold lower case letters for vectors, bold upper case letters for matrices. In Section~\ref{sec:computation}, we use lower case letters for vectors, upper case letters for matrices. We use $\mathbf{A}^\dagger$ to denote the pseudo inverse of matrix $\mathbf{A}$. Random vector $Z$ follows zero mean normal distribution with identity covariance. We use $r(\mathbf{A})$ to denote the rank of matrix $\mathbf{A}$, $R(\mathbf{A})$ to denote the range of matrix $\mathbf{A}$. We say $\mathbf{A} \geq 0$ if $\mathbf{A}$ is positive semidefinite, $\mathbf{A} > 0$ if $\mathbf{A}$ is positive definite. We use $\mathsf{TV}(\bbP,\bbQ) = \sup_{A} \bbP(A) - \bbQ(A)$ to denote the total variation distance between $\bbP$ and $\bbQ$. 
For non-negative sequences $a_\gamma$, $b_\gamma$, we use the notation $a_\gamma \lesssim_{\alpha} b_\gamma$ to denote that there exists a constant $C$
that only depends on $\alpha$ such that $\sup_\gamma \frac{a_\gamma}{b_\gamma}\leq C$, and $a_\gamma \gtrsim_{\alpha} b_\gamma$  is equivalent to $b_\gamma \lesssim_{\alpha} a_\gamma$.  When the constant $C$ is universal
we do not write subscripts for $\lesssim$ and $\gtrsim$. Notation $a_\gamma\asymp b_\gamma$ is equivalent to $a_\gamma \gtrsim b_\gamma$ and $a_\gamma \lesssim b_\gamma$. Notation $a_\gamma \gg b_\gamma$ means that $\liminf_{\gamma}\frac{a_\gamma}{b_\gamma}=\infty$, and $a_\gamma \ll b_\gamma$ is equivalent to $b_\gamma \gg a_\gamma$. 
%XXX [need edits, in the whole paper we did not follow completely this, also there are new stuff not introduced here]

\section{Formulation of GANs}\label{sec:formulation}

GANs aim at selecting a generator $g(\cdot)$ from some generator family $\cG$. The generator function takes Gaussian random noise $Z \sim \mathcal{N}(\mathbf{0},\mathbf{I}_r)$ as input and outputs a random variable which may lie in the space of images. Given \emph{infinite} samples from the real distribution $\bbP_X$, the goal of GANs can be understood intuitively as finding a generator $g$ whose distribution $\bbP_{g(Z)}$ is the \emph{closest} to the real distribution $\bbP_X$. At the minimum we would need a distance\footnote{This distance needs to be broadly understood as a function that may not satisfy the axioms of a distance function. } $L(\bbP_{g({Z})}, \bbP_X)$ for every $g\in \mathcal{G}$. Here $L(\cdot,\cdot): \mathcal{M}(\bbR^d)\times \mathcal{M}(\bbR^d) \mapsto \bbR_{\geq 0}$ is a function that characterizes the distance between distributions. Naturally with \emph{infinite} samples and \emph{unbounded} computation we would ideally like to solve for
\begin{equation}
\label{equ:formulation}
   {\hat g} = \argmin_{g\in\cG}L(\bbP_{g({Z})}, \bbP_X). 
\end{equation}

Note that the minimum in~(\ref{equ:formulation}) may be not achievable in general. However, we assume throughout this paper that it is attained, since the other case can be dealt with by a standard limiting argument. The formulation problem looks at the design of $L$ and $\mathcal{G}$. As a concrete example, the vanilla GAN~\cite{goodfellow2014generative} defines the distance function $L$ as the Jensen-Shannon divergence, and uses neural network as generator family. 

\subsection{A perturbation view}

Intuitively, at minimum the design of $L$ and $\mathcal{G}$ needs to satisfy that \textbf{
\begin{align}
    \inf_{g\in \mathcal{G}} L(\bbP_{g({Z})}, \bbP_X)\text{ is ``small''.}
\end{align}}
Indeed, were it not true, then one should certainly design another generator family $\mathcal{G}$ that can better approximate the real distribution $\bbP_X$. In other words, we implicitly assume \textbf{
\begin{align*}
    \text{The real distribution }\bbP_X\text{ is a ``slightly'' perturbed generator }g\in \mathcal{G} \text{ under } L. 
\end{align*} }
The perturbation view allows us to unify the field of robust statistics and GANs. In robust statistics, it is usually assumed that the observed distribution (in GANs, the true distribution $\bbP_X$) is a \emph{slightly perturbed} version of a generator member under \emph{total variation} distance. Hence, we can connect GANs to the robust statistics framework\footnote{For GAN, the perturbation metric and the metric to recover original distribution are the same, but for general robust statistics problem they might be different. The latter problem is beyond the scope of the paper.}. This viewpoint proves to be beneficial: indeed, we would later utilize a general analysis technique we developed for GANs to provide a new proof for the optimal statistical properties of the Tukey median in robust estimation. Our perturbation view of GANs is inspired by~\cite{gao2018robust} which constructed GAN architectures that are provably achieving the information theoretic limits in robust estimation of mean and covariances. However, it also immediately raises a conflict: it was argued in~\cite{arjovsky2017wasserstein} that the Jensen--Shannon divergence (similarly, the total variation) distance is \emph{not} supposed to be used as the distance function for GANs. How can we unify these two viewpoints? 

\subsection{Cascade GANs}

\begin{comment}

\begin{figure}[!htb]
    \centering
    \begin{subfigure}[t]{0.8\textwidth}
    \includegraphics[width=\linewidth]{robust3.PNG}
     \end{subfigure}
    \caption{ok}
    \label{fig:2dgenerator}
\end{figure}

\end{comment}

\begin{figure}[!htb]
    \centering
    \begin{subfigure}[t]{0.45\textwidth}
    \includegraphics[width=\linewidth]{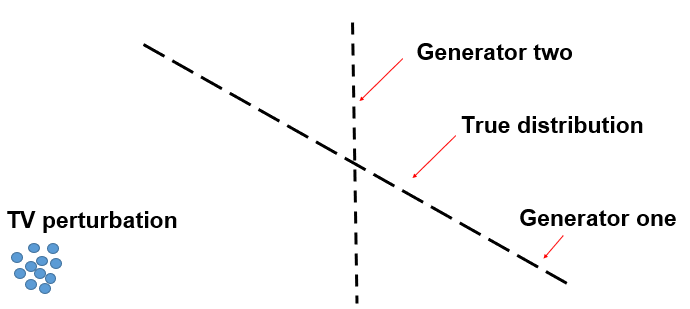}
     \end{subfigure}
         ~\begin{subfigure}[t]{0.45\textwidth}
    \includegraphics[width=\linewidth]{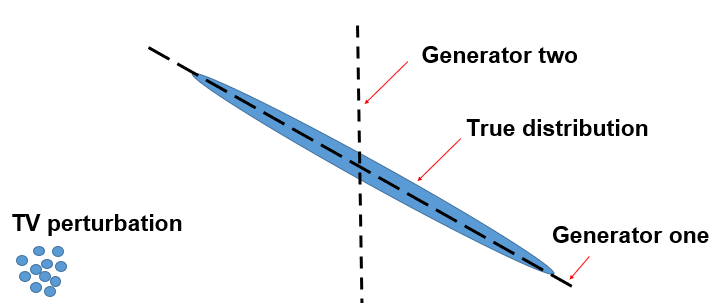}
     \end{subfigure}
    \caption{Illustration of the perturbation in the robust statistics setting, and the setting of applications of GANs. Left: the generator family is one-dimensional distributions. The real distribution is a mixture of 99\% of generator one and 1\% of the outliers marked as ``$\mathsf{TV}$ perturbation'' in the figure. Right: the generator family is one-dimensional distributions. The real distribution is a mixture of 99\% of the two-dimensional Gaussian whose support covers generator one, and 1\% of the outliers marked as ``$\mathsf{TV}$ perturbation''. Both the total variation distance and the Wasserstein distance between generator one (which intuitively is the correct generator to select in infinite samples) and the real distribution are gigantic.}
    \label{fig:twogenerator}
\end{figure}

A careful inspection of the arguments in~\cite{arjovsky2017wasserstein} shows that the total variation is inappropriate when the generator family consists of distributions with intrinsically low dimensions. Indeed, for any two low-dimensional distributions whose supports do not overlap, the total variation distance between them is the maximum value one. The Wasserstein distance is proposed in~\cite{arjovsky2017wasserstein} to be a distance that is continuous with respect to small local variations and is not sensitive to support mismatch. 

However, the usual robust statistics setting is closer to the left of Figure~\ref{fig:twogenerator}, where the observed distribution shares $99\%$ of the support of generator one, with $1\%$ outliers far away. Then, our goal is to successfully recover generator one given the corruption. 

A natural combination of the robust statistics formulation and the motivation of the Wasserstein GANs would give rise to the right of Figure~\ref{fig:twogenerator}, where neither total variation nor Wasserstein distance provides a meaningful description of the perturbation. Indeed, the generator family is the one-dimensional distributions, and the observed distribution $\bbP_X$ has $99\%$ of the points in the \emph{two}-dimensional distribution covering generator one, plus $1\%$ outliers. Ideally, the \emph{best} generator that fits this distribution should be generator one. However, both Wasserstein distance and total variation distance between generator one and the observed distribution are gigantic. This is precisely due to the fact that this perturbation is in fact a \emph{sequential} perturbation: we first perturb generator under \emph{Wasserstein distance} to obtain the two-dimensional distribution around generator one, and then we perturb it under \emph{total variation} distance to add the $1\%$ outlier. 

We naturally have the following \emph{Cascade GAN} architecture to capture this perturbation model:
\begin{align}\label{equ:cascaded}
    \min_{g \in \cG}    \min_{g' \in \cG'} \left (\mathsf{TV}(\bbP_X , \bbP_{g(Z)}) + \lambda W_2(\bbP_{g(Z)},\bbP_{g'(Z)}) \right ).
\end{align}
In the context of Figure~\ref{fig:twogenerator}, the final generator family $\mathcal{G}$ is one-dimensional distributions, and the intermediate generator family $\mathcal{G}'$ is full-dimensional distributions. By solving equation (\ref{equ:cascaded}), we are assuming that a low dimensional Gaussian distribution $\bbP_{g(Z)}$ is perturbed under $W_2$ distance to form a high dimensional Gaussian distribution $\bbP_{g'(Z)}$, and then perturbed under $\mathsf{TV}$ distance to get the observed distribution $\bbP_X$.

Generalizing the Cascade GAN architecture, we have multiple rounds of perturbation, and $\mathsf{TV}$ and $W_2$ can be replaced by any distance functions. A similar formulation was raised in~\cite{farnia2018convex}, where it was shown that sometimes Cascade GANs can be viewed as constraining the discriminator for vanilla GANs. The chained $\mathsf{TV}$ and $W_2$ perturbation problem may also be cast as a single perturbation under a new distance function $L(\bbP_X, \bbP_Y) =  \inf_\pi \bbE [\min(\|X-Y\|^2, {\bf 1}_{X\neq Y})]$. However, this would require a new design of GAN, while the Cascade GAN architecture is readily implementable given two GANs designed specifically for $\mathsf{TV}$ and $W_2$. We mention that AmbientGAN~\cite{bora2018ambientgan} is also proposed to learn under the chained perturbation model, but it requires that the second perturbation is simulatable.

%Our perturbation view is inspired by~\cite{gao2018robust}, who used GANs to derive schemes that are provably robustly information theoretically optimal under the Huber's contamination model~\cite{huber1964robust}.  

%The $d_{TV}$-$W_2$ combination is also related to the setting of robust PCA~\cite{candes2011robust} and compressive outlier sensing~\cite{carrillo2016robust}, where the data matrix is decomposed into a low rank matrix plus sparse noise perturbation. 

\subsection{Continuity with non-zero gradient}

This property requires that the $L$ we construct is neither too \emph{strong} nor too \emph{weak} for the family $\mathcal{G}$. Similar observations were made in~\cite{arjovsky2017wasserstein} that motivated the $W_1$ distance. For the sake of completeness, we illustrate this point via a different example here. Consider the left figure in Fig.~\ref{fig:onegenerator}. Here the generators are one-dimensional Gaussian distributions in high-dimensions, and we assume that generator one is the real distribution. To ensure that our GAN is able to find the correct generator one, it is natural to assume that the distance function $L(\bbP_{g(Z)}, \bbP_{X})$ is providing \emph{gradient} for us to locate the optimal generator. The middle figure of Figure~\ref{fig:onegenerator} shows the distance as a function of $\theta$ (the angle between these two one-dimensional Gaussians) under the Jensen--Shannon divergence, while the right figure in Figure~\ref{fig:onegenerator} shows that under the Wasserstein-2 distance. It is clear that the Jensen--Shannon divergence reaches zero if and only if $\theta = 0$, and then saturates for other values of $\theta$, while under $W_2$ it is a continuous function of $\theta$ with non-vanishing gradient almost everywhere. 

Intuitively speaking, the Jensen--Shannon divergence is a \emph{too strong} distance such that it always tells that the distributions are maximally different even if they are slightly apart. It is the saturation that causes the discontinuity and the gradient to vanish. We remark that constructing a too weak distance could also cause the gradient to vanish while being a continuous distance. For example, if we construct a distance like $L(\bbP,\bbQ) = \| \mathbb{E}_\bbP[X] - \mathbb{E}_\bbQ[X]\|$, then $L(\bbP_{g_1(Z)}, \bbP_{g_2(Z)})\equiv 0$ for the generator family of centered one-dimensional Gaussian distributions.

\begin{figure}[!htb]
    \centering
    \begin{subfigure}[t]{0.3\textwidth}
    \includegraphics[width=\linewidth]{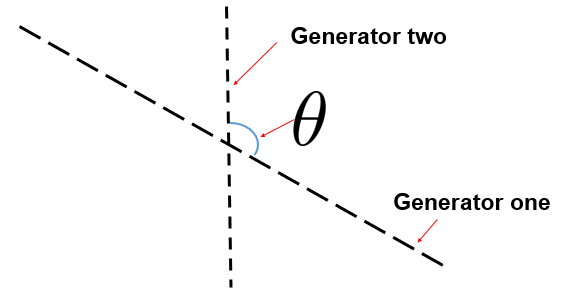}
     \end{subfigure}
     ~\begin{subfigure}[t]{0.3\textwidth}
    \includegraphics[width=\linewidth]{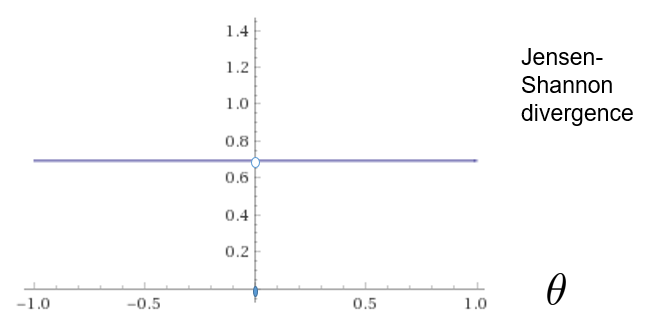}
     \end{subfigure}
       ~\begin{subfigure}[t]{0.3\textwidth}
    \includegraphics[width=\linewidth]{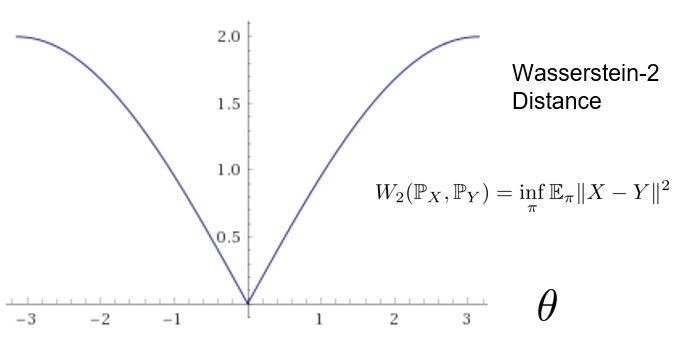}
     \end{subfigure}
    \caption{Plot of different discrepancy measures $L$ as a function of $\theta$. Left: the generator family is one-dimensional Gaussians, and we denote the angle between any generator with generator one as $\theta$. Middle: the plot of the Jensen--Shannon divergence as a function of $\theta$; Right: the plot of the $W_2$ distance as a function of $\theta$. It is clear that the Jensen--Shannon divergence reaches zero if and only if $\theta = 0$, and then saturates for other values of $\theta$, while under $W_2$ it is a continuous function of $\theta$ with non-vanishing gradient almost everywhere. }
    \label{fig:onegenerator}
\end{figure}

\section{Generalization of GANs}
\label{sec:generalization}

Given the population target of the minimization problem
\begin{align}\label{eqn.popuminimization}
    \min_{g\in \mathcal{G}} L(\bbP_{g(Z)}, \bbP_X),
\end{align} 
in this section we study the problem of \emph{approximately} solving this minimization problem given only finitely many samples from $\bbP_X$. Concretely, given $L$ and $\mathcal{G}$, we observe $n$ i.i.d. samples $X_1,X_2,\ldots,X_n \stackrel{\text{i.i.d.}}{\sim} P_X$, and produce a generator $\tilde{g}(\cdot) \in \cG$ such that with probability at least $1-\delta$, the following equation 
\begin{equation}
\label{eqn:generalization}
\mathop L(\bbP_{\tilde g(Z)}, \bbP_X) \leq C\cdot \inf_{g(\cdot)\in \mathcal{G}} L(\bbP_{g(Z)},\bbP_X) + \epsilon_n(\mathcal{G}, L, \bbP_X, \delta),
\end{equation}
holds. From now on, we define
\begin{align}\label{eqn.optdef}
    OPT = \inf_{g(\cdot)\in \mathcal{G}} L(\bbP_{g(Z)},\bbP_X)
\end{align}
as the oracle error. 

For fixed $L$ and $\mathcal{G}$, the oracle error is fixed. We focus our study on finding the best $\tilde{g}$ that minimizes both the factor $C$ on the oracle error, and also the stochastic error $\epsilon_n$ which should vanish as $n\to \infty$. Intuitively, the constant $C$ should be $1$, since when $n = \infty$ we can perform the minimization in~(\ref{eqn.popuminimization}) exactly to achieve it. Interestingly, we show that forcing $C = 1$, which is called \emph{sharp oracle inequality} in the literature~\cite{dalalyan2007aggregation}, may be fundamentally in trade-off with the stochastic error. This phenomenon is made precise in Theorem~\ref{thm:tightness}. 

We emphasize that in formulation~(\ref{eqn:generalization}), we have not specified the \emph{algorithm family} that generates  $\tilde{g}$. In this paper we mainly study $\tilde{g}$ that is produced by a GAN architecture, which proves to achieve near-optimal sample complexity in specific problems.

\subsection{The naive algorithm}

Intuitively, an approximate method to solve~(\ref{eqn.popuminimization}) is to replace the real distribution $\bbP_X$ with the empirical distribution $\hat{\bbP}_X^n$ formed with $n$ i.i.d. samples from $\bbP_X$ and perform
\begin{align}
   \min_{g\in \mathcal{G}} L(\bbP_{g(Z)}, \hat{\bbP}_X^n). 
\end{align}

It was observed in~\cite{arora2017generalization} if $L$ is the $W_1$ distance, then it requires number of samples $n$ exponential in the dimension $d$ of $X$ to ensure that $W_1(\hat{\bbP}_X^n, \bbP_X)$ vanishes. Hence, it intuitively suggests that $\hat{\bbP}_X^n$ is a \emph{bad} proxy for $\bbP_X$ for $n$ not big enough. Fig.~\ref{fig:illu_gan} provides an intuitive explanation for the poor performance of the naive projection algorithm. The poor performance occurs when the true distribution (on the boundary of the larger circle) is far from the empirical distribution (at the center), while one is \emph{able} to find another generator output distribution (on the boundary of the inner smaller circle) much closer to the center. In this case, the distance between the output generator distribution and the real distribution would be at least the difference between the radius of two circles. Completing the arguments, we show in Theorem~\ref{thm:naive_convergence_w2} that if $\cG$ is the Gaussian family and $L$ is $W_2$, indeed one \emph{can} find some distribution inside generator family that is closer to the empirical distribution but very far away from the real distribution.

\begin{figure}[hbtp!]
    \centering
    \includegraphics[width=0.4\linewidth]{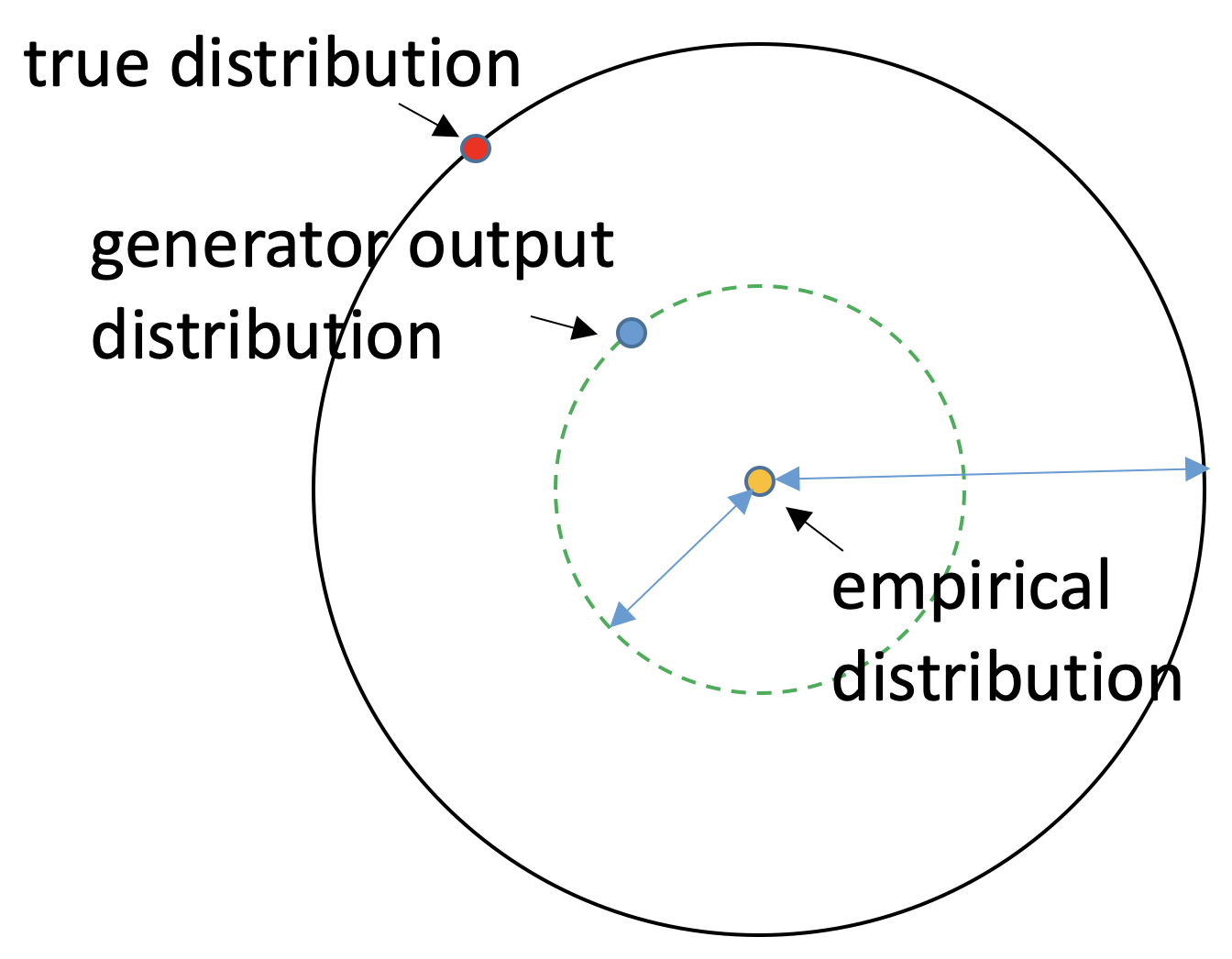}
    \caption{Illustration of poor performance of projecting the empirical distribution to the general family under distance $W_2$. The true distribution lies on the large circle, while one of the generator distributions could be much closer to empirical distribution. This leads to the large discrepancy between true distribution and the generator output distribution found by the naive algorithm.}
    \label{fig:illu_gan}
\end{figure}
\subsection{Projection under admissible distance}

\begin{comment}
Precisely, we propose to resolve this issue by designing a weakened distance to reduce the radius of the outer circle. Informally speaking, we propose to follow:
\begin{principle}\label{prin.stat}
Use the weakest\footnote{We say distance $L_1$ is weaker than distance $L_2$ if for any distribution $\bbP,\bbQ$, we have
\begin{align}
L_1(\bbP,\bbQ) \leq L_2(\bbP,\bbQ). 
\end{align}} distance that can only distinguish between different generators.
\end{principle}

By only being able to distinguish between different generators, we require the new distance $\tilde{L}$ to be very close to the original distance $L$ only when the two distributions are all inside the generator family. And we choose the \emph{weakest} distance that satisfies this condition as our weakened distance. Principle~\ref{prin.stat} may sound counter-intuitive. Indeed, we are always measuring the performances of learning through a fixed distance $L$, and we \emph{do not} assume that the true distribution $\bbP$ is a member of the generator family $\cG$. Nonetheless, projecting under the weakened distance can be shown to be optimal and achieve a significant sample complexity reduction compared to projecting under the original distance (from exponential to polynomial). 
\end{comment}

We propose a general method, termed as \emph{projection under admissible distance} to reduce the sample complexity, and for a wide range of problems to achieve polynomial, even near optimal sample complexities. Our key observation is that the population algorithm in~(\ref{eqn.popuminimization}) is \emph{not} the unique approach to achieve \emph{approximate} optimality in searching the best generator $g$, even in population. We have the following definition.  
\begin{definition}[Admissible distance]\label{def.admissible}
Suppose $L$ is a pseudometric\footnote{The definition of pseudometric is included in Appendix (\ref{appendix.def}).}. We say that the function $L'$ is \emph{admissible} with parameter $(c_1,c_2,L'')$ with respect to ($L,\mathcal{G}$) if it satisfies the following two properties:
\begin{enumerate}
    \item \textbf{Approximate resolution within generators:} for any $g_1,g_2 \in \mathcal{G}$, 
    \begin{align}
        c_1 (L'(\bbP_{g_1(Z)}, \bbP_{g_2(Z)})- L'(\bbP_{g_2(Z)}, \bbP_{g_2(Z)}))  \geq  L(\bbP_{g_1(Z)}, \bbP_{g_2(Z)}) -c_2,
    \end{align}
    \item \textbf{Robust to perturbation:} for any $g\in \mathcal{G}, \bbP_1,\bbP_2$, there exists a function $L''$ such that
    \begin{align}
        |L'(\bbP_{g(Z)}, \bbP_1) - L'(\bbP_{g(Z)}, \bbP_2)| \leq L''(\bbP_1, \bbP_2) \leq L(\bbP_1, \bbP_2),
    \end{align}
 %   and $L''$ is weaker than $L$:
 %   \begin{align}
 %       L''(\bbP_1,\bbP_2) \leq L(\bbP_1,\bbP_2). 
 %   \end{align}
\end{enumerate}
\end{definition}
 
In other words, the \emph{approximate resolution} requirement ensures that the new function $L'$ can distinguish between generators at least as well as the desired distance $L$, and the robustness property shows that observing a slightly different real distribution may not change the value of $L'$ by too much.  Of course, $L' = L$ is admissible with parameter $(1,0,L)$, but the key point of Definition~\ref{def.admissible} is that there exist a large variety of functions $L'$ that are admissible. Note that we do not require $L'$ to be a pseudometric.  %The requirement that $L''$ is weaker than $L$ guarantees that the convergence of empirical distribution to the real distribution under $L''$ happens as fast as possible. 

The next theorem shows that projecting under an admissible function achieves near optimality. 

\begin{theorem}\label{thm.populationinsights}
Suppose $L$ is a pseudometric, and $L'$ is admissible with parameter $(c_1,c_2,L'')$ with respect to $(L,\mathcal{G})$. Define
\begin{align} \label{eqn.lprojectionnewthm1}
    g' = \argmin_{g\in \mathcal{G}} L'(\bbP_{g(Z)}, \bbP_X).
\end{align}
Then, 
\begin{align}
    L(\bbP_{g'(Z)}, \bbP_X) \leq  (1+ 2c_1)OPT + c_2. 
\end{align}
Furthermore, suppose we observe $n$ i.i.d. samples from $\bbP_X$ and denote the empirical distribution formed by the samples as $\hat{\bbP}_X^n$. Define
\begin{align}
    g'' = \argmin_{g\in \mathcal{G}} L'(\bbP_{g(Z)}, \hat{\bbP}_X^n).
\end{align}
Then, 
\begin{align}
    L(\bbP_{g''(Z)}, \bbP_X) \leq  (1+2c_1)OPT + 2c_1 L''(\bbP_X,  \hat{\bbP}_X^n) + c_2 . 
\end{align}
\end{theorem}

\begin{remark}
Theorem~\ref{thm.populationinsights} shows that projecting the empirical distribution $\hat{\bbP}_X^n$ to $\mathcal{G}$ under $L'$ still achieves $O(OPT)$ approximation error, and the finite sample effect is reflected by the convergence of $L''(\bbP_X,  \hat{\bbP}_X^n)$ but not $L(\bbP_X,  \hat{\bbP}_X^n)$, which could be significantly faster. In concrete examples, we aim at minimizing $c_1$, and balancing $c_2$ with $c_1 L''(\bbP_X,  \hat{\bbP}_X^n)$.
\end{remark}

\begin{remark}
Theorem~\ref{thm.populationinsights} characterizes the performance of projecting under any admissible $L'$. Training under an alternative discrepancy measure is already observed to be beneficial in practice. It is claimed in~\cite{rainforth2018tighter} that training under a tighter evidence lower bounds can be detrimental to the process of learning an inference network by reducing the signal-to-noise ratio of the gradient estimator for variational autoencoder. In~\cite{farnia2018convex} and~\cite{liu2017approximation}, using a restricted discriminator family is shown to be related to a moment-matching effect. The recent work of~\cite{bai2018approximability} is most related to us, where a notion of ``conjoining discriminator'' is developed, while we emphasize that the results in~\cite{bai2018approximability} only applies to cases where $OPT = 0$, and the conjoining discriminator family is derived for a few specific generators with $L$ being the $W_1$ distance.
\end{remark}

The proof is deferred to Appendix (\ref{proof.populationinsights}).

\subsection{Arbitrary $\mathcal{G}$ and pseudonorm $L$} 

Suppose $L(\bbP_1,\bbP_2)$ can be written as $\| \bbP_1 - \bbP_2\|_L$, where $\| \cdot \|_L$ is a pseudonorm\footnote{The definition of pseudonorm is included in Appendix (\ref{appendix.def}).} on a vector space $\mathcal{M}$ which is a subset of signed measures on $\mathbb{R}^d$. We have the following representation of $\| \cdot \|_L$. 
\begin{lemma}\label{lemma.pseudonormrepres}
Assume appropriate topology on $\mathcal{M}$ such that any continuous linear functional of $\mu\in \mathcal{M}$ can be written as $\int fd\mu$ for some measurable $f$, and for any $\mu \in \mathcal{M}$, the linear functional $M: \alpha \mu \mapsto \alpha \|\mu\|_L$ is continuous. For any measurable function $f$ on $\mathbb{R}^d$, define
\begin{align}
    \| f\|_{L^*} = \sup_{\| \mu\|_L = 1, \mu \in \mathcal{M}} \int f d\mu.
\end{align}
Then, we can represent $\| \cdot \|_L$ as an integral probability metric:
\begin{align}
    \| \mu \|_L = \sup_{f: \|f\|_{L^*} = 1} \int f d\mu.
\end{align}
\end{lemma}

Lemma~\ref{lemma.pseudonormrepres} is a consequence of the Hahn---Banach Theorem~\cite[Theorem 4.2]{kreyszig1978introductory}. The proof is deferred to Appendix~\ref{proof.pseudonormrepres}. Note that we already know any integral probability metric is a pseudonorm.

In fact, the pseudonorm includes a wide range of discrepancy measures used in various GANs. As a concrete example, MMD-GAN~\cite{li2017mmd} minimizes the kernel maximum mean discrepancy distance, which is a pseudonorm and thus can be further studied under our framework.
Denoting

\begin{align}
\cF = \{f: \|f\|_{L^*} = 1\},     
\end{align}
we can write
\begin{align}
    L(\bbP_1, \bbP_2) = \sup_{f\in \cF} \int f d\bbP_1 - \int f d\bbP_2. 
\end{align}

We now design an admissible distance for $L$. Construct the $\epsilon$-covering of $\mathcal{G}: \mathcal{G}_\epsilon = \{g_1, g_2, \cdots, g_{N(\epsilon)}\}$ such that $\forall g \in \mathcal{G}, \exists g_i \in \mathcal{G}_\epsilon,  L(\bbP_{g_i(Z)},\bbP_{g(Z)})\leq\epsilon$. We define 
\begin{equation}
\label{equ:discriminator_family}
\mathcal{F}_\epsilon = \{f_{ij}: f_{ij} = \argmax_{f\in\mathcal{F}}\int f d(\bbP_{g_i(Z)}-\bbP_{g_j(Z)}), 1\leq i,j\leq N(\epsilon)  \}.
\end{equation}

In other words, interpreting $\cF$ as the family of discriminators, in $\mathcal{F}_{\epsilon}$ we only include those discriminators that are \textbf{optimally discriminating} between members in the $\epsilon$-covering of $\cG$. In some sense, $\mathcal{F}_\epsilon$ is the \emph{minimal} set of useful discriminators for the generator family $\cG$. Naturally, it leads to the definition of $L'$ as
\begin{align}\label{eqn.weakendesignfornorms}
L'(\bbP_X,\bbP_Y) = \sup_{f \in \mathcal{F}_\epsilon} \int f d(\bbP_X - \bbP_Y). 
\end{align}

We have the following result. 
\begin{theorem}
\label{thm:new_convergence}
Assuming the true distribution is $\bbP_X$, let the GAN estimator be
\begin{equation}\label{eqn.tildewfopt}
\tilde g = \argmin_{g(\cdot)\in \mathcal{G}} L'(\bbP_{g(Z)}, {\hat \bbP_X^n}),
\end{equation}
where $L'$ is defined in~(\ref{eqn.weakendesignfornorms}). Then
\begin{align}
L(\bbP_{\tilde g(Z)}, \bbP_X) & \leq 3\inf_{g(\cdot)\in \mathcal{G}} L(\bbP_{g(Z)},\bbP_X) + 2 L'(\bbP_X, \hat{\bbP}_X^n) + 4\epsilon.%C \sqrt{\log(\frac{2}{\delta})} \cdot (\frac{d \log n}{n})^{\frac{1}{2}}
\end{align}
Furthermore, suppose the $\epsilon$-covering number $N(\epsilon)$ of $\cG$ under $L$ satisfies $N(\epsilon) \leq C \left ( \frac{1}{\epsilon} \right )^d$, and there exists a constant $B$ such that for any $f\in \cF_\epsilon$, 
\begin{align}
\bbP_X \left (\bbE [f(X)] - \frac{1}{n} \sum_{i=1}^n f(X_i) > t \right ) \leq 2 \exp \left (-\frac{nt^2}{2B^2}\right ).
\end{align} 

Then, taking $\epsilon = \left (\frac{d\log n}{n} \right )^\frac{1}{2}$, with probability at least $1-\delta$, there exists some constant $D$ depending on $B,C$ such that
\begin{align}
L(\bbP_{\tilde g(Z)}, \bbP_X) \leq 3\inf_{g(\cdot)\in \mathcal{G}} L(\bbP_{g(Z)},\bbP_X) + D(B,C)\cdot \left (\frac{1}{n}\max \left (d \log( n), \log \frac{1}{\delta} \right ) \right )^\frac{1}{2}.%C \sqrt{\log(\frac{2}{\delta})} \cdot (\frac{d \log n}{n})^{\frac{1}{2}}
\end{align}
\end{theorem}

\begin{remark}
Theorem~\ref{thm:new_convergence} relies on two assumptions: the covering number of $\cG$ and the concentration properties of $f(X)$. The behavior of the assumed covering number is common in practice for parametric families~\cite{devroye2012combinatorial}. On the concentration properties, for any bounded function $f$, it follows from the Hoeffding inequality~\cite{hoeffding1963probability} that the sub-Gaussian behavior is true. In another case, if $f(\cdot)$ is a Lipschitz function and the distribution of $X$ satisfies the logarithmic Sobolev inequality \cite{ledoux1999concentration}, the result also holds. 
\end{remark}

The proof is deferred to Appendix (\ref{proof.new_convergence}). Theorem~\ref{thm:new_convergence} provide insights into practice. It implies 
\begin{enumerate}
\item The number of parameters of the discriminator family should be \emph{at most} that of the generator family since it suffices to only being able to distinguish between generator family distributions;
\item As the number of samples $n$ increases, we need to gradually increase the parameters (discrimination power) of the discriminators, as shown by the formula $\epsilon = \left ( \frac{d \log n}{n} \right )^{1/2}$ in Theorem~\ref{thm:new_convergence}; 
\item Whenever \emph{discriminator wins} phenomenon occurs, which means that in solving the projection estimator the discriminator can \emph{always} distinguish between the real and forged data, \emph{reduce} the number of parameters of the discriminator and it will \emph{not hurt} the learning performance as long as $L'$ is admissible;
\item As long as $L(\cdot,\cdot)$ is strong enough such that proximity under $L$ ensures mode collapse does not happen, projecting under $L'$ prohibits mode collapse. 
\end{enumerate}

\subsection{Robust estimation: Gaussian location model and total variation distance}

Consider the generator family being Gaussian location model and discrepancy measure being total variation distance, i.e., $\cG = \{\mathcal{N}(\bm{\mu}, \mathbf{I}):\bm{\mu}\in \mathbb{R}^d \}, L = \mathsf{TV}$. We have allowed ourselves to slightly abuse notation here to denote $\cG$ by the family of probability measures rather than the family of functions. We denote by $\bbP_g$ some member in the generator family $\cG$. In the robust statistics literature, it is usually assumed that there exists some distribution $g^* \in \cG$, and we observe $\bbP_X$ such that 
\begin{align}
    \mathsf{TV}(\bbP_X, \bbP_{g^*})
\end{align}
is small. Since $g^*$ is unknown, we can equivalently state this assumption as 
\begin{align}
    OPT = \inf_{g\in\cG}\mathsf{TV}(\bbP_X,\bbP_g)
\end{align}
is small. This model is understood in the robust statistics literature as the \emph{oblivious adversary contamination} model which is stronger than the Huber additive contamination model since it allows "deletion". 

In robust statistics, the ultimate goal is that given empirical observations from $\bbP_X$, we would like to obtain an estimate $\hat{g}$ such that $\mathsf{TV}(\bbP_{\hat{g}}, \bbP_{g^*})$ is small. Our observation is that it suffices to guarantee that $\mathsf{TV}(\bbP_{\hat{g}}, \bbP_X)$ is small, which is the goal of GANs. Indeed, it follows from the triangle inequality that
\begin{align}
    \mathsf{TV}(\bbP_{\hat{g}}, \bbP_{g^*}) & \leq \mathsf{TV}(\bbP_{\hat{g}}, \bbP_X) + \mathsf{TV}(\bbP_X, \bbP_{g^*}) \\
    & \leq OPT + \mathsf{TV}(\bbP_{\hat{g}}, \bbP_X). 
\end{align}

Utilizing the general framework of admissible distances, we present proofs for two GAN architectures that are both admissible distances of ($\mathsf{TV}$, $\cG$), and both achieve optimal statistical error up to a universal constant. 

\subsubsection{Admissible distance with parameter $(1,0,\mathsf{TV}')$}

%Denote $\{g_{i}:i\in[k]\}$ as the set of candidate distribution to search. We want to tell which candidate is good enough to make the total variation distance in equation (\ref{equ:tvdistance}) small.
We first construct an admissible distance for $(\mathsf{TV}, \cG)$ with parameter $(1,0,\mathsf{TV}')$. Define
\begin{align}\label{equ:def_yatracos}
   \mathsf{TV}'(\bbP_g, \bbP_X)  = \sup_{\|v\| = 1, b\in \mathbb{R}^d} \bbP_X( v^T(X-b)\geq 0) - \bbP_g(v^T (X-b) \geq 0). 
\end{align}

\begin{comment}
We construct the weakened distance by constraining the set of $A$ in the definition of total variation distance. 
Denote $A_{ij}$ as the set $\{\mathbf{x}\in \mathbb{R}^d: \bbP_{g_{i}}(\mathbf{x}) > \bbP_{g_{j}}(\mathbf{x}), g_i, g_j \in \cG \}$, $\mathcal{A} = \{A_{ij}\}$. When $g_i, g_j$ takes all possible generators in $\cG = \{\mathcal{N}(\mu,\mathbf{I}):\mu\in\mathbb{R}^d\}$, $\mathcal{A}$ is the set of all halfspaces $\{ \{\mathbf{x}\in\mathbb{R}^d:\left <\mathbf{v},\mathbf{x}\right>+\mathbf{b}>0\}:\|\mathbf{v}\|_2 =1,\mathbf{b}\in\mathbb{R}^d\}$.
The weakened distance is thus defined as The optimality of projection under $\mathsf{TV}'$ is guaranteed in the following theorem:
\end{comment}

\begin{theorem}
\label{thm:yatracos_weakened}
Let the GAN estimator be
\begin{align}
\label{equ:yatracos_weakened}
    \tilde g = \argmin_{g\in\cG}  \mathsf{TV}'(\bbP_g, \hat \bbP_X^n),
\end{align}
where $\mathsf{TV}'$ is defined in (\ref{equ:def_yatracos}). Then, $\mathsf{TV}'$ is an admissible distance of $(\mathsf{TV},\mathcal{G})$ with parameter $(1,0,\mathsf{TV}')$, and with probability at least $1-\delta$, there exists some universal constant $C$, such that 
\begin{align}
    \mathsf{TV}(\bbP_{\tilde g}, \bbP_{X})\leq 3\cdot OPT +  C\cdot \sqrt{\frac{d+1}{n}}+ 2\sqrt{\frac{\log(1/\delta)}{2n}}.
\end{align}
\end{theorem}

The proof is deferred to Appendix (\ref{proof.yatracos}).

\subsubsection{Tukey median}

The Tukey median approach can be equivalently formulated as projecting under another admissible distance. It is defined as the following: for two distributions $\bbP,\bbQ$, we define
\begin{align}\label{equ:tukeymedian_def}
L_{\mathsf{Tukey}}(\bbP, \bbQ) = \sup_{\mathbf{v}\in \mathbb{R}^d} \bbQ\left( (X-\mathbb{E}_\bbP[X])^T \mathbf{v} >0 \right) -1/2. 
\end{align}

It has been shown in previous literature that Tukey median is able to achieve optimal convergence rate under Huber's additive contamination model~\cite{chen2018robust}. Based on the theory of admissible distance, We provide a new simple proof for the optimality of Tukey median.
\begin{theorem}
 \label{thm:equ:tukeymedian}
Let the GAN estimator (Tukey median) be defined as
\begin{align}
\label{equ:tukeymedian}
\tilde g = \argmin_{g\in \mathcal{G}} L_{\mathsf{Tukey}}(\bbP_g,\hat \bbP_X^n), 
\end{align}
where $L_{\mathsf{Tukey}}$ is defined in (\ref{equ:tukeymedian_def}). Then, $L_{\mathsf{Tukey}}$ is an admissible distance of $(\mathsf{TV},\mathcal{G})$ with parameter $(2,0,\mathsf{TV}')$, where $\mathsf{TV}'(\bbP_1, \bbP_2)$ is defined in equation (\ref{equ:def_yatracos}). With probability at least $1-\delta$, there exists some universal constant $C$ such that 
\begin{align}
    \mathsf{TV}(\bbP_{\tilde g}, \bbP_{X})\leq 5\cdot OPT +  C\cdot  \sqrt{\frac{d+1}{n}}+ 4\sqrt{\frac{\log(1/\delta)}{2n}}.
\end{align}
\end{theorem}

The proof is deferred to Appendix (\ref{proof.tukey}).

\subsection{$W_2$-GAN model}

It was first shown in~\cite{feizi2017understanding} that when the generator family is rank-constrained linear maps, $L$ is the Wasserstein-2 distance, and the real distribution is Gaussian, then the population GAN solution~(\ref{eqn.popuminimization}) produces the PCA of the covariance matrix of the true distribution. Inspired by this phenomenon, Quadratic GAN is proposed and both theoretical and empirical results demonstrate its fast generalization. In this section, we relax the condition in~(\ref{eqn.popuminimization}) to consider arbitrary real distribution, but stick to the rank-constrained linear map generator and $W_2$ distance. This model generalizes PCA and provides an intellectually stimulating playground for studying the generalization and optimization properties of GANs. We demonstrate a few striking behaviors of the $W_2$-GAN model in this section.

\subsubsection{Exponential sample complexity for the naive algorithm}

The $W_2$ distance is defined as
\begin{equation}
\label{eq:wpdist}
W_2(\bbP_X,\bbP_Y) = \sqrt{\inf_{\pi: \pi_X = \bbP_X, \pi_Y = \bbP_Y} \bbE  \|X - Y\|^2},
\end{equation}
which admits a dual form:
\begin{equation}
\label{equ:w2dual}
W_2^2(\bbP_X, \bbP_Y) = \mathbb{E}\|X\|_2^2 + \mathbb{E}\| Y \|_2^2 + 2 \sup_{\psi(\cdot): \text{convex}} -\mathbb{E}[\psi(X)] - \mathbb{E}[\psi^\star(Y)],
\end{equation}
where $\psi^\star(x) = \sup_\mathbf{v} \mathbf{v}^T\mathbf{y} - \psi(\mathbf{v)}$ is the convex conjugate of $\psi$. This dual form of $W_2^2$ can be used to implement the quadratic GAN optimization in a min-max architecture between function $g(\cdot)$ and $\psi(\cdot)$. 

Without loss of generality, we can assume the means of $\bbP_X,\bbP_Y$ are zero. Indeed, for any two distributions $\bbP_X, \bbP_Y$, assume the corresponding means are $\bm{\mu_X}$ and $\bm{\mu_Y}$, and denote $X' = X - \bm{\mu_X}$, $Y' = Y - \bm{\mu_Y}$. Then, we have
\begin{align}
    W_2^2(\bbP_X, \bbP_Y) & =  \inf_\pi \bbE\|X\|_2^2+\bbE\|Y\|_2^2 - 2\bbE[X^TY] \nonumber \\
    & = \inf_\pi \bbE[\operatorname{Tr}(XX^T)]+\bbE[\operatorname{Tr}(YY^T)] - 2\bbE[X^TY] \nonumber \\
    & = \inf_\pi \bbE[\operatorname{Tr}(X'X'^T)]+\bm{\mu_X}^T\bm{\mu_X}+\bbE[\operatorname{Tr}(Y'Y'^T)]+ \bm{\mu_Y}^T\bm{\mu_Y} - 2\bbE[X'^TY'] - 2\bm{\mu_X}^T\bm{\mu_Y} \nonumber \\
    & = \|\bm{\mu_X} - \bm{\mu_Y} \|_2^2 + W_2^2(\bbP_{X'}, \bbP_{Y'}).
\end{align}
Thus, the $W_2^2$ between distributions with non-zero means can be decomposed to the difference between means and $W_2^2$ between the centered distributions. In the rest of the paper, we mainly consider $W_2$ distance between zero-mean distributions except for Theorem~\ref{thm:new_convergence2} and~\ref{thm:tightness}. 

To demonstrate the intricacies of GAN design, we start with the vanilla GAN architecture. We assume the generator is linear, i.e. $\cG = \{g(Z) = \mathbf{A}Z: \mathbf{A} \in \mathbb{R}^{d\times r}, Z\in \mathbb{R}^r\}$, which means that $g(Z) \sim \mathcal{N}(0,\mathbf{K})$ for some $\mathbf{K}$ positive semidefinite with rank at most $r$. The optimal generator in this setting is trivially $g(Z) = \mathbf{K}^{1/2} Z$ if $\bbP_X = \mathcal{N}(0, \mathbf{K})$ and $r = d$. However, as the next theorem shows, it takes the GAN architecture more than exponential number of samples to approach the optimal generator even when $\mathbf{K}$ is identity. It is a refinement of~\cite{feizi2017understanding}. 

\begin{theorem}
\label{thm:naive_convergence_w2}
Let the GAN estimator be
\begin{align} \label{equ:origin_naive}
       \tilde g = \argmin_{g(\cdot)\in \cG} W_2(\bbP_{g(Z)}, \hat{\bbP}_X^n).
\end{align}
Then we have
\begin{align}\label{eqn.thm5conclution_population}
    W_2(\bbP_{\tilde{g}(Z)},\bbP_X) & \leq \inf_{g\in \cG}W_2(\bbP_{{g}(Z)}, \bbP_X) + 2W_2(\bbP_X,\hat \bbP_X^n).
\end{align}
Furthermore, assume $\mathbb{E}_{\bbP_X}[X] = 0$, and $\bbP_X$  satisfies the following assumption (which $\mathcal{N}(0, \mathbf{I}_d)$ satisfies when $d\geq 3$):
\begin{itemize}
    \item $X\sim\mu$ satisfies logarithmic Sobolev inequality, i.e. for some constant $B>0$ and all smooth $f$ on $\mathbb{R}^d$,
\begin{align}\label{eqn.thm5assumption1}
    \operatorname{Ent}(f^2)\leq 2B\mathbb{E}(|\nabla f|^2),
\end{align}
where $\nabla f$ is the usual gradient of $f$, $\operatorname{Ent}(f) = \bbE_\mu(f\log f)-\bbE_\mu(f)\log\bbE_\mu(f)$.
\item Upper bound on expected $W_2$ convergence rate\footnote{When the dimension of the support of $\bbP_X$ is less than the ambient dimension $d$, the convergence rate is depending on the intrinsic dimension but not the ambient dimension. } for some constant $C_1$ depending on $\bbP_X$: \begin{align}\label{eqn.thm5assumption2}
    \bbE [W_2^2(\bbP_X, {\hat \bbP_X^n})] \leq C_1(\bbP_X)\cdot n^{-2/d}.
\end{align}
\end{itemize}
Then, there exists some constant $C$ depending on $\bbP_X$ , such that with probability at least $1-\delta$, 
\begin{align}\label{eqn.thm5conclution2}
    W_2(\bbP_{\tilde{g}(Z)},\bbP_X) \leq \inf_{g\in \cG}W_2(\bbP_{{g}(Z)}, \bbP_X) + C(\bbP_X)\cdot n^{-1/d}  + 2\sqrt{\frac{2B\log(1/\delta)}{n}}.
\end{align}

For lower bound, when $\bbP_X= \mathcal{N}(0, \mathbf{I}_d)$, we have with probability at least 0.99, there exist some constants $C_2, C_3(d)$,  such that when $n > C_3(d)$ and $d\geq 5$,
\begin{align}
W_2( \bbP_{\tilde{g}(Z)}, \mathcal{N}(\mathbf{0}, \mathbf{I}_d))  \geq C_2d^{1/2}n^{-3/d}.
\end{align}
\end{theorem}
\begin{remark}
Equation (\ref{eqn.thm5conclution2}) depends on two assumptions: the logarithmic Sobolev inequality (\ref{eqn.thm5assumption1}) and the expected convergence rate (\ref{eqn.thm5assumption2}). The logarithmic Sobolev inequality holds for a wide range of distributions including Gaussian distribution~\cite{ledoux1999concentration}. 
The upper bound on the expected convergence rate holds for any distribution $\bbP_X$  with finite $\lceil 2d/(d-2) \rceil$ moment and $d\geq5$~\cite{fournier2015rate}. It also holds for $\mathcal{N}(0, \mathbf{I}_d)$ when $d\geq 3$.
\end{remark}
The proof is deferred to Appendix \ref{proof.w2_weakened}. The proof of the theorem also provides a lower bound for convergence rate of Wasserstein-2 distance between Gaussian and its empirical distribution, which appears to be new. 

Theorem~\ref{thm:naive_convergence_w2} shows that even when there exists a generator that can perfectly reconstruct the real distribution that produces the samples, the sample complexity grows at least exponentially with the dimension of the data distribution. The result in Theorem~\ref{thm:naive_convergence_w2} can also be understood as an instantiation of the \emph{discriminator winning} phenomenon. Indeed, unless we are given samples at least exponential in the dimension $d$, the discriminator in the $W_2$ dual representation can always easily tell the difference between the empirical distribution $\hat{\bbP}_X^n$ and the solved optimal generator distribution $\tilde{g}$. Intuitively, to alleviate this phenomenon, one needs to design discriminators to properly weaken their discriminating power.

\subsubsection{Polynomial sample complexity GAN design}

\begin{comment}
We design the weakest discriminator that is only able to distinguish the distribution in  generator family. The Wasserstein-2 distance between two generators $g_1, g_2 \in \cG$ can be characterized as
\begin{equation}
\label{equ:w2dualgenerator}
W_2^2(\bbP_{g_1(Z)}, \bbP_{g_2(Z)}) = \mathbb{E}\|g_1(Z)\|_2^2 + \mathbb{E}\| g_2(Z) \|_2^2 + 2 \sup_{\psi(\cdot): \text{convex}} -\mathbb{E}[\psi(g_1(Z))] - \mathbb{E}[\psi^\star(g_2(Z))].
\end{equation}

Without the loss of generality, The choice of $\psi$ can be shown in the following lemma.

\begin{lemma}[Brenier's theorem \cite{villani2008optimal}]
\label{lemma:opt}
%Let $\bbP_X$ be absolutely continuous whose support contained in a convex set in $\mathbb{R}^d$.
For a fixed $g_1(\cdot), g_2(\cdot) \in \cG$, the optimal solution $\psi^{opt}(\cdot)$ for (\ref{equ:w2dualgenerator})  is unique. Moreover, we have 

\begin{equation}
\nabla \psi^{opt}(g_1(Z)) \stackrel{dist}{=} g_2(Z).
\end{equation} 

where $\stackrel{dist}{=}$ means matching distributions.
\end{lemma}

Since $g(Z)$ is always Gaussian, $\nabla \psi^{opt}(\cdot)$ should be a linear function. Thus, without loss of
generality,  $\psi(\cdot)$ in the discriminator optimization can be constrained to  $\psi(\mathbf{y}) = \mathbf{y}^t\mathbf{A}\mathbf{y}/2$ where $\mathbf{A}$ is positive semidefinite.  Thus the $W_2$ distance can be rewritten as 
given weakened distance in terms of constraining the generator and discriminator family.
\end{comment}

It was proposed in~\cite{feizi2017understanding} that when $\bbP_X$ is Gaussian, one shall project under the \emph{Quadratic GAN} distance defined as
\begin{align}\label{eqn.moment_w2}
    \tilde W_2(\bbP_1,\bbP_2) = W_2(\mathcal{N}(\bm{\mu_1}, \bm{\Sigma_1}), \mathcal{N}(\bm{\mu_2}, \bm{\Sigma_2})),
\end{align}
where $\bm{\mu_1}, \bm{\mu_2}$ are means of $\bbP_1, \bbP_2$, and $\bm{\Sigma_1}, \bm{\Sigma_2}$ are covariances of $\bbP_1, \bbP_2$. This distance is also used in the literature as the Fr\'echet Inception Distance~\cite{heusel2017gans}. The following result analyzes the performance of projecting under $\tilde{W}_2$. 
\begin{theorem}
\label{thm:new_convergence2}
Assume the generator family to search is rank-$r$ affine generators:      
\begin{align}
         \cG_r = \{g(Z) = \mathbf{A}Z+\mathbf{b}: \mathbf{A} \in \mathbb{R}^{d\times r} , Z\in \mathbb{R}^r, \mathbf{b}\in \mathbb{R}^d\}.
     \end{align}
Denote the family of full rank Gaussian distribution generator as $\cG = \{g(Z) = \mathbf{A}Z+\mathbf{b}: \mathbf{A} \in \mathbb{R}^{d\times d}, Z\in \mathbb{R}^d, \mathbf{b}\in \mathbb{R}^d \}$. Let the GAN estimator be
\begin{align}
\label{equ:project_weakened_W2}
\tilde g & = \argmin_{g(\cdot)\in \mathcal{G}_r} \tW_2(\bbP_{g(Z)}, {\hat \bbP}_X^n).
%& = \argmin_{g(\cdot)\in \mathcal{G}} \sup_{\mathbf{A} \geq 0, R(\mathbf{A}) \supset R(\mathbf{K}_{g(Z)})} \operatorname{Tr}((\mathbf{I} - \mathbf{A})\hat{\Sigma}) + \operatorname{Tr}((\mathbf{I} - \mathbf{A}^\dagger)\mathbf{K}_{g(Z)})
\end{align}
Then, $\tilde W_2$ is an admissible distance for $(W_2, \cG)$ with parameter $(1,0,\tW_2)$, and 
\begin{align}
W_{2}(\bbP_{\tilde g(Z)}, \bbP_X) & \leq \| \bm{\Sigma}_X^{1/2} - \bm{\Sigma}_r^{1/2} \|_F  + 2\inf_{g(\cdot)\in \mathcal{G}} W_{2}(\bbP_{g(Z)},\bbP_X) + 2\tW_2(\hat \bbP_X^n, \bbP_X), 
\label{equ:new_convergence_w2_pca_population}
\end{align}
here  $\mu_X, \bm{\Sigma}_X$ are the mean and covariance matrix of $X$, respectively, $\bm{\Sigma}_r$ is the best rank-$r$ approximation of $\bm{\Sigma}_X$ under Frobenius norm (the $r$-PCA solution). 

Furthermore, if the minimum eigenvalue of the covariance matrix of $X$ is lower bounded by a constant $B$\footnote{If the minimum eigenvalue is not lower bounded by a positive constant, we can show the stochastic error $\lesssim d^{1/2}\| \mathbf{\Sigma}_X - \hat{\mathbf{\Sigma}}\|^{1/2} + + \| \bm{\mu}_X - \Hat{\bm{\mu}} \|$.}, then
\begin{align}
W_{2}(\bbP_{\tilde g(Z)}, \bbP_X) \leq \| \bm{\Sigma}_X^{1/2} - \bm{\Sigma}_r^{1/2} \|_F  + 2\inf_{g(\cdot)\in \mathcal{G}} W_{2}(\bbP_{g(Z)},\bbP_X) + 2\sqrt{ \frac{d}{B}} \| \mathbf{\Sigma}_X - \hat{\mathbf{\Sigma}}\|+ 2\| \bm{\mu}_X - \Hat{\bm{\mu}} \|, 
\label{equ:new_convergence_w2_pca}
\end{align}
where $\bm{\hat{\mu}},\bm{\hat{\Sigma}}$ are the empirical mean and covariance matrix of $\hat{\bbP}_X^n$, respectively.

\end{theorem}

\begin{remark}
It follows from~\cite{vershynin2012close} that if $\bbP_X$ is sub-Gaussian, i.e., there exists some constant $L$ such that
\begin{align}
    \bbP_X(|\langle X - \mathbb{E}[X] , x \rangle|>t) \leq 2 e^{-t^2/L^2}
\end{align}
for $t>0$ and $\|x\|_2 = 1$, then with probability at least $1-\delta$, we have
\begin{align}
     \| \mathbf{\Sigma}_X - \hat{\mathbf{\Sigma}}\| \lesssim L^2 \sqrt{\log \left ( \frac{2}{\delta} \right )} \sqrt{\frac{d}{n}}
\end{align}
and
\begin{align}
\| \bm{\mu}_X - \Hat{\bm{\mu}} \| \lesssim_{\delta, L} \sqrt{\frac{d}{n}}. 
\end{align}
More generally, it follows from \cite[Theorem 5.6.1]{vershynin2018high} that if $\mathbb{E}[X] = 0$, and for some $K>0$, $\|X\| \leq K$ almost surely, then for every $n\geq 1$, 
\begin{align}
    \mathbb{E}\| \mathbf{\Sigma}_X - \Hat{\mathbf{\Sigma}}\| \lesssim \sqrt{\frac{K^2 \| \mathbf{\Sigma} \| \log d}{n}} + \frac{(K^2 + \|\mathbf{\Sigma}\|) \log d}{n}. 
\end{align}
\end{remark}

The proof of theorem is deferred to Appendix \ref{app:pf_thm3}. The results in Theorem~\ref{thm:new_convergence2} can be interpreted in the following way. 
\begin{enumerate}
    \item $ \| \bm{\Sigma}_X^{1/2} - \bm{\Sigma}_r^{1/2} \|_F$: the PCA error introduced by rank mismatch between generator family and true distribution;
    \item $2 \cdot \inf_{g(\cdot)\in \mathcal{G}} W_{2}(\bbP_{g(Z)},\bbP_X)$: the \emph{non-Gaussianness} of distribution $\bbP_X$; 
    \item $2 \cdot \tW_2(\hat \bbP_X^n, \bbP_X)$: stochastic error that vanishes as $n\to \infty$. 
\end{enumerate}

It is interesting to note that the when the distribution $\bbP_X$ is indeed Gaussian, which is equivalent to $\inf_{g(\cdot)\in \mathcal{G}} W_{2}(\bbP_{g(Z)},\bbP_X)=0$, the PCA error term cannot be improved even by a constant when $n\to \infty$. Indeed, it is due to the observation~\cite{feizi2017understanding} that
\begin{align}
    \inf_{g\in \cG_r} W_2(\bbP_{g(Z)}, \cN(\bm{\mu}_X, \mathbf{\Sigma}_X)) = \| \bm{\Sigma}_X^{1/2} - \bm{\Sigma}_r^{1/2} \|_F. 
\end{align}
However, when $r = d$, it is not clear whether the factor $2$ in front of the term $\inf_{g(\cdot)\in \mathcal{G}} W_{2}(\bbP_{g(Z)},\bbP_X)$ is necessary. Can we reduce the constant $2$ to $1$ perhaps by using other algorithms? 

\subsubsection{The near-optimality of the Quadratic GAN}

Motivated by Theorem~\ref{thm:new_convergence2}, we first study the problem of whether it is possible to reduce the prefactor $2$ of the \emph{non-Gaussianness} term in~(\ref{equ:new_convergence_w2_pca}). The following theorem shows that it is \emph{impossible} to achieve constant $1$ without sacrificing the stochastic error term. 
\begin{theorem}
\label{thm:tightness}
Consider full-rank affine generator family $\cG =  \{g(Z) = \mathbf{A}Z+\mathbf{b}: \mathbf{A} \in \mathbb{R}^{d\times d}, \mathbf{b}\in \mathbb{R}^d, Z\in \mathbb{R}^d\}$. Define
\begin{align}
    \mathcal{P}_K = \{\cN(0,I)\}\cup \{\bbP_X: \| X - \mathbb{E}[X]\| \leq K\text{ a.s.}, \lambda_{\text{min}}(\mathbb{E}[(X-\mathbb{E}[X])(X-\mathbb{E}[X])^T]) \geq \frac{1}{2}\}.
\end{align}
Then, there exists a constant $D>0$ such that there is no algorithm with the following property: for any $\bbP_X \in \mathcal{P}_K, K \geq D( \sqrt{d} + \sqrt{ \ln\left ( n \right )})$, algorithm $\mathcal{A}$ makes $n$ draws from $\bbP_X$ and with probability at least $51/100$ outputs a hypothesis $\hat{\bbP}_X$ satisfying 
\begin{equation}
\label{equ:tightness}
W_2(\hat{\bbP}_X, \bbP_X) \leq \inf_{g(\cdot)\in \mathcal{G}} W_2(\bbP_{g(Z)},\bbP_X) + \epsilon(n,d),
\end{equation}
where $\epsilon(n,d) \ll_d n^{-6/d}$ as $n\to \infty$.
%with $\epsilon(n, d) = o(\frac{1}{d^{\frac{3}{2}}}n^{-\frac{3}{d}})$. 
\end{theorem}
The proof is relegated to Appendix (\ref{proof.tightness}). As an illustrating example in Figure \ref{fig:fundamental_limit}, we specify the discrepancy measure $L$ to be $W_2$, $\cG =  \{g(Z) = \mathbf{A}Z+\mathbf{b}: \mathbf{A} \in \mathbb{R}^{d\times d}, \mathbf{b}\in \mathbb{R}^d, Z\in \mathbb{R}^d\}$, and plot $\mathbb{E}[W_2(\bbP_{\tilde{g}(Z)},\bbP_X)]$  as a function of $OPT=\inf_{g\in\cG}W_2(\bbP_X, \bbP_{g(Z)})$, for naive projection and proposed algorithm. Furthermore, we show the minimal stochastic error for keeping the slope $1$, which means in order to make the stochastic error $\epsilon_n$ in equation (\ref{eqn:generalization}) smaller, it is \emph{necessary} to sacrifice the constant $C$.

\begin{figure}[htb!]
    \centering
    \includegraphics[width=0.6\linewidth]{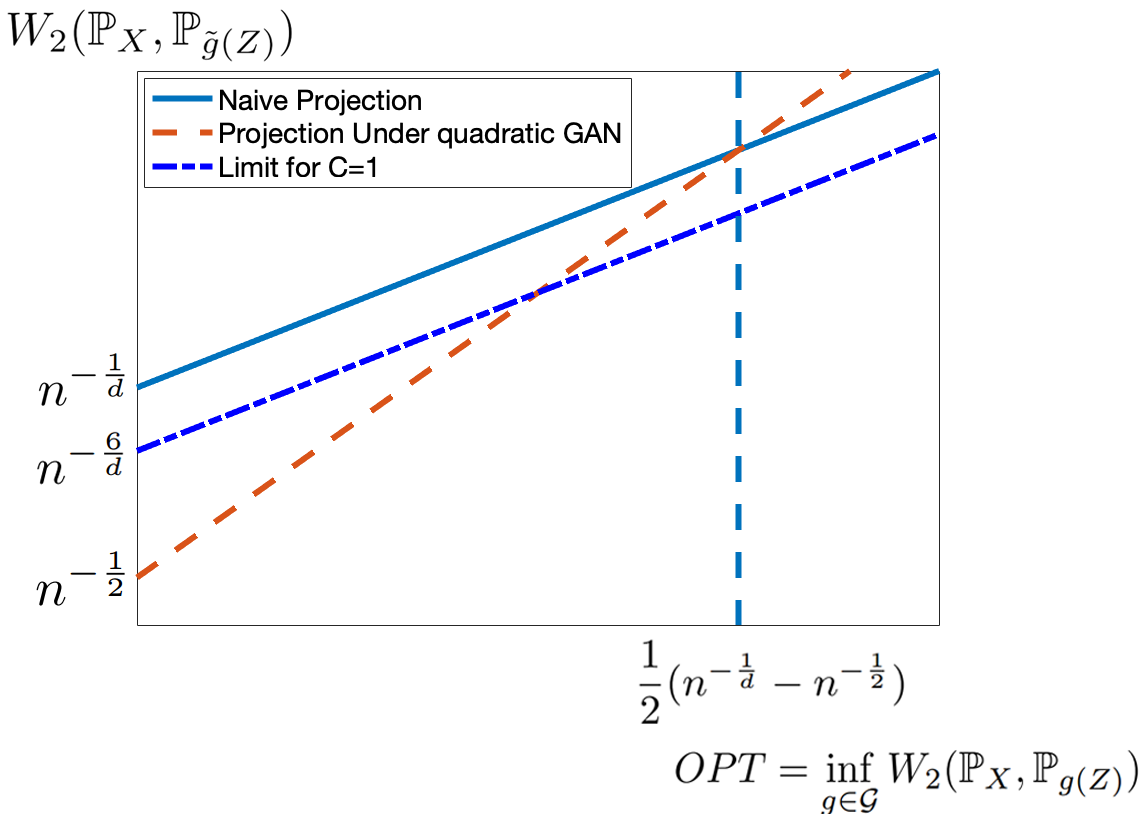}
    \caption{The trade-off between oracle error $OPT$ and stochastic error $\epsilon_n$.  Although Quadratic GAN suffers from a larger prefactor $C = 2$ for the oracle error term, it achieves significantly smaller overall error when the generator family is close to the target data distribution. 'Limit for $C=1$' line represents the best possible performance we can achieve if we want the slope to be $1$, which still requires exponential sample complexity when $OPT=0$. For any $OPT<\frac{1}{2}(n^{-1/d}-n^{-1/2})$ (left of the vertical line), the projection under the quadratic GAN method achieves smaller overall error. Beyond the vertical line, one may consider switching to a more appropriate generator family. }
    \label{fig:fundamental_limit}
    
\end{figure}

Combining Theorem~\ref{thm:new_convergence2} and~\ref{thm:tightness}, we know that for the non-Gaussianness term, the prefactor $2$ is achievable by the Quadratic GAN, while the constant $1$ is not achievable without sacrificing the stochastic error. It is also shown in~\cite{bousquet2019optimal} that the sacrifice of constant factor in general cannot be avoided if we require proper learning.  Can we pin down the exact constant while still keeping the stochastic error term small? The next theorem shows that the Quadratic GAN method cannot achieve any constant that is strictly less than $\sqrt{2}$ even with \emph{infinite} sample size. 

\begin{theorem}
\label{thm:sqrt2_lowerbound}
Suppose 
\begin{align}
    \cG = \{g(Z) = \mathbf{A}Z: \mathbf{A} \in \mathbb{R}^{d\times d}, Z \in \mathbb{R}^d\},
\end{align}
and the GAN estimator $\tilde{g}$ is defined in~(\ref{equ:project_weakened_W2}). Then, for any $\epsilon>0$, there exists some distribution $\bbP_X, \mathbb{E}_{\bbP_X}[X] = 0, \inf_{g\in \cG} W_2(\bbP_{g(Z)}, \bbP_X)>0$, such that with infinite sample size, we have
\begin{align}
W_2(\bbP_{\tilde g(Z)}, \bbP_X) \geq (\sqrt{2}-\epsilon)\inf_{g\in \cG} W_2(\bbP_{g(Z)}, \bbP_X). 
\end{align}
\end{theorem}

\subsection{Admissible distances for Cascade GANs}

In the Cascade GAN setting in~(\ref{equ:cascaded}), we have motivated the usage of the distance
\begin{align}\label{eqn.cascadelpopulation}
    L(\bbP, \bbQ) =  \min_{g' \in \cG'} L_1(\bbP_{g'(Z)}, \bbQ) + \lambda L_2(\bbP, \bbP_{g'(Z)})
\end{align}
when the overall perturbation cannot be represented as only perturbing under $L_1$ or $L_2$, but a cascaded perturbation $L_2$ followed by $L_1$. In this section, we consider the generalization problem~(\ref{eqn:generalization}) for this new distance. 

\begin{theorem}
\label{thm:cascaded}
Assume we have $L_1'$ as an admissible distance for ($L_1, \cG'$) with parameter $(c_1, c_2, L_1'')$, $L_2'$ as an admissible distance for ($L_2, \cG'$) with parameter $(d_1, d_2, L_2'')$ for generator family $\cG'$, and $\cG\subset \cG'$. We further assume that $L_1'(\bbP_g, \bbP_g) = L_2'(\bbP_g, \bbP_g) = 0$ for all $g\in\mathcal{G}'$. Define the new distance as
\begin{align}
    L'(\bbP, \bbQ) = \min_{g'\in\cG'} L_1'(\bbP_{g'(Z)},\bbQ) + \lambda L_2'(\bbP,\bbP_{g'(Z)}).
\end{align}
Then, $L'$ is an admissible distance for $L$ in~(\ref{eqn.cascadelpopulation}) and $\cG$ with parameter 
\[
\left(\max(c_1, d_1), \max(c_1, d_1)\cdot \left(\frac{c_2}{c_1} +\frac{d_2}{d_1} \right ),L_1'' \right ),
\]
and if we define the GAN estimator as
\begin{align}
    \hat{g} = \argmin_{g\in \cG} L'(\bbP_{g(Z)}, \hat{\bbP}_X^n),
\end{align}
we have
\begin{align}
    L(\bbP_{\hat{g}(Z)}, \bbP_X) \leq  (1+2\max(c_1, d_1))OPT + 2\max(c_1, d_1) L_1''(\bbP_X,  \hat{\bbP}_X^n) + \max(c_1, d_1)\cdot \left(\frac{c_2}{c_1} +\frac{d_2}{d_1} \right )，
\end{align}
where
\begin{align}
    OPT = \inf_{g\in \cG} L(\bbP_{g(Z)}, \bbP_X). 
\end{align}
\end{theorem}

\subsection{Testing the strength of discriminators with empirical samples}

In order to achieve a fast convergence rate, it is crucial that $L''$ in Theorem~\ref{thm.populationinsights} is weak enough. As illustrated in Theorem~\ref{thm:naive_convergence_w2}, with inappropriate choice of discriminator, discriminator winning phenomenon will lead to poor performance of GANs. In this section, we propose a principled way to \emph{predict} the number of samples needed to alleviate the discriminator winning phenomenon  from empirical samples by a matching theory~\cite{ambrosio2018finer}. Concretely, we aim at estimating the distance between empirical distribution and its population distribution $L(\hat{\bbP}_X^n,\bbP_X)$ using \emph{only} empirical samples. 

The method we propose is as follows. Given $N$ samples from true distribution, we select the first $n<N/2$ samples to form empirical distribution $\hat \bbP_1$, and another disjoint $n$ samples to form empirical distribution $\hat \bbP_2$. Then we compute $L(\hat \bbP_1, \hat \bbP_2)$ as a \emph{proxy} for $L(\bbP_X, \hat{\bbP}_X^n)$. A special case of this method is known as maximum discrepancy~\cite{Bartlett2002} in the literature. However, as is shown below, our method is more general and can be applied to any pseudometric that is convex in its first argument. Our non-asymptotic result below can be seen as a generalization of the asymptotic justification for matching method in~\cite{rippl2016limit}. 

\begin{theorem}\label{thm:strengthtesting}
 Given $N$ samples from true distribution, we select the first $n<N/2$ samples to form empirical distribution $\hat \bbP_1$, and another disjoint $n$ samples to form empirical distribution $\hat \bbP_2$. Suppose $L(\bbP, \bbQ)$ is a pseudometric which is convex in its first argument. Then, 
\begin{align}
  \bbE[L(\bbP_X, \hat \bbP_X^n)] \leq \bbE[L(\hat \bbP_1, \hat \bbP_2)] \leq 2\bbE[L(\bbP_X, \hat \bbP_X^n)]. 
\end{align}
\end{theorem}
The proof is deferred to Appendix \ref{pf:strengthtesting}. In the experiments below, we use the matching method to successfully recover the exponent as predicted in Theorem \ref{thm:naive_convergence_w2} and \ref{thm:new_convergence2}. We conduct the experiments for estimating $W_2$ and $\tW_2$ distances between an isotropic Gaussian distribution and its empirical distribution with $d = 10$ and $d=20$, where the result is shown in Figure \ref{fig:matching}. We can see that the matching method succesfully recovers the exponent of sample size $n$.  In our analysis, it is shown that for $d\geq 3$ and $n$ large enough, the quadratic GAN can achieve convergence rate as $dn^{-1/2}$, and the fitted rate from two figures are $d^{1.00}n^{-0.53}$. Meanwhile the convergence rate for the naive $W_2$ GAN structure is lower bounded by  $d^{1/2}n^{-3/d}$ and upper bounded by $C(d)n^{-1/d}$. The fitted rate is $d^{0.48}n^{-1.24/d}$. We conjecture that the convergence rate is exactly $d^{1/2}n^{-1/d}$ up to universal multiplicative constants for the $W_2$ distance when $d\geq 3$, whose validity is currently an open problem in optimal transport theory. Based on the matching method, we can estimate the number of samples required for $L(\bbP_X, \hat{\bbP}_X^n)$ to achieve certain error $\epsilon$, which helps predict when discriminator winning problem occurs.

\begin{figure}[htb!]
    \centering
    \begin{subfigure}[t]{0.45\textwidth}
    \includegraphics[width=\linewidth]{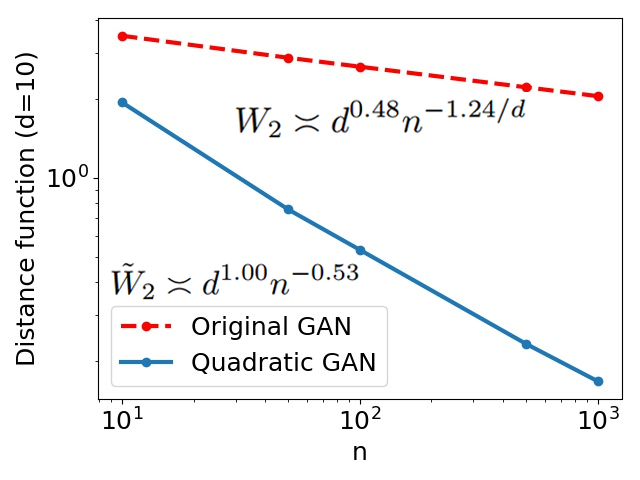}
     \end{subfigure}
     ~\begin{subfigure}[t]{0.45\textwidth}
    \includegraphics[width=\linewidth]{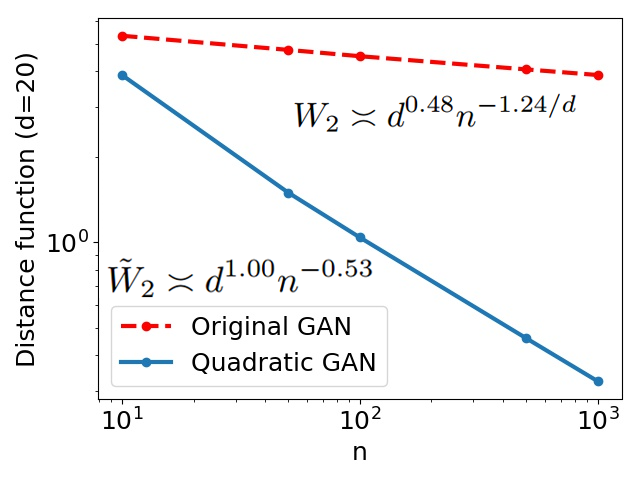}
     \end{subfigure}
    \caption{The estimated $W_2$ and $\tW_2$ distances between a high dimensional isotropic Gaussian distribution and its empirical distribution for original $W_2$ GAN and Quadratic GAN. The left figure is for 10-dimensional Gaussian and the right figure is for 20-dimensional Gaussian. The $x$ and $y$ axes are all logarithmic. The fitted coefficients of linear model from two lines show that the matching method produces results that are consistent with theoretical predictions with reasonably small sample sizes. }
    \label{fig:matching}
\end{figure}

\section{Optimization of GANs}
\label{sec:computation}

GANs are usually optimized using alternating gradient descent (AGD), which is known to be difficult to train, with little known stability or convergence guarantees. An unstable training procedure may lead to {\bf mode collapse}. Given a learning algorithm proposed from statistical considerations, it is crucial to ensure that it is \emph{provably computationally efficient}. Thus it is of paramount importance to ensure that widely used algorithms such as AGD are able to \emph{provably} solve the GAN optimization problem.

It was shown in prior work~\cite{feizi2017understanding} that the optimal generated distribution under (\ref{equ:project_weakened_W2}) is a Gaussian distribution whose covariance matrix is the $r$-PCA of the \emph{empirical} covariance matrix $\hat{\Sigma}$. However, in practice, PCA is solved using specialized algorithms but not alternating gradient descent. Recent work in Fr\'echet Inception distance~\cite{heusel2017gans} also suggests training GAN under the quadratic GAN distance $\tilde W_2$ between coding layers. Hence, it is natural to answer whether \emph{particularizing} the alternating gradient descent algorithm to this specific setting produces a \emph{convergent, stable, and efficient} algorithm for the $r$-PCA. We mention that \cite{feizi2017understanding} has obtained some results in this direction but here we explore the question in greater depth.

In this section, we show that for the linear generator family 
\begin{align}\label{eqn.grdefcom}
         \cG_r = \{g(Z) = \mathbf{A}Z: \mathbf{A} \in \mathbb{R}^{d\times r}, Z\in \mathbb{R}^r\}
\end{align}
and $W_2$ distance, AGD is globally asymptotically stable when $r=d$, but is generally not \emph{locally asymptotically stable} when $r<d$. We propose another GAN architecture based on parameter sharing which is provably \emph{globally asymptotically stable}\footnote{The definition of local and global asymptotic stability is in Appendix (\ref{appendix.def}).} and successfully computes the PCA using AGD. Intuitively, one reason why AGD is not \emph{locally asymptotically stable} in solving the $W_2$-GAN formulation of PCA is that the minimax value is not the same as the maximin value in the GAN formulation, but AGD does not favor minimax explicitly over maximin. However, we do emphasize that this is not the sole reason, since the example in Figure~\ref{fig:dynamics_sim} has duality gap zero, but alternating gradient descent is not asymptotically stable in solving the minimax. 

The key insight we draw is that utilizing information on where the maximum is attained can help us design a \emph{new minimax} formulation whose duality gap (minimax value minus maximin value) is zero, while at the same time the new minimax value is \emph{equal} to the old minimax value. It is in constrast with other convexification approaches such as MIX-GAN~\cite{arora2017generalization}, MAD-GAN~\cite{ghosh2017multi} and Stackelberg GAN~\cite{zhang2019stackelberg}, where new minimax problems are proposed with smaller duality gap, but the new minimax problem shares \emph{neither} the same minimax value \emph{nor} the same generator set as the old formulation. Indeed, applying their formulations to the PCA setting mentioned above would trivialize the problem: by using multiple generators it is equivalent to finding the best \emph{full}-dimensional Gaussian approximation for an arbitrary distribution, and the Quadratic GAN would simply output the Gaussian distribution whose mean and covariance match the observed distribution. 

We summarize the minimax and maximin values in Table~\ref{tab:value} for the original Quadratic GAN in Theorem~\ref{thm:instability} and proposed Parameter Sharing GAN in Theorem~\ref{thm:stability}. One can see that for original Quadratic GAN, the duality gap is non-zero, while Parameter Sharing GAN circumvents this issue, resulting in global stability for alternative gradient descent.

\begin{table}[htb]
\caption{Minimax and maximin values of original Quadratic GAN in (\ref{optimization1}) for $r = 1$ and Parameter Sharing GAN in (\ref{eqn.parameter_sharing}).}
\centering
\label{tab:value}
%\wuhao[1.5]
\begin{tabular}{c|c|c}
\toprule[2pt]
& Original GAN for $r = 1$ & Parameter Sharing GAN  \\
\hline

%\midrule[1pt]
minimax & $\operatorname{Tr}(K)-\lambda_1$ & $\operatorname{Tr}(K)-\lambda_1$ \\
maximin &  0  & $\operatorname{Tr}(K)-\lambda_1$ \\

\bottomrule[1pt]
\end{tabular}
\end{table}

\subsection{Instability for optimization in degenerate case}\label{subsec.instability}

Let $\cG_r$ be defined in~(\ref{eqn.grdefcom}). Theorem~\ref{thm:new_convergence2} shows that the algorithm specified in (\ref{equ:project_weakened_W2_com}) below
\begin{align}\label{equ:project_weakened_W2_com}
\tilde g & = \argmin_{g(\cdot)\in \mathcal{G}_r} W_2(\bbP_{g(Z)}, \mathcal{N}( \hat{\mu}, \hat{\Sigma})) 
\end{align}
enjoys desirable statistical properties. Can this problem be solved efficiently numerically? First, we show two equivalent minimax formulation of~(\ref{equ:project_weakened_W2_com}) via the following theorem:
%, the details and proof of this theorem is included in Appendix \ref{appen:degenerate}.

\begin{theorem}
\label{thm:equ}
The GAN optimization problem in equations (\ref{equ:project_weakened_W2_com}) with $\hat \mu = 0$ can be cast as several equivalent forms. We can either do
\begin{align}
\label{optimization1}
\min_{U\geq 0, r(U)\leq r} \sup_{A>0} \operatorname{Tr}( (I-A)K + (I-A^{-1})U)
\end{align}
or do
\begin{align}
\min_{U\geq 0, r(U)\leq r} \sup_{A\geq 0, R(U)\subset R(A)} \operatorname{Tr}( (I-A)K + (I-A^{\dagger})U). 
\end{align}
Here $K = \hat{\Sigma}$, $A$ is the discriminator, and $U$ is the covariance matrix for the distribution of generator.
\end{theorem}

It follows from \cite[Pg. 129]{arrow1958studies} that when $r = d$, alternating gradient descent converges to the optimal solution $U^* = K$~\cite{feizi2017understanding}. However, one of the key features of GANs is that usually the generator distribution is supported on a low-dimensional manifold embedded in the space with high ambient dimensions. What happens when $r<d$? It is not difficult to see that alternate gradient descent is not globally convergent in this case \cite{feizi2017understanding}. We show here that one does not even have local asymptotic stability.  
\begin{theorem}
\label{thm:instability}
The optimization problem (\ref{optimization1}) is generally not locally asymptotically stable when $r(U) = 1$ using alternating gradient descent. 
\end{theorem}

As an illustrating example, we consider $d=2$, and decompose $U = vv^T$. We denote $A$ as $[a_{11}, a_{12}; a_{21}, a_{22}]$. Let $K$ be $[1,0;0,0]$. The corresponding optimal solution is $v=(1,0)$, $a_{11}=1$, $a_{12}=a_{22}=0$. 
Figure~\ref{fig:gradient_flow} shows that starting from the optimal solution pair $(A^*, U^*)$, if we let $A$ deviate from $A^*$ a little, $U$ would be immediately repulsed off $U^*$. Furthermore, we simulate the dynamics of both $A$ and $U$ after perturbation in Figure \ref{fig:dynamics_sim} under the same setting. It can be seen that $a_{22}$ is monotone increasing and other values are fluctuating around the equilibrium. Thus the system is not locally asymptotically stable.

%\begin{center}
\begin{figure}[!ht]
    \centering
    \includegraphics[width=0.6\linewidth]{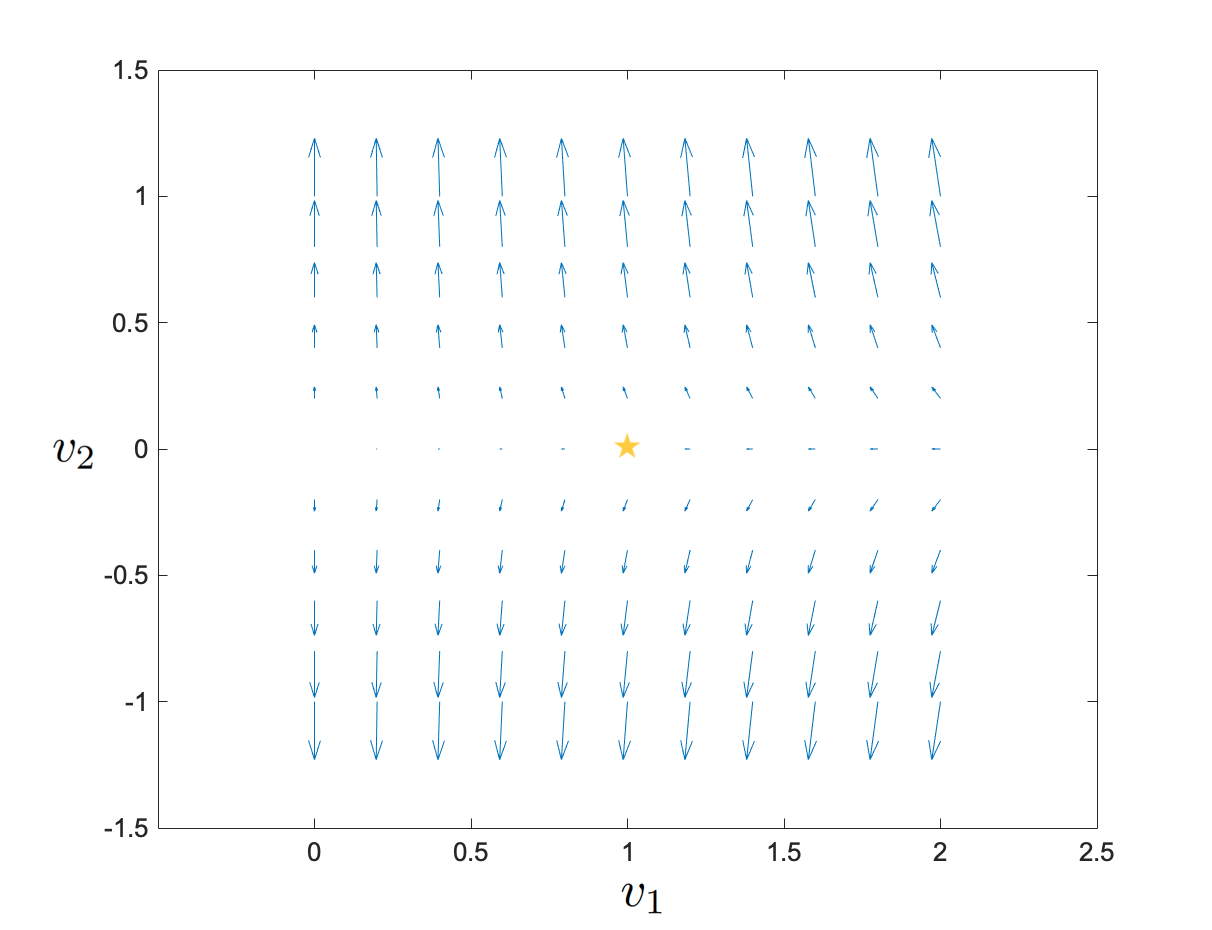}
    \caption{The illustration of the instability of alternating gradient descent in GAN optimization in Theorem \ref{thm:instability}. We add $0.1$ perturbation on both $a_{11}$and $a_{22}$, and plot the gradient flow for $v$. The star in the middle of the figure is the optimal point for $v$. As long as $A$ is perturbed a little, it would be repelled from the optimal point. This intuitively explains why alternating gradient descent is not \emph{locally asymptotically} stable.}
    \label{fig:gradient_flow}
\end{figure}
%\end{center}

\begin{figure}[!ht]
    \centering
    \includegraphics[width=0.67\linewidth]{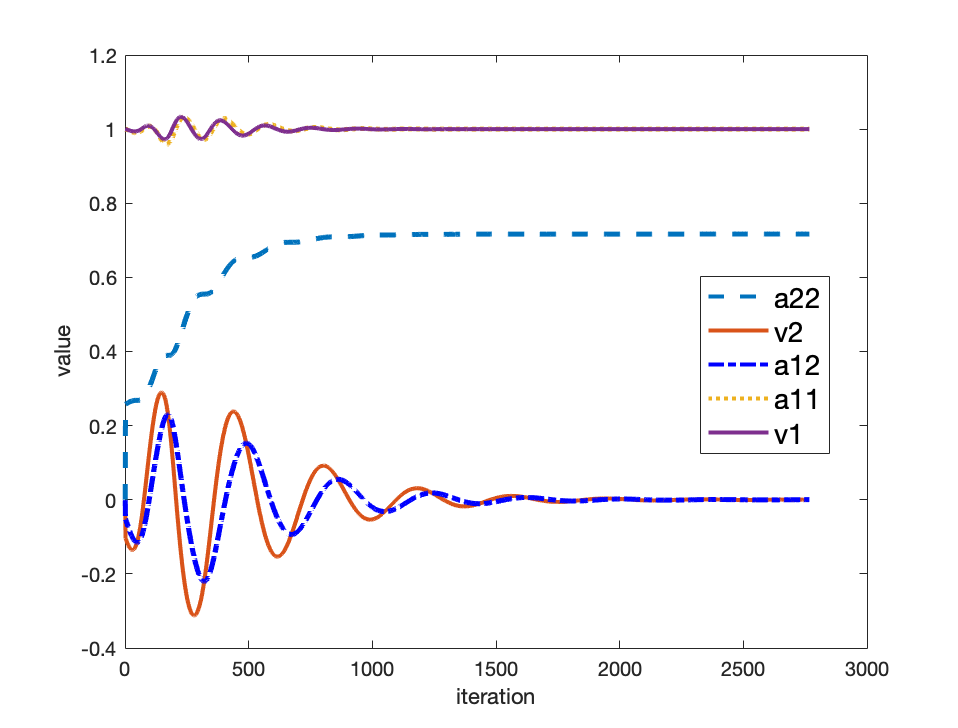}
    \caption{The simulation of the dynamics of alternating gradient descent in GAN optimization in Theorem \ref{thm:instability}.  We add $0.0001$ perturbation on $a_{11},a_{12},a_{22},v_{1}$ and $-0.0001$ perturbation on $v_{2}$. We plot the change of values with time. We can see $a_{22}$ is monotone increasing and other values are fluctuating around equilibrium. This shows that alternating gradient descent is not \emph{locally asymptotically} stable.}
    \label{fig:dynamics_sim}
\end{figure}

Our result provides a provable example that is of the same spirit as the example in~\cite{li2017towards} which showed experimentally that standard gradient dynamics of GAN to learn a mixture of Gaussian often fails to converge, while training generator with optimal discriminator would lead to convergence. It is also shown in~\cite{DBLP:journals/corr/abs-1804-05464} that employing a gradient-based learning algorithm for general-sum and potential games will avoid a non-negligible subset of the local Nash equilibrium, and will lead to convergence to non-Nash strategies.

\subsection{Stability results for Parameter Sharing GAN}

Ideally, in order to solve the minimax GAN optimization problem, we need to first solve the max and then the min, since minimax may not equal maximin in general. The failure of alternating gradient descent in Theorem~\ref{thm:instability} shows that treating the min and max optimization problems in an equal footing may cause instability problems. The idea of solving \emph{more} the maximization problem and \emph{less} the minimization problem is implicitly suggested in~\cite{sanjabi2018convergence} and the two time-scale update rule~\cite{heusel2017gans}. 

\begin{comment}
In reality, we are doing projection under the distance, and it may not make too much sense to think of this as a zero-sum game, but more of a Stackelberg game where there is a specific order of moving (the generator moves first, and the discriminator moves second, so the discriminator is in a stronger position).
\end{comment}

The next result shows that if we embed some knowledge of the optimal solution of the inner maximization problem, we can create a new minimax formulation with the following properties:
\begin{itemize}
    \item \textbf{Zero duality gap}: the minimax value equals the maximin value;
    \item \textbf{Relation to old minimax formulation:} the minimax value of the new formulation is equal to the minimax value of the old formulation \emph{without} changing the generator. 
    \item \textbf{Optimization stability:} AGD is \emph{globally asymptotically} stable in solving the $r$-PCA for $r = 1$. 
\end{itemize}

Naturally, one shall embed the connections between the optimal value $A^*$ and $U^*$ in the optimization problem. It can be shown that when $r = 1$, the optimal solutions in~(\ref{optimization1}) can be written in the following form:
\begin{align}
K & = Q \Sigma Q^T \label{eqn.eigendecompositionofK} \\
U^* & = \lambda_1 v_1 v_1^T \\
A^* & = v_1 v_1^T,
\end{align}
where~(\ref{eqn.eigendecompositionofK}) is the eigenvalue decomposition of $K$, $Q = [v_1,v_2,\ldots,v_d]$, and $\Sigma = \text{diag}(\lambda_1,\lambda_2,\ldots,\lambda_d)$ ranks the eigenvalues of $K$ in a descending order. In other words, $U^*$ is the $1$-PCA of $K$, and $A^*$ shares the linear subspace as $U^*$ with the eigenvalue being one. It naturally leads to the following \emph{parameter sharing} GAN architecture.  (\cite{feizi2017understanding} considered another parameter sharing architecture.)
\begin{theorem}
\label{thm:stability}
Alternating gradient descent is globally asymptotically stable in solving GAN in the parameter sharing formulation, where
\begin{align}
A & = \lambda v v^T, \\
U & = b v v^T,
\end{align}
and we use alternating gradient descent to solve
\begin{align}\label{eqn.parameter_sharing}
\min_{\|v\|=1,b\geq 0} \sup_{\lambda \geq 0} \operatorname{Tr}( (I-A)K + (I-A^{\dagger})U). 
\end{align}
\end{theorem}

\section{Acknowledgements}
We would like to thank Farzan Farnia and Soheil Feizi for helpful discussions at the initial stage of this project. We thank Luigi Ambrosio and Yihong Wu for discussions related to optimal transport, Jacob Steinhardt and Ilias Diakonikolas for discussions related to robust statistics. We are grateful to Chao Gao for stimulating discussions on the relationship between GAN and robust statistics at the workshop of robust and high-dimensional statistics held at the Simons Institute at the University of California, Berkeley in fall, 2018.   

\newpage

\setcounter{page}{1}
%\renewcommand{\baselinestretch}{0.}
%\begin{footnotesize}
\bibliographystyle{abbrv}
\bibliography{di}

\newpage

\appendix
\section{Definition of Terms}\label{appendix.def}

\begin{definition}[Pseudometric]
A non-negative real-valued function $d:X\times X \rightarrow\mathbb{R}^+$ is a pseudometric if for any $x,y,z \in X$, the following condition holds:
 \begin{align}
 & d(x,x)=0.\\
 & d(x,y) = d(y,x).\\
 & d(x,z)\leq d(x,y)+ d(y,z).
\end{align}
Unlike a metric, one may have $d(x,y)=0$ for distinct values $x\neq y$.
\end{definition}

\begin{definition}[Pseudonorm]
A real-valued function $\| \cdot\|:X \rightarrow\mathbb{R}$ is a pseudonorm if for any $x,y \in X$, the following condition holds:
 \begin{align}
 & \|x \|\geq0, x = 0 \Rightarrow \|x\|=0.\\
 & \| cx\|=|c|\|x\|.\\
 & \|x+y\|\leq \|x\|+\|y\|.
\end{align}
A pseudonorm differs from a norm in that the pseudonorm can be zero for nonzero vectors (functions).
\end{definition}
\begin{definition}[$\epsilon$-covering number~\cite{devroye2012combinatorial}]
Let $(V,\| \cdot \|)$ be a normed space, and $\cG \in V$. We say $\{V_1, ..., V_N \}$ is an $\epsilon$-covering of $\cG$ if $\forall g \in \cG$, $\exists i$ such that $\| g -  V_i\| \leq  \epsilon$.

\end{definition}
\begin{definition}[Equilibrium point~\cite{sastry2013nonlinear}]
Consider a nonlinear time-invariant system $\frac{dx}{dt} = f(x)$, where $f:\mathbb{R}^d \rightarrow \mathbb{R}^d$.  A point $x^*\in \mathbb{R}^d$ is an equilibrium point of the system if $f(x^*)=0$.
\end{definition}
\begin{definition}[Local stability~\cite{sastry2013nonlinear}]
An equilibrium point in a system is locally stable if any point initialized near the equilibrium point stays near that equilibrium point. Formally, we say that an equilibrium point $x^*$ is locally stable if for all $\epsilon > 0$, there exists a $\delta > 0$ such that
\begin{align}
    \|x(0) - x^* \|< \delta \Rightarrow   \|x(t) - x^* \|< \epsilon,  \forall t>0
\end{align}
\end{definition}

\begin{definition}[Local asymptotic stability~\cite{sastry2013nonlinear}]
An equilibrium point $x^*$ in a system is locally asymptotically stable if it is locally stable, and there exists a $\delta>0$, such that 
\begin{align}
    \|x(0) - x^* \|< \delta \Rightarrow  \lim_{t\rightarrow \infty} \|x(t) - x^* \|=0
\end{align}
\end{definition}

\begin{definition}[Global asymptotic stability~\cite{sastry2013nonlinear}]
An equilibrium point $x^*$ is globally asymptotically stable if it is locally stable, and $\lim_{t\rightarrow \infty}\|x(t)-x^*\|=0$ for all $x(0)\in\mathbb{R}^d$.
\end{definition}

\section{Proof of Lemmas and Theorems}

\subsection{Proof of Lemma~\ref{lemma.pseudonormrepres}}\label{proof.pseudonormrepres}

It follows from the definition of the pseudonorm that $\|x\|_L$ is a sublinear function on the vector space. It follows from the definition of $\|f\|_{L^*}$ that for any $f$ such that $\|f\|_{L^*} = 1$, we have
\begin{align}
    \int f d\mu \leq \| \mu \|_L. 
\end{align}
Hence, it suffices to show that for every $\mu$, we can indeed find some $f$ such that $\| f\|_{L^*} = 1$ and $\int f d\mu = \|\mu\|_L$. 

For any $\mu$, consider a continuous linear functional $M$ defined on the one-dimensional subspace $\{\alpha \cdot \mu: \alpha \in \mathbb{R}\}$, such that $M(\alpha \mu ) = \alpha \| \mu \|_L$. Clearly, we have 
\begin{align}
M(\alpha \mu) \leq \| \alpha \mu \|_L = |\alpha| \| \mu\|_L.
\end{align}
It follows from the Hahn--Banach theorem~\cite[Theorem 4.2]{kreyszig1978introductory} that we can extend $M$ to the whole vector space to obtain another linear functional $T$ such that $T(\mu) = \| \mu\|_L$, and the $T(\mu') \leq \|\mu'\|_L$ for any $\mu'$. The proof is completed by noting that we have assumed any continuous linear functional can be written as $\int fd\mu$ for some measurable $f$.

\subsection{Proof of Theorem~\ref{thm.populationinsights}}\label{proof.populationinsights}

\begin{proof}
It follows from the definition of $OPT$ in~(\ref{eqn.optdef}) that there exists some $g^*$ such that
\begin{align}
    OPT = L(\bbP_{g^*(Z)}, \bbP_X). 
\end{align}
We have
\begin{align}
    L(\bbP_{g'(Z)}, \bbP_X) & \leq L(\bbP_{g'(Z)}, \bbP_{g^*(Z)}) + L(\bbP_{g^*(Z)}, \bbP_X) \label{eqn.trianglethm1} \\
    & \leq c_2 + c_1 (L'(\bbP_{g'(Z)}, \bbP_{g^*(Z)}) - L'(\bbP_{g^*(Z)}, \bbP_{g^*(Z)})) + OPT \label{eqn.resolutionthm1} \\
    & \leq c_2 + c_1 \left ( L'(\bbP_{g'(Z)}, \bbP_X) - L'(\bbP_{g^*(Z)}, \bbP_{g^*(Z)}) + OPT \right ) + OPT \label{eqn.robustthm1} \\
    & = c_1 (L'(\bbP_{g'(Z)}, \bbP_X)- L'(\bbP_{g^*(Z)}, \bbP_{g^*(Z)})) + (1+c_1) OPT + c_2   \label{eqn.refthm1}\\
    & \leq c_1 (L'(\bbP_{g^*(Z)}, \bbP_X)- L'(\bbP_{g^*(Z)}, \bbP_{g^*(Z)})) + (1+c_1) OPT + c_2 \label{eqn.defofgthm1} \\
    & \leq c_1 (L'(\bbP_{g^*(Z)}, \bbP_{g^*(Z)}) +  OPT- L'(\bbP_{g^*(Z)}, \bbP_{g^*(Z)})) + (1+c_1 ) OPT + c_2 \label{eqn.robust2thm1} \\
    & = (1+ 2c_1)OPT + c_2 . \label{eqn.zerothm1}
\end{align}

Concretely, (\ref{eqn.trianglethm1}) follows from the triangle inequality of $L$, (\ref{eqn.resolutionthm1}) follows from the first property of admissibility of $L'$, (\ref{eqn.robustthm1}) follows from the second property of admissibility and the fact that $L'' \leq L$, (\ref{eqn.defofgthm1}) follows from the definition of $g'$, (\ref{eqn.robust2thm1}) follows from the second property of admissibility. 

For the estimator $g''$, following similar lines of arguments in equation (\ref{eqn.refthm1}), we have
\begin{align}
    L(\bbP_{g''(Z)}, \bbP_X) & \leq c_1 (L'(\bbP_{g''(Z)}, \bbP_X)- L'(\bbP_{g^*(Z)}, \bbP_{g^*(Z)})) + (1+c_1 ) OPT + c_2  \nonumber\\
    & \leq c_1 (L'(\bbP_{g''(Z)}, \hat{\bbP}_X^n) +  L''(\bbP_X,  \hat{\bbP}_X^n) - L'(\bbP_{g^*(Z)}, \bbP_{g^*(Z)})) + (1+c_1 ) OPT + c_2 \nonumber \\
    & \leq c_1 (L'(\bbP_{g^*(Z)}, \hat{\bbP}_X^n) +  L''(\bbP_X,  \hat{\bbP}_X^n) - L'(\bbP_{g^*(Z)}, \bbP_{g^*(Z)}) ) + (1+c_1 ) OPT + c_2 \nonumber \\
    & = c_1 L'(\bbP_{g^*(Z)}, \hat{\bbP}_X^n) + c_1  L''(\bbP_X,  \hat{\bbP}_X^n ) - c_1 L'(\bbP_{g^*(Z)}, \bbP_{g^*(Z)}) + (1+c_1 ) OPT + c_2 \nonumber \\
    & \leq c_1 (L'(\bbP_{g^*(Z)}, \bbP_X) +  L''(\bbP_X,  \hat{\bbP}_X^n)  ) + c_1  L''(\bbP_X,  \hat{\bbP}_X^n) - c_1 L'(\bbP_{g^*(Z)}, \bbP_{g^*(Z)}) + (1+c_1 ) OPT + c_2 \nonumber \\
    & = c_1 L'(\bbP_{g^*(Z)}, \bbP_X) + 2c_1 L''(\bbP_X,  \hat{\bbP}_X^n) - c_1 L'(\bbP_{g^*(Z)}, \bbP_{g^*(Z)}) + (1+c_1 ) OPT + c_2 \nonumber \\
    & \leq c_1(L'(\bbP_{g^*(Z)}, \bbP_{g^*(Z)}) +  OPT ) + 2c_1 L''(\bbP_X,  \hat{\bbP}_X^n) - c_1 L'(\bbP_{g^*(Z)}, \bbP_{g^*(Z)})  + (1+c_1 ) OPT + c_2 \nonumber \\
    & = (1+2c_1)OPT + 2c_1 L''(\bbP_X,  \hat{\bbP}_X^n) + c_2 . 
\end{align}
\end{proof}

\subsection{Proof of Theorem \ref{thm:new_convergence}}\label{proof.new_convergence}
We first show that the first property is satisfied with $c_1 = 1, c_2 = 4\epsilon$. Apparently, $L'(\bbP_X, \bbP_X) = 0$ for any $\bbP_X$. Indeed, for any $\tilde g, g^* \in \cG$, by the definition of $\cG_\epsilon$, $\exists g_1, g_2 \in \cG_\epsilon$, such that 
\begin{equation}
L(\bbP_{\tilde g(Z)}, \bbP_{g_1(Z)}) \leq \epsilon, L(\bbP_{g_2(Z)}, \bbP_{g^*(Z)}) \leq \epsilon.
\end{equation}
Thus by triangle inequality
\begin{align}
%\forall g_1, g_2 \in \mathcal{G}_\epsilon, 
L(\bbP_{\tilde g(Z)}, \bbP_{g^*(Z)})  & \leq L(\bbP_{\tilde g(Z)}, \bbP_{g_1(Z)}) + L( \bbP_{g_1(Z)}, \bbP_{g_2(Z)}) + L(\bbP_{g_2(Z)}, \bbP_{g^*(Z)}) \nonumber \\
& \leq L(\bbP_{g_1(Z)}, \bbP_{g_2(Z)}) + 2\epsilon \nonumber \\
& = L'(\bbP_{g_1(Z)}, \bbP_{g_2(Z)}) + 2\epsilon \nonumber \\
& \leq L'(\bbP_{g_1(Z)}, \bbP_{\tilde g(Z)}) + L'(\bbP_{\tilde g(Z)}, \bbP_{g^*(Z)})  + L'(\bbP_{g^*(Z)}, \bbP_{g_2(Z)})  + 2\epsilon \nonumber \\
& \leq L'(\bbP_{\tilde g(Z)}, \bbP_{g^*(Z)}) + 4\epsilon,
 \end{align}
 where in the last inequality we used the fact that $L'\leq L$. 
 
The second property of admissibility is satisfied since $L'$ satisfies the triangle inequality and we can take $L'' = L'$.  Applying Theorem~\ref{thm.populationinsights}, we have
\begin{align}
    L(\bbP_{\tilde g(Z)}, \bbP_X) \leq 3 \cdot OPT + 2 L'(\bbP_X, \hat{\bbP}_X^n) + 4\epsilon
\end{align}

In order to bound $L'(\bbP_X, \hat{\bbP}_X^n) $, we write it as 
\begin{equation}
L'(\bbP_X, {\hat \bbP_X^n}) = \sup_{f\in \cF_\epsilon} \bbE [f(X)]- \frac{1}{n} \sum_{i=1}^n f(X_i).
\end{equation}

Suppose there exists a constant $B$ such that for any $f\in \cF_\epsilon$, 
\begin{align}
\bbP (\bbE [f(X)] - \frac{1}{n} \sum_{i=1}^n f(X_i) > t) \leq 2 \exp(-\frac{nt^2}{2B^2}).
\end{align}

Using union bound, we can get
\begin{align}
\bbP ( \sup_{f\in \cF_\epsilon} \bbE [f(X) - \frac{1}{n} \sum_{i=1}^n f(X_i)] > t) & \leq 2 |\cF_\epsilon| \exp(-\frac{nt^2}{2B^2}) \nonumber \\
& \leq 2N_\epsilon^2 \exp(-\frac{nt^2}{2B^2}) \nonumber \\
& \leq 2C^2 \left (\frac{1}{\epsilon} \right )^{2d}\exp(-\frac{nt^2}{2B^2}),
\label{equ:thm3union}
\end{align}
since we have assumed that $N_\epsilon \leq C \left ( \frac{1}{\epsilon} \right )^d$. Denoting the right hand side in equation (\ref{equ:thm3union}) as $\delta$, we have 
\begin{align}
t = \sqrt{\frac{2B^2}{n}\log(\frac{2C^2}{\delta\epsilon^{2d}})}.
\end{align}
Then with probability at least $1-\delta$, we have
\begin{align}
L(\bbP_{\tilde g(Z)}, \bbP_X) \leq 3 \cdot OPT + 2\sqrt{\frac{2B^2}{n}\log(\frac{2C^2}{\delta\epsilon^{2d}})} + 4\epsilon .
\end{align}
Taking $\epsilon = (\frac{d\log n}{n})^\frac{1}{2}$, we can derive 
\begin{align}
L(\bbP_{\tilde g(Z)}, \bbP_X) \leq 3 \cdot OPT + D(B,C) \left (\frac{1}{n}\max \left (d \log( n), \log \frac{1}{\delta} \right )\right )^\frac{1}{2}.%((\frac{d \log n}{n})^{\frac{1}{2}})
\end{align}

\subsection{Proof of Theorem \ref{thm:yatracos_weakened}}\label{proof.yatracos}

From the definition of $\mathsf{TV}'$ in~(\ref{equ:def_yatracos}), there is $\mathsf{TV}'(\bbP_X, \bbP_X) = 0$ for any $\bbP_X$. By definition of total variation distance, we have for $g_1, g_2\in \cG$, 
\begin{align}
    \mathsf{TV}'(\bbP_{g_1},\bbP_{g_2}) =  \mathsf{TV}(\bbP_{g_1},\bbP_{g_2}),
\end{align}
since the optimal set to distinguish between $\bbP_{g_1}$ and $\bbP_{g_2}$ is a halfspace. Thus the first property is satisfied with $c_1 = 1, c_2 = 0$. 
 
The second property of admissibility is satisfied since $\mathsf{TV}'$ satisfies the triangle inequality and we can take $L'' = \mathsf{TV}'$.  Applying Theorem~\ref{thm.populationinsights}, we have
\begin{align}
    \mathsf{TV}(\bbP_{\tilde g}, \bbP_X)  & \leq 3 \cdot \inf_{g\in\cG}\mathsf{TV}(\bbP_g, \bbP_X) + 2  \mathsf{TV}'(\bbP_X, \hat{\bbP}_X^n) \nonumber \\
    & \leq  3\cdot OPT + 2  \mathsf{TV}'(\bbP_X, \hat{\bbP}_X^n). \label{equ.yatracos_pf1}
\end{align}

\begin{comment}
Thus
\begin{align}\label{equ.yatracos_pf1}
    \mathsf{TV}(\bbP_{\tilde g}, \bbP_{g^*}) & \leq \mathsf{TV}(\bbP_{\tilde g}, \bbP_X) + \mathsf{TV}(\bbP_X, \bbP_{g^*}) \nonumber \\
    &\leq 4\cdot OPT +  2  \mathsf{TV}'(\bbP_X, \hat{\bbP}_X^n)
\end{align}
\end{comment}

Since the class $\mathcal{A} = \{ \{\mathbf{x}\in\mathbb{R}^d:\left <\mathbf{v},\mathbf{x}\right>+\mathbf{b}>0\}:\|\mathbf{v}\|_2 =1,\mathbf{b}\in\mathbb{R}^d\}$ is all halfspaces in $\mathbb{R}^d$, it has VC dimension $d+1$~\cite{vapnik2015uniform}. By VC inequality~\cite[Theorem 2.2, Chap. 4]{devroye2012combinatorial}, regardless of the distribution $\bbP_X$, we have with probability at least $1-\delta$, there exists some absolute constant $C$, such that 
\begin{align}\label{equ.yatracos_pf2}
    \mathsf{TV}'(\bbP_X, \hat{\bbP}_X^{n}) & \leq C\cdot \sqrt{\frac{d+1}{n}} + \sqrt{\frac{
    \log(1/\delta)}{2n}}.
\end{align}

Combining  equation (\ref{equ.yatracos_pf1}) and (\ref{equ.yatracos_pf2}), we can derive with probability at least $1-\delta$, there exists some universal constant $C$, such that 
\begin{align}
    \mathsf{TV}(\bbP_{\tilde g}, \bbP_{X})\leq 3\cdot OPT +  C\cdot \sqrt{\frac{d+1}{n}}+ 2\sqrt{\frac{\log(1/\delta)}{2n}}.
\end{align}
\subsection{Proof of Theorem \ref{thm:equ:tukeymedian}}\label{proof.tukey}

We first show the first property is satisfied. From the definition of $L_{\mathsf{Tukey}}$ in~(\ref{equ:tukeymedian_def}), there is $L_{\mathsf{Tukey}}(\bbP_X, \bbP_X) = 0 $ for any $\bbP_X \in \cG$ since each distribution in $\cG$ is symmetric.  For any two distribution $\bbP_1 = \mathcal{N}(\mu_1, I), \bbP_2  = \mathcal{N}(\mu_2,I)$, it follows from straightforward computation that
\begin{align}
\mathsf{TV}(\bbP_1,\bbP_2) = 2 \Phi\left( \frac{\| \mu_1 - \mu_2\|_2}{2} \right) -1. 
\end{align} 
We also have
\begin{align}
L_{\mathsf{Tukey}}(\bbP_1,\bbP_2) =  \Phi\left( \| \mu_1 - \mu_2 \|_2 \right) -1/2. 
\end{align}
Here $\Phi(x) = \int_{-\infty}^x \phi(x) dx$ is the standard normal CDF. We notice that for any $\bbP_1,\bbP_2\in \mathcal{G}$, we have
\begin{align}
L_{\mathsf{Tukey}}(\bbP_1,\bbP_2) \geq \frac{1}{2} \mathsf{TV}(\bbP_1,\bbP_2),
\end{align}
which follows from the fact that $\Phi(x)$ as a function on $[0,\infty)$ is non-decreasing. Hence the first property is satisfied with $c_1=2, c_2=0$. 

As for the second property, we notice that 
\begin{align}
    |L_{\mathsf{Tukey}}(\bbP_g,\bbP_1) - L_{\mathsf{Tukey}}(\bbP_g,\bbP_2)  | & = |\sup_{\mathbf{v}\in \mathbb{R}^d} \bbP_1\left( (X-\mathbb{E}_{\bbP_g}[X])^T \mathbf{v} >0 \right) - \sup_{\mathbf{v}\in \mathbb{R}^d} \bbP_2\left( (X-\mathbb{E}_{\bbP_g}[X])^T \mathbf{v} >0 \right)  | \nonumber \\
    & \leq \sup_{ A \in \{ \{x: (x - b)^T v >0\}: v,b\in \mathbb{R}^d\}} \bbP_1(A) - \bbP_2(A) \nonumber \\
    & = \mathsf{TV}'(\bbP_1, \bbP_2),
\end{align}
where $\mathsf{TV}'(\bbP_1, \bbP_2)$ is defined in equation (\ref{equ:def_yatracos}). Thus $L_{\mathsf{Tukey}}$ satisfies the second property with $L'' = \mathsf{TV}'$. 

Applying Theorem \ref{thm.populationinsights} we can derive 
\begin{align}\label{equ.tukey_pf1}
    \mathsf{TV}(\bbP_{\tilde g}, \bbP_X)  & \leq 5 \cdot OPT + 4  \mathsf{TV}'(\bbP_X, \hat{\bbP}_X^n).
\end{align}
It follows from similar arguments to that in the proof of Theorem~\ref{thm:yatracos_weakened} that there exists an absolute constant $C>0$ such that with probability at least $1-\delta$, 
\begin{align}
    \mathsf{TV}(\bbP_{\tilde g}, \bbP_{X})\leq 5\cdot OPT +  C\cdot  \sqrt{\frac{d+1}{n}}+ 4\sqrt{\frac{\log(1/\delta)}{2n}}.
\end{align}

\subsection{Proof of Theorem \ref{thm:naive_convergence_w2}}\label{proof.w2_weakened}

\subsubsection{Concentration for $W_2$ distance for distributions that satisfy logarithmic Sobolev inequality}
Following the similar idea in \cite[Corollary 5.5]{gozlan2010transport}, we can prove the following lemma.
\begin{lemma}
\label{lem:lipschitz}
For any measure $\mu$ on Polish space $\mathcal{X}$,
consider random variable $X\sim \mu$ and corresponding empirical samples  $x\in \mathcal{X}^n$, with each $x_i$ sampled i.i.d. from distribution $\mu$. The map 
$x \rightarrow W_2(\bbP_X, \hat \bbP_X^n)$ is $\frac{1}{\sqrt{n}}$-Lipschitz under Euclidean distance, i.e.,
viewing $W_2(\bbP_X, \hat \bbP_X^n)$ as a function $W(x) $, %of i.i.d. samples $X_1, X_2, \cdots, X_n \sim \mathcal{N}(0,I)$, 
for different sets of samples $x, x'$
\begin{align}
\|W(x) - W(x')\| \leq \frac{1}{\sqrt{n}}\|x-x'\|.
\end{align}
\end{lemma}

\begin{proof}
By triangle inequality of the $W_2$ distance
\begin{align}
\|W(x) - W(x')\| \leq W_2(\hat  \bbP_X^n, \hat \bbP_{X'}^n).
\end{align}
We know that the Wasserstein distance is related to Talagrand function~\cite{gozlan2010transport} as follows:
\begin{align}
W_2(\mu, v) & = \mathcal{T}_2(\mu, v)^\frac{1}{2} \nonumber \\
& = \inf_{\pi} \left \{ \int_{\mathcal{X}^2} |x - y|^2d\pi(x,y);\pi(x,y)\in \bbP(\mathcal{X}^2);\pi(x)=\mu,\pi(y)=v\right \}^\frac{1}{2}.
\end{align}
According to the convexity of $\mathcal{T}_2(\cdot,\cdot)$ \cite[Theorem 4.8]{villani2008optimal}, one has 
\begin{align}
W_2(\hat \bbP_X^n, \hat \bbP_{X'}^n) & = (\mathcal{T}_2(\hat \bbP_X^n, \hat \bbP_{X'}^n))^\frac{1}{2} \nonumber \\
& \leq (\frac{1}{n}\sum_{i=1}^{n} \mathcal{T}_2(\delta(x_i), \delta(x_i'))^\frac{1}{2} \nonumber \\
& = \frac{1}{\sqrt{n}}(\sum_{i=1}^n (x_i-x_i')^2)^\frac{1}{2} \nonumber \\
& = \frac{1}{\sqrt{n}}\|x-x'\|.
\end{align}
Thus we have 
\begin{align}
\|W(x) - W(x')\| \leq \frac{1}{\sqrt{n}}\|x-x'\|.
\end{align}
\end{proof}

Under the same setting, we can derive the concentration property of $W_2$ distance between any distribution that satisfies logarithmic Sobolev inequality and its empirical distribution as follows.

\begin{lemma}
\label{lem:concentration}
Let $\mu$ be a probability measure on $\mathbb{R}^n$ such that for some constant $B>0$ and all smooth $f$ on $\mathbb{R}^n$,
\begin{align}\label{eqn.appendix_lsi}
    \operatorname{Ent}(f^2)\leq 2B\mathbb{E}(|\nabla f|^2),
\end{align}
where $\nabla f$ is the usual gradient of $f$, $\operatorname{Ent}(f) = \bbE_\mu(f\log f)-\bbE_\mu(f)\log\bbE_\mu(f)$. Assume $X\sim \mu$, 
the Wasserstein distance $W_2(\bbP_X, \hat \bbP_X^n)$ satisfies
\begin{align}
\label{equ:concentration}
\bbP(W_2(\bbP_X, \hat \bbP_X^n)\geq \bbE[W_2(\bbP_X, \hat \bbP_X^n)] + t) \leq \exp(-nt^2/2B).
\end{align}
\end{lemma}
From Lemma \ref{lem:lipschitz}, we know that $W_2(\bbP_X, \bbP_X^n)$ is a Lipschitz function for samples $x$. Thus by the Herbst argument~\cite[2.3]{ledoux1999concentration}, we can directly derive the inequality (\ref{equ:concentration}).
\subsubsection{Proof for upper bound}

Denote 
\begin{align}
    g^* = \argmin_{g \in \cG} W_2^2(\bbP_{g(Z)}, {\bbP}_X).
\end{align}

By the triangle inequality, we have for any $\bbP_X$,
\begin{align}
W_2(\bbP_X, \bbP_{\tilde{g}(Z)}) & \leq 
W_2(\bbP_X, {\hat \bbP_X^n}) + W_2(\bbP_{\tilde{g}(Z)}, {\hat \bbP_X^n}) \nonumber \\
& \leq W_2(\bbP_X, {\hat \bbP_X^n}) + W_2(\bbP_{{g^*}(Z)}, {\hat \bbP_X^n}) \nonumber \\
& \leq 2W_2(\bbP_X, {\hat \bbP_X^n}) + W_2(\bbP_{{g^*}(Z)}, \bbP_X) \nonumber \\
& = \inf_{g\in \cG}W_2(\bbP_{{g}(Z)}, \bbP_X) + 2W_2(\bbP_X, {\hat \bbP_X^n}).
\end{align}

We can upper bound the term $W_2(\bbP_X, {\hat \bbP_X^n})$ given the following two assumptions:

\begin{itemize}
    \item $X\sim\mu$ satisfies logarithmic Sobolev inequality in (\ref{eqn.appendix_lsi}).
\item Upper bound on expected convergence rate for some constant $C_1$ depending on distribution $\bbP_X$: \begin{align}\label{eqn.appendix_thm5assumption2}
    \bbE [W_2^2(\bbP_X, {\hat \bbP_X^n})] \leq C_1(\bbP_X)\cdot n^{-2/d}.
\end{align}
\end{itemize}

The first assumption would lead to concentration inequality (\ref{equ:concentration}) by Lemma (\ref{lem:lipschitz}). From (\ref{eqn.appendix_thm5assumption2}), by Jensen's inequality, we can show that 
\begin{align}
    \bbE [W_2(\bbP_X, {\hat \bbP_X^n})] \leq  \sqrt{C_1(\bbP_X)}\cdot n^{-1/d}.
\end{align}

Taking right-hand side in equation (\ref{equ:concentration}) as $\delta$ and solving for $t$, we have with probability at least $1-\delta$, there exists some constant $C_1$ depending on $\bbP_X$, such that 
\begin{align}
    W_2(\bbP_X, {\hat \bbP_X^n}) \leq C_1(\bbP_X)\cdot n^{-1/d} + \sqrt{\frac{2B\log(1/\delta)}{n}}.
\end{align}
Now we can derive the conclusion: for any distribution satisfying (\ref{eqn.appendix_lsi}) and (\ref{eqn.appendix_thm5assumption2}), there exists some constant $C_1(\bbP_X)$, such that with probability at least $1-\delta$, we have
\begin{align}
     W_2(\bbP_X, \bbP_{\tilde g(Z)}) \leq \inf_{g\in \cG}W_2(\bbP_{{g}(Z)}, \bbP_X) + C_1(\bbP_X)\cdot n^{-1/d}  + 2\sqrt{\frac{2B\log(1/\delta)}{n}}.
\end{align}

\subsubsection{Proof of lower bound for Gaussian distribution}

We first lower bound the term $W_2(\mathcal{N}(\mathbf{0}, I), {\hat \bbP_X^n})$. We bound $W_2$ distance by $W_1$ distance by H\"older's inequality and then use the Kantorovich duality:
\begin{align}
W_2(\mathcal{N}(\mathbf{0}, I), {\hat \bbP_X^n})& \geq W_1(\mathcal{N}(\mathbf{0}, I), {\hat \bbP_X^n}) \nonumber \\
& = \sup_{f:\|f\|_L\leq 1}(\bbE f(X) - \bbE f(Y)),
\end{align} 
where $X\sim \mathcal{N}(\mathbf{0}, I)$, $Y \sim {\hat \bbP_X^n}$. We take a specific Lipschitz-1 function to further lower bound it: 
\begin{align}
    f(\mathbf{x}) = \min_i\|\mathbf{x}-X_i\|,
\end{align}
where $i\in [n]$ and $X_i$ is the $i$-th sample. We know that $\bbE f(Y) = 0$ since $Y$ is empirical distribution, thus
\begin{align}
W_2(\mathcal{N}(\mathbf{0}, I), {\hat \bbP_X^n})& \geq 
\bbE_{X\sim \mathcal{N}(\mathbf{0}, I)} \min_i\|X-X_i\|.
\end{align} 

This is the nearest neighbor distance for Gaussian distribution, which is extensively studied in the literature~\cite{wade2007explicit, penrose2011laws, liitiainen2011asymptotic}. We know that there exist constants $C_2(d), C_3$, such that when $n>C_2(d)$ and 
$d>1$, with probability at least $0.999$, 

\begin{align}
   \bbE_{X\sim \mathcal{N}(\mathbf{0}, I)} [\min_i\|X-X_i\|] & \geq C_3 \frac{\Gamma(1+\frac{1}{d})}{V_d^{1/d}}\int_{\mathbb{R}^n} f^{1-1/d}(
   \mathbf{x})d\mathbf{x} \cdot n^{-1/d} \nonumber \\
   & \geq C_3' d^{1/2} n^{-1/d}
\end{align}
as $n\to \infty$. 

Thus, we can conclude that for $n>C_2(d)$ and $d>1$, there exist a constant $C_3$, such that with probability at least $0.999$,
\begin{align}
      W_2(\mathcal{N}(\mathbf{0}, I), {\hat \bbP_X^n}) \geq C_3 d^{1/2} n^{-1/d}.
\end{align}

Note that when $d=1$, it is known from~\cite[Corollary 6.14]{bobkov2014one} that the convergence rate of $\mathbb{E}[ W_2(\mathcal{N}(\mathbf{0}, I), {\hat \bbP_X^n}) ]$ is $\Theta\left (\frac{\log \log n}{n}\right )^{1/2}$. The exact convergence rate for the case of $d = 2$ remains open. 

Before proving the exponential convergence rate lower bound for the Gaussian distribution, we first consider a simpler case which searches the generator in a subset of the linear generator family, $\mathcal{G} = \{\mathcal{N}(\mathbf{0}, c\mathbf{I}_d): c \in \mathbb{R}\} $, with real distribution $ \mathcal{N}(\mathbf{0}, \mathbf{I}_d)$. Naive estimator aims at looking for $c$ that minimizes the $W_2$ distance between empirical distribution and generator function:
\begin{equation}
\label{equ:csearch}
\hat c = \argmin_{c \geq 0} W_2^2(\mathcal{N}(\mathbf{0}, c\mathbf{I}_d), {\hat \bbP_X^n}).
\end{equation}

Assume $\hat X \sim {\hat \bbP_X^n}$, $Z \sim \mathcal{N}(\mathbf{0}, \mathbf{I}_d)$, denote $\rho_n = \sup_{\pi_Z =\mathcal{N}(\mathbf{0}, \mathbf{I}_d), \pi_{\hat{X}} =  {\hat \bbP_X^n}} \mathbb{E}[Z^T\hat X] $. The $W_2^2$ distance in equation (\ref{equ:csearch}) is given by
\begin{align}
W_2^2(\mathcal{N}(\mathbf{0}, \hat{c}\mathbf{I}), {\hat \bbP_X^n}) & = \min_{c\geq 0}W_2^2(\mathcal{N}(\mathbf{0}, c\mathbf{I}), {\hat \bbP_X^n}) \nonumber \\
& = \min_{c\geq 0} \mathbb{E}\|Z\|_2^2 c^2  + \mathbb{E}\| \hat X \|_2^2 - 2 \sup_\pi \mathbb{E}[Z^T\hat X]c \nonumber \\
& = \min_{c\geq 0} dc^2 + \frac{1}{n}\sum_{i=1}^n \|x_i\|^2  - 2 \rho_n c \nonumber \\
& = \frac{1}{n}\sum_{i=1}^n \|x_i\|^2 + \min_{c\geq 0} dc^2  - 2 \rho_n c \nonumber \\
& = \frac{1}{n}\sum_{i=1}^n \|x_i\|^2 -  \frac{\rho_n ^2}{d}.
\end{align}

The optimal $c$ that minimizes the distance between empirical distribution and real distribution is
\begin{equation}
\label{equ:optimalc}
\hat c = \frac{\rho_n}{d}. 
\end{equation}

Furthermore, we know that 
\begin{align}
W_2^2(\mathcal{N}(\mathbf{0}, I), {\hat \bbP_X^n}) & = \mathbb{E}\|Z\|_2^2 + \mathbb{E}\| \hat X \|_2^2 - 2 \sup_\pi \mathbb{E}[Z^T\hat X] \nonumber \\
& = d + \frac{1}{n}\sum_{i=1}^n \|x_i\|^2 - 2 \rho_n.
\label{equ:empiricalandreal}
%& = 2(d - \rho_n)
\end{align}

We denote $A =\frac{1}{n}\sum_{i=1}^n \|x_i\|^2$, which satisfies $\mathbb{E}[A] = d$ and $\mathsf{Var}(A) = \frac{2d}{n}$. We note that for $d\geq 5$, we have $d n^{-2/d} \gg \sqrt{\frac{d}{n}}$ as $n\to \infty$. Hence, with probability at least $0.99$, there exists some $C_2(d)>0, C_3>0, C_4>0$ such that for $n>C_2(d)$, $d\geq 5$, we have $A\leq d + C_3 \sqrt{\frac{d}{n}}, d-\rho_n \geq C_4 d n^{-2/d}$. 

For the lower bound of $ W_2(\mathcal{N}(\mathbf{0}, I), \bbP_{\tilde{g}(Z)})$, we use triangle inequality and the fact that $W_2(\bbP_{\tilde{g}(Z)}, {\hat \bbP_X^n}) \leq W_2(\mathcal{N}(\mathbf{0}, \hat cI), {\hat \bbP_X^n})$. With probability at least 0.99, for $n>C_2(d)$ and $d\geq 5$, 
\begin{align}
\label{equ:polydist}
 W_2(\mathcal{N}(\mathbf{0}, I), \bbP_{\tilde{g}(Z)})& \geq
 W_2(\mathcal{N}(\mathbf{0}, I), {\hat \bbP_X^n}) - W_2(\bbP_{\tilde{g}(Z)}, {\hat \bbP_X^n}) 
 \nonumber \\
 & \geq
 W_2(\mathcal{N}(\mathbf{0}, I), {\hat \bbP_X^n}) - W_2(\mathcal{N}(\mathbf{0}, \hat cI), {\hat \bbP_X^n}) \nonumber \\
 & = \sqrt{d+A-2\rho_n} - \sqrt{A-\frac{\rho_n^2}{d}} \nonumber \\
 & = \frac{d+A-2\rho_n - (A-\frac{\rho_n^2}{d})}{\sqrt{d+A-2\rho_n} + \sqrt{A-\frac{\rho_n^2}{d}}}
 \nonumber \\
& = \frac{(d - \rho_n)^2}{d (\sqrt{d+A-2\rho_n} + \sqrt{A-\frac{\rho_n^2}{d}}) } \\
& = \frac{(d-\rho_n)^{3/2}}{d}  \frac{\sqrt{d-\rho_n}}{\sqrt{d+A-2\rho_n} + \sqrt{A-\frac{\rho_n^2}{d}}} \\
& \geq C_5 d^{1/2}n^{-3/d}.
\end{align}

\subsection{Proof of Theorem \ref{thm:new_convergence2}}
\label{app:pf_thm3}
%\subsubsection{Detailed statement of Brenier's Theorem}
\subsubsection{Uniform continuity of matrix square root operator in operating norm}
We need the following lemma for the proof of main theorem.
\begin{lemma}\label{lem:squareroot}
\cite{math1940274} For any symmetric positive definite matrices $\mathbf{A}$, $\mathbf{B}$, we have
\begin{align}
\label{equ:uniform_continuous}
    \|\mathbf{A}^\frac{1}{2} - \mathbf{B}^\frac{1}{2} \| \lesssim \|\mathbf{A} - \mathbf{B} \|^\frac{1}{2},
\end{align}
where $\|\cdot \|$ is the operator norm. 
\end{lemma}
\begin{proof}

Since equation (\ref{equ:uniform_continuous}) is homogeneous in scaling, it suffices to prove that for $\|\mathbf{A} - \mathbf{B} \|\leq 1$, there is $   \|\mathbf{A}^\frac{1}{2} - \mathbf{B}^\frac{1}{2} \|\leq C $. Now consider the function
\begin{align}
    f(x) = \int_0^1\left[1-\frac{1}{1+tx}\right ]t^{-3/2}dt.
\end{align}
Let $s=tx$, we get
\begin{align}
    f(x) & = x^{1/2}\int_0^x\frac{s}{1+s}s^{-3/2}ds \nonumber \\
    & = x^{1/2}\left (\int_0^\infty \frac{s}{1+s}s^{-3/2}ds- \int_x^\infty \frac{s}{1+s}s^{-3/2}ds \right ) \nonumber \\
    & = Kx^\frac{1}{2} + g(x),
\end{align}
where $K = \int_0^\infty \frac{s}{1+s}s^{-3/2}ds$, $g(x) = x^{1/2}\int_x^\infty \frac{s}{1+s}s^{-3/2}ds$.
It can be easily seen that $K < \infty$ and $|g(x)|\leq 2$. Thus, we know
\begin{align}
    \| K\mathbf{A}^{1/2}-f(\mathbf{A})\| = \|g(\mathbf{A}) \| \leq 2.
\end{align}
This would imply
\begin{align}
   \|\mathbf{A}^\frac{1}{2} - \mathbf{B}^\frac{1}{2} \| & \leq \|\mathbf{A}^\frac{1}{2} - \frac{f(\mathbf{A})}{K} \| + \|\frac{f(\mathbf{A})}{K} - \frac{f(\mathbf{B})}{K} \|+ \|\frac{f(\mathbf{B})}{K} - \mathbf{B}^\frac{1}{2} \| \nonumber \\
   & \leq \frac{4}{K} + \frac{1}{K}\|f(\mathbf{A}) - f(\mathbf{B}) \|.
\end{align}
Since we know that $\|\mathbf{A} - \mathbf{B} \|\leq 1$, as long as $f$ is Lipschitz in operating norm, we can conclude that $\|\mathbf{A}^\frac{1}{2} - \mathbf{B}^\frac{1}{2} \|$ is bounded, thus leading to uniform continuity. We show $f$ is Lipschitz in the rest of the proof. Note that integration and matrix operation commute. Thus,
\begin{align}
    f(\mathbf{A}) & = \int_0^1(\mathbf{I} - (\mathbf{I}+t\mathbf{A})^{-1})t^{-3/2}dt,\\
 f(\mathbf{A})  -   f(\mathbf{B}) & = \int_0^1 ((\mathbf{I}+t\mathbf{B})^{-1} - (\mathbf{I}+t\mathbf{A})^{-1})t^{-3/2}dt \nonumber \\
 & = \int_0^1 (\mathbf{I}+t\mathbf{B})^{-1}(\mathbf{A} - \mathbf{B})(\mathbf{I}+t\mathbf{A})^{-1}t^{-1/2}dt.
\end{align}
The last equation comes from the identity $\mathbf{X}^{-1} - \mathbf{Y}^{-1} = \mathbf{X}^{-1}(\mathbf{Y} - \mathbf{X})\mathbf{Y}^{-1}$.
Then we can derive
\begin{align}
    \| f(\mathbf{A})  -   f(\mathbf{B})  \| & = \| \int_0^1 (\mathbf{I}+t\mathbf{B})^{-1}(\mathbf{A} - \mathbf{B})(\mathbf{I}+t\mathbf{A})^{-1}t^{-1/2}dt\| \nonumber \\
    & \leq \int_0^1 \|(\mathbf{I}+t\mathbf{B})^{-1}(\mathbf{A} - \mathbf{B})(\mathbf{I}+t\mathbf{A})^{-1}t^{-1/2}\| dt \nonumber \\
    & \leq \int_0^1 \|(\mathbf{I}+t\mathbf{B})^{-1}\|\cdot \|(\mathbf{A} - \mathbf{B})\|\cdot \|(\mathbf{I}+t\mathbf{A})^{-1}\|t^{-1/2} dt \nonumber \\
    & \leq \|(\mathbf{A} - \mathbf{B})\|\int_0^1 t^{-1/2} dt \nonumber \\
    & = 2\|(\mathbf{A} - \mathbf{B})\|.
\end{align}
This finishes the proof of the whole lemma.
\end{proof}
\subsubsection{Proof of main theorem}

Similar to Theorem~\ref{thm:new_convergence}, the moment matching Wasserstein-2 distance $\tW_2$ is an admissible distance of the original Wasserstein-2 distance with parameters $(1,0)$. Define 
\begin{align}
g^*_r = \argmin_{g\in \cG_r} \tilde{W}_{2}(\bbP_{g(Z)}, \bbP_X), \\
g^* = \argmin_{g \in \cG} W_{2}(\bbP_{g(Z)}, \bbP_X).
\end{align}

We have
\begin{align}
    W_2(\bbP_{\tilde g(Z)}, \bbP_X)&\leq W_2(\bbP_{\tilde g(Z)}, \bbP_{g^*(Z)}) + W_2(\bbP_{g^*(Z)}, \bbP_X) \nonumber \\
    & = \tW_2(\bbP_{\tilde g(Z)}, \bbP_{g^*(Z)}) + W_2(\bbP_{g^*(Z)}, \bbP_X) \nonumber \\
    & \leq \tW_2(\bbP_{\tilde g(Z)}, \bbP_X) + \tW_2(\bbP_X, \bbP_{g^*(Z)}) + W_2(\bbP_{g^*(Z)}, \bbP_X) \nonumber \\
    & \leq  \tW_2(\bbP_{\tilde g(Z)}, \bbP_X) + 2 W_2(\bbP_{g^*(Z)}, \bbP_X) \nonumber \\
    & \leq \tW_2(\bbP_{\tilde g(Z)}, \hat \bbP_X^n) + \tW_2(\hat \bbP_X^n, \bbP_X) + 2 W_2(\bbP_{g^*(Z)}, \bbP_X) \nonumber \\
    & \leq \tW_2(\bbP_{g^*_r(Z)}, \hat \bbP_X^n) + \tW_2(\hat \bbP_X^n, \bbP_X) + 2 W_2(\bbP_{g^*(Z)}, \bbP_X) \nonumber \\
    & \leq  \tW_2(\bbP_{g^*_r(Z)}, \bbP_X) + 2\tW_2(\hat \bbP_X^n, \bbP_X)+ 2 W_2(\bbP_{g^*(Z)}, \bbP_X) \nonumber \\
    & = \inf_{g\in \cG_r} \tilde{W}_{2}(\bbP_{g(Z)}, \bbP_X) + 2 \inf_{g \in \cG} W_{2}(\bbP_{g(Z)}, \bbP_X) + 2\tW_2(\hat \bbP_X^n, \bbP_X) \nonumber \\
    & \leq \| \bm{\Sigma}_X^{1/2} - \bm{\Sigma}_r^{1/2} \|_F + 2 \inf_{g \in \cG} W_{2}(\bbP_{g(Z)}, \bbP_X) + 2\tW_2(\hat \bbP_X^n, \bbP_X).
\end{align}
Here $\bm{\Sigma}_X = \mathbb{E}[ XX^T]$ is the covariance matrix of $X$, and $\bm{\Sigma}_r$ is the best rank-$r$ approximation of $\bm{\Sigma}_X$ under Frobenius norm. Next, we bound the term $\tW_2(\hat \bbP_X^n, \bbP_X)$. Denote the covariance matrices for $\bbP_X$ and $ {\hat \bbP_X^n}$ as $\mathbf{K}_X$ and $\hat{\mathbf{K}}_{X}$. The distance $\tW_2^2(\bbP_X, {\hat \bbP_X^n})$ can be written as~\cite{olkin1982distance} 
\begin{align}
\tW_{2}^2(\bbP_X, {\hat \bbP_X^n}) = \operatorname{Tr}(\mathbf{K}_X) + \operatorname{Tr}(\hat{\mathbf{K}}_{X}) - 2\operatorname{Tr}((\mathbf{K}_X^{\frac{1}{2}}\hat{\mathbf{K}}_{X}\mathbf{K}_X^{\frac{1}{2}})^{\frac{1}{2}}) + \| \bm{\mu}_X - \Hat{\bm{\mu}}_X \|^2. 
\end{align}

We know
\begin{align}
\operatorname{Tr}((\mathbf{K}_X^{\frac{1}{2}}\hat{\mathbf{K}}_{X}\mathbf{K}_X^{\frac{1}{2}})^{\frac{1}{2}}) & = \operatorname{Tr}(((\hat{\mathbf{K}}_{X}^{\frac{1}{2}}\mathbf{K}_X^{\frac{1}{2}})^T\hat{\mathbf{K}}_{X}^{\frac{1}{2}}\mathbf{K}_X^{\frac{1}{2}})^{\frac{1}{2}}) \nonumber \\
& \geq \operatorname{Tr}(\hat{\mathbf{K}}_{X}^{\frac{1}{2}}\mathbf{K}_X^{\frac{1}{2}}).
\end{align}
Denote $X = \hat{\mathbf{K}}_{X}^{\frac{1}{2}}\mathbf{K}_X^{\frac{1}{2}}$. The last inequality follows from the fact that for any matrix $\mathbf{X}$ and its singular value decomposition $\mathbf{X} = \mathbf{U}\mathbf{\Sigma} \mathbf{V}$, 
\begin{align}
\operatorname{Tr}((\mathbf{X}^T\mathbf{X})^{\frac{1}{2}}) & = \operatorname{Tr}((\mathbf{V}^T\mathbf{\Sigma}^2 \mathbf{V})^\frac{1}{2}) =  \operatorname{Tr}(\mathbf{V}^T\mathbf{\Sigma} \mathbf{V}) \geq \operatorname{Tr}(\mathbf{U}\mathbf{\Sigma} \mathbf{V}) = \operatorname{Tr}(\mathbf{X}).
\end{align}
Thus we have 
\begin{align}
\tW_{2}^2(\bbP_X, {\hat \bbP_X^n}) & = \operatorname{Tr}(\mathbf{K}_X) + \operatorname{Tr}(\hat{\mathbf{K}}_{X}) - 2\operatorname{Tr}((\mathbf{K}_X^{\frac{1}{2}}\hat{\mathbf{K}}_{X}\mathbf{K}_X^{\frac{1}{2}})^{\frac{1}{2}})+ \| \bm{\mu}_X - \Hat{\bm{\mu}}_X \|^2 \nonumber \\
& \leq \operatorname{Tr}(\mathbf{K}_X) + \operatorname{Tr}(\hat{\mathbf{K}}_{X}) - 2\operatorname{Tr}(\hat{\mathbf{K}}_{X}^{\frac{1}{2}}\mathbf{K}_X^{\frac{1}{2}}) + \| \bm{\mu}_X - \Hat{\bm{\mu}}_X \|^2 \nonumber \\
& = \operatorname{Tr}((\mathbf{K}_X^{\frac{1}{2}} - \hat{\mathbf{K}}_X^{\frac{1}{2}})^2)+ \| \bm{\mu}_X - \Hat{\bm{\mu}}_X \|^2  \nonumber \\
& = \|\mathbf{K}_X^{\frac{1}{2}} - \hat{\mathbf{K}}_X^{\frac{1}{2}}\|_F^2+ \| \bm{\mu}_X - \Hat{\bm{\mu}}_X \|^2  \nonumber \\
& \leq d\|\mathbf{K}_X^{\frac{1}{2}} - \hat{\mathbf{K}}_X^{\frac{1}{2}}\|^2 + \| \bm{\mu}_X - \Hat{\bm{\mu}}_X \|^2. 
%& \leq Cd\|\mathbf{K}_X - \hat{\mathbf{K}}_X\|.
%& \leq C\sqrt{\log(\frac{2}{\delta})}(d\sqrt{\frac{d}{n}})
\end{align}

\begin{comment}
can be written as
\begin{align}
\tW_{2}^2(\bbP_X, {\hat \bbP_X^n}) = \sup_{\mathbf{A}>0}   \operatorname{Tr}((\mathbf{I} - \mathbf{A})\mathbf{K}_X) + \operatorname{Tr}((\mathbf{I} - \mathbf{A}^\dagger) \hat{\mathbf{K}}_{X}).
\end{align}
Below we study the convergence rate of the term $\tW_{2}^2(\bbP_X, {\hat \bbP_X^n})$. %For the convenience of notation, we simply use $W_{2}^2$ to represent $W_{2}'^2$.
%By \cite{david's paper}
We know from~\cite{feizi2017understanding} that
\end{comment}

If the minimum eigenvalue of $K_X$ is lower bounded by a constant $B$, by Ando-Hemmen inequality~\cite[Thoerem 6.2]{higham2008functions}, we have
\begin{align}
    \tW_{2}^2(\bbP_X, {\hat \bbP_X^n}) & \leq d\|\mathbf{K}_X^{\frac{1}{2}} - \hat{\mathbf{K}}_X^{\frac{1}{2}}\|^2+ \| \bm{\mu}_X - \Hat{\bm{\mu}}_X \|^2 \nonumber \\
    & \leq  \frac{d}{(\lambda^{1/2}_{\text{min}}(\mathbf{K}_X)+ \lambda^{1/2}_{\text{min}}(\hat{\mathbf{K}}_X))^2}\| \mathbf{K}_X - \hat{\mathbf{K}}_X\|^2 + \| \bm{\mu}_X - \Hat{\bm{\mu}}_X \|^2\\
    & \leq \frac{d}{B} \| \mathbf{K}_X - \hat{\mathbf{K}}_X\|^2+ \| \bm{\mu}_X - \Hat{\bm{\mu}}_X \|^2.
\end{align}
If the minimum eigenvalue of $K_X$ is not lower bounded, it follows from Lemma~\ref{lem:squareroot} that 
\begin{align}
    \tW_{2}^2(\bbP_X, {\hat \bbP_X^n}) & \leq d\|\mathbf{K}_X^{\frac{1}{2}} - \hat{\mathbf{K}}_X^{\frac{1}{2}}\|^2+ \| \bm{\mu}_X - \Hat{\bm{\mu}}_X \|^2 \nonumber \\
    & \lesssim d\|\mathbf{K}_X - \hat{\mathbf{K}}_X\|+ \| \bm{\mu}_X - \Hat{\bm{\mu}}_X \|^2. 
    \end{align}
\begin{comment}
Assume the true distribution $\bbP_X$ is sub-gaussian, i.e. there exists some constant $L$, such that
\begin{align}
    \bbP(|\left < X, \mathbf{x}\right> |>t)\leq 2e^{-t^2/L^2}
\end{align}
for $t>0$ and $\|\mathbf{x}\|_2=1$.
By the concentration condition for covariance matrix of sub-Gaussian distributions~\cite{vershynin2012close}, we have with probability at least $1-\delta$, 
\begin{align}
\|\mathbf{K}_X - \hat{\mathbf{K}}_X\| \nonumber  \lesssim L^2\sqrt{\log\left (\frac{2}{\delta}\right )}\sqrt{\frac{d}{n}}.
\end{align}

Thus we can conclude that if the minimum eigenvalue of $K_X$ is lower bounded by constant $C$, there exists a constant $D(C)$, such that 
\begin{align}
    \tW_{2}(\bbP_X, {\hat \bbP_X^n}) & \lesssim \frac{1}{\sqrt{C}}  L^2\sqrt{\log\left (\frac{2}{\delta}\right )}dn^{-1/2}.
\end{align}
Without the condition that the minimum eigenvalue of $K_X$ is lower bounded by constant $C$, we can still guarantee
\begin{align}
    \tW_{2}(\bbP_X, {\hat \bbP_X^n}) & \lesssim  \sqrt[1/4]{\log\left (\frac{2}{\delta}\right )} L d^{3/4}n^{-1/4}. 
\end{align}
\end{comment}

\subsection{Proof of Theorem \ref{thm:tightness}}\label{proof.tightness}

We aim at justifying the tightness of Theorem \ref{thm:new_convergence}. Consider the following two sampling models:
\begin{itemize}
    \item Case 1. We draw $n$ i.i.d. samples directly from $\mathcal{N}(0,I)$. Denote the resulting empirical distribution as $E_1$.
    \item Case 2. We draw $N$ elements i.i.d. from $\mathcal{N}(0,I)$. Denote the empirical distribution as $\mathcal{D}_N$. Then, we sample $n$ i.i.d. points from $\mathcal{D}_N$. Denote the resulting empirical distribution as $E_2$. 
\end{itemize}

%By lemma \ref{ }, we know that with probability at least.

%The empirical distribution $\bbQ$ can be viewed as 
%\begin{align}
%    \bbQ(X_1,X_2,\cdots,X_n) = \int_{\Omega} %\mathcal{D}_N(X_1)\mathcal{D}_N(X_2)\cdots\mathcal{D}_N(X_n) d\mathcal{D}_N
%\end{align}

For any distribution $E$ we write $\mathcal{A}^E$ to indicate that algorithm $\mathcal{A}$ is \emph{only} given access to the information contained in $E$. The following lemma states that no algorithm can successfully distinguish these two cases when $n \ll \sqrt{N}$. 

\begin{lemma}
\label{lem:distinguish}
There is an absolute constant $c > 0$ such that the following holds. If $n\leq  c \sqrt{N}$, and $\mathcal{B}$ is any ``distinguishing'' algorithm that aims outputting either $``1"$ or $``2"$. Then, 
\begin{align}
    |\bbP(\mathcal{B}^{E_1} outputs ``1") - \bbP(\mathcal{B}^{E_2} outputs ``1") | \leq 0.01.
\end{align}
\end{lemma}

\begin{proof}
We show that in both cases, with probability at least 0.99, the $c\sqrt{N}$ draws received by $\mathcal{B}$ are a set consisting of $c\sqrt{N}$ i.i.d. samples from $\mathcal{N}(0,I)$. This can be shown using a birthday paradox type argument below. The empirical distribution $E_2$ can be understood as drawing $n$ i.i.d. samples from $\mathcal{D}_N$ with replacement. On the other hand, the empirical distribution $E_1$ can be understood as drawing $n$ i.i.d. samples from $\mathcal{D}_N$ \emph{without} replacement. It follows from the birthday paradox that the probability of sampling all distinct points in $\mathcal{D}_N$ in sampling with replacement is 
\begin{align}
\left (1-\frac{1}{N} \right)\left (1-\frac{2}{N}\right )\left (1-\frac{3}{N}\right ) \cdots \left (1-\frac{n}{N}\right ) 
& > e^{-\frac{2}{N}}e^{-\frac{4}{N}}e^{-\frac{6}{2N}} \cdots e^{-\frac{2n}{N}} \nonumber \\
& = e^{-\frac{n(n-1)}{N}}.
\end{align}
The inequality comes from the fact that $1-x\geq e^{-2x}$ for $x \in [0, \frac{1}{2}]$. Taking $c = 0.01$, we have $e^{-\frac{n(n-1)}{N}}>0.99$, which gives the conclusion of the lemma.
\end{proof}

%\begin{theorem}
%\label{thm:tightness}
%For linear generator family $\cG$, there is no algorithm with the following property: given access to independent points drawn from an unknown distribution $\bbP_X$, algorithm $\mathcal{A}$ makes $o(\sqrt{N})$ draws from $\bbP_X$ and with probability at least $51/100$ outputs a hypothesis $\hat{\bbP}_X$ satisfying 
%\begin{equation}
%\label{equ:tightness}
%W_2(\hat{\bbP}_X, \bbP_X) \leq \inf_{g(\cdot)\in \mathcal{G}} W_2(\bbP_{g(Z)},\bbP_X) + \epsilon(n,d) 
%\end{equation}
%with $\epsilon(n, d) = o(\frac{1}{d^{\frac{3}{2}}}n^{-\frac{3}{d}})$. 
%\end{theorem}

Now we use Lemma~\ref{lem:distinguish} to prove Theorem \ref{thm:tightness}. Suppose that there exists an algorithm $\mathcal{A}$, such that given $n = c \sqrt{N}$ independent points draws from any $\bbP_X \in \mathcal{P}_K$, $\mathcal{A}$ outputs a hypothesis $\hat{\bbP}_X$  with probability at least $51/100$ satisfying  
\begin{equation} \label{eqn.impossiblecontrapositive}
W_2(\hat{\bbP}_X, \bbP_X) \leq \inf_{g(\cdot)\in \mathcal{G}} W_2(\bbP_{g(Z)},\bbP_X) + \epsilon(n,d),% o({d^{1/2}}n^{-3/d}).
\end{equation}
where $\epsilon(n,d)\ll_d n^{-6/d}$.
We will describe how the existence of $\mathcal{A}$ yields a distinguishing algorithm $\mathcal{B}$ violating Lemma \ref{lem:distinguish}.
The algorithm $\mathcal{B}$ works as follows. Given access to i.i.d. draws from distribution $\bbP_X$, it first runs algorithm $\mathcal{A}$, obtaining with probability at least $51/100$ a distribution $\hat{\bbP}_X$ satisfying equation (\ref{eqn.impossiblecontrapositive}), and then computes the value $W_2(\hat{\bbP}_X, \mathcal{N}(0,I))$. If $W_2(\hat{\bbP}_X, \mathcal{N}(0,I)) \leq \epsilon(n, d)$ then it outputs ``1", and otherwise it outputs ``2".

Since $\mathcal{G}$ is the family of full-dimensional affine generators, $\inf_{g(\cdot)\in \mathcal{G}} W_2(\bbP_{g(Z)}, \mathcal{N}(0,I)) = 0$, which shows that in Case 1, we have with probability $51/100$ $W_2(\hat{\bbP}_X, \mathcal{N}(0,I)) \leq \epsilon(n, d)$. Hence,
\begin{align}
    \bbP\left ( \mathcal{B}^{E_1} \text{ outputs }``1" \right ) \geq 0.51.
\end{align}

Now let's consider Case $2$. In this case the true underlying distribution is $\mathcal{D}_N$, and for $N \gg d$~\cite{vershynin2012close}, we have the minimum eigenvalue of the covariance matrix of $\mathcal{D}_N$ is bigger than $1/2$ with probability at least $0.999$. It follows from the concentration of the norm~\cite[Theorem 3.1.1]{vershynin2018high} there exists a universal constant $L>0$ such that if $Y\sim \cN(0,I)$, 
\begin{align}
    \bbP(|\|Y\|-\sqrt{d}|\geq t) \leq 2 \exp\left( - \frac{t^2}{L}  \right ),
\end{align}
which implies that with probability at least $1-\delta$, we have
\begin{align}
    \|Y\| \leq \sqrt{d} + \sqrt{L \ln\left ( \frac{2}{\delta} \right )}. 
\end{align}
Since $\mathcal{D}_N$ consists of $N$ vectors, if we set $\delta = \frac{1}{1000N}$, then with probability at least $0.999$, we know
\begin{align}
    \|X\| \leq \sqrt{d} + \sqrt{L \ln\left ( 2000 N \right )}
\end{align}
almost surely if $X\sim \mathcal{D}_N$. In other words, we have $K =\sqrt{d} + \sqrt{L \ln\left ( 2000 N \right )}$. 

Hence, with probability at least $0.508$, by running algorithm $\mathcal{A}$ we obtain a distribution $\Hat{P}_X$ such that
\begin{align}
W_2(\Hat{\bbP}_X, \mathcal{D}_N) \leq \inf_{g(\cdot)\in \mathcal{G}} W_2(\bbP_{g(Z)},\mathcal{D}_N) + \epsilon(n,d).
\end{align}
It follows from the triangle inequality that
\begin{align}
    W_2(\Hat{\bbP}_X, \cN(0, I)) & \geq W_2(\mathcal{D}_N, \cN(0,I)) - W_2(\Hat{\bbP}_X, \mathcal{D}_N) \\
    & \geq W_2(\mathcal{D}_N, \cN(0,I)) - \inf_{g(\cdot)\in \mathcal{G}} W_2(\bbP_{g(Z)},\mathcal{D}_N) - \epsilon(n,d). 
\end{align}

According to Theorem \ref{thm:naive_convergence_w2}, following the same procedure as (\ref{equ:polydist}), we have with probability at least $0.999$
\begin{align}
    W_2(\mathcal{D}_N, \cN(0,I)) - \inf_{g(\cdot)\in \mathcal{G}} W_2(\bbP_{g(Z)},\mathcal{D}_N) & \gtrsim_{d} N^{-3/d} \\
    & \gtrsim_d n^{-6/d}, 
\end{align}
where in the last step we used the assumption that $n = c\sqrt{N}$. It implies that with probability at least $0.507$
\begin{align}
    W_2(\Hat{\bbP}_X, \cN(0, I)) > \epsilon(n,d)
\end{align}
if $\epsilon(n,d) \ll_{d} n^{-6/d}$ as $n\to \infty$. 

For algorithm $\mathcal{B}$, it outputs ``1" if $W_2(\Hat{\bbP}_X, \cN(0,1)) \leq \epsilon(n,d)$, hence we know 
\begin{align}
    \bbP\left ( \mathcal{B}^{E_2} \text{ outputs }``1" \right ) \leq 0.493,
\end{align}
which implies that 
\begin{align}
    |\bbP(\mathcal{B}^{E_1} outputs ``1") - \bbP(\mathcal{B}^{E_2} outputs ``1") | > 0.01.
\end{align}
This contradicts with Lemma \ref{lem:distinguish} and proves the theorem.

\subsection{Proof of Theorem \ref{thm:sqrt2_lowerbound}}

For infinite sample size, we have access to the real distribution $\bbP_X$, and our GAN estimator would output a Gaussian distribution $\mathcal{N}(\bm{\mu}_X, \bm{\Sigma}_X)$, where $\bm{\mu}_X = \mathbb{E}_{\bbP_X}[X], \bm{\Sigma}_X = \mathbb{E}_{\bbP_X}[XX^T]$ are the mean and covariance of the distribution $\bbP_X$. It suffices to show that for any $\epsilon>0$, we can find some distribution $\bbP_X$ with $\bm{\mu}_X = 0, \inf_{\bm{\mu}, \bm{\Sigma}} W_2(\cN(\bm{\mu}, \bm{\Sigma}),\bbP_X)>0$, such that
\begin{align}\label{eqn.targetoftheorem8}
     W_2(\cN(0, \bm{\Sigma}_X),\bbP_X)\geq (\sqrt{2}-\epsilon) \inf_{\bm{\mu}, \bm{\Sigma}} W_2(\cN(\bm{\mu}, \bm{\Sigma}),\bbP_X).
\end{align}

We construct distribution $\bbP_X$ in the following way. Denote $X = (X_1,X_2,\ldots,X_d)^T$, we assume that the $d$ random variables $\{X_1,X_2,\ldots,X_d\}$ are mutually independent, $X_i \sim \mathcal{N}(0,1)$ for $1\leq i\leq d-1$, and 
\begin{align}
    \bbP_{X_d} = Q_a = \left ( 1- \frac{1}{a^2} \right) \delta_0 + \frac{1}{2a^2} \delta_{-a} + \frac{1}{2a^2} \delta_a,
\end{align}
where $a>1$ is some parameter that will be chosen later. Here $\delta_y$ denotes the point mass distribution that puts probability one for the point $y$. Clearly, $\mathbb{E}_{\bbP_X}[X] = 0, \mathbb{E}_{\bbP_X}[XX^T] = \mathbf{I}$. 

We first show an upper bound on the right hand side of~(\ref{eqn.targetoftheorem8}). Indeed, for the $\bbP_X$ we constructed, 
\begin{align}
  \inf_{\bm{\mu}, \bm{\Sigma}} W_2^2(\cN(\bm{\mu}, \bm{\Sigma}),\bbP_X) \leq \inf_{c\geq 0} W_2^2(\cN(0, c),Q_a),
\end{align}
where we have reduced the set of $\bm{\mu}, \bm{\Sigma}$ we take the infimum over and only considered couplings that are independent across the indices of $X$. Using the tensorization property of $W_2^2$, we know that 
\begin{align}
    W_2(\cN(0, \mathbf{I}),\bbP_X) & = W_2(\cN(0,1), Q_a). 
\end{align}

Hence, it suffices to show that for any $\epsilon>0$, there exists some $a>1$ such that
\begin{align}
    W_2(\cN(0,1), Q_a) \geq (\sqrt{2}-\epsilon) \inf_{c\geq 0} W_2(\cN(0, c),Q_a). 
\end{align}

It follows from the definition of $W_2$ that
\begin{align}
    W_2^2(\cN(0,c),Q_a) & = \inf_{\pi: \pi_X = \cN(0,c), \pi_Y = Q_a} \mathbb{E} (X-Y)^2 \\
    & = \mathbb{E}[X^2] + \mathbb{E}[Y^2] - 2 \sup_{\pi:\pi_X = \cN(0,c), \pi_Y = Q_a} \mathbb{E}_\pi[XY] \\
    & = c + 1 - 2 \sup_{\pi: \pi_Z = \cN(0,1), \pi_Y = Q_a} \mathbb{E}_\pi[\sqrt{c} ZY]. 
\end{align}
Note that $\sup_{\pi: \pi_Z = \cN(0,1), \pi_Y = Q_a} \mathbb{E}_\pi[ ZY]$ is independent of $c$. Denote 
\begin{align*}
    \rho_a = \sup_{\pi: \pi_Z = \cN(0,1), \pi_Y = Q_a} \mathbb{E}_\pi[ZY]. 
\end{align*}
Then, 
\begin{align}
    \inf_{c\geq 0} W_2^2(\cN(0,c),Q_a) & = \inf_{c\geq 0} (c + 1 -2 \sqrt{c} \rho_a) \\
    & = 1-\rho_a^2,
\end{align}
where the infimum achieving $c = \rho_a$. Taking $c = 1$, we know that
\begin{align}
    W_2^2(\cN(0,1), Q_a) & = 2-2\rho_a.
\end{align}
Hence, we have
\begin{align}
    \frac{W_2(\cN(0,1), Q_a) }{\inf_{c\geq 0} W_2(\cN(0, c),Q_a)} & = \sqrt{ \frac{2-2\rho_a}{1-\rho_a^2} } \\
    & = \sqrt{\frac{2}{1+\rho_a}}. 
\end{align}

Hence, it suffices to show that one can make $\rho_a \to 0$ as $a\to \infty$. Define $\tau_a$ such that for $Z\sim \cN(0,1)$, we have $\bbP(|Z|\geq \tau_a) = \frac{1}{a^2}$. It follows from the Gaussian tail abound that
\begin{align}
    \tau_a \lesssim 1 + \sqrt{\ln a}. 
\end{align}

Consider function $f: \mathbb{R}\mapsto \mathbb{R}$ defined as
\begin{align}
    f(z) = \begin{cases} 0 & 0\leq |z|\leq \tau_a \\ a(|z|-\tau_a) & |z|\geq \tau_a \end{cases}
\end{align}

The key observation is the following inequality for any $z,x$ such that $|x|=a$ or $x = 0$:
\begin{align}
    x z \leq f(z) + \tau_a |x|. 
\end{align}
Hence, we can upper bound
\begin{align}
    \rho_a & \leq \mathbb{E}_{Z\sim \cN(0,1)}[f(Z)] + \tau_a \mathbb{E}_{Q_a} [|Y|]. 
\end{align}

We have
\begin{align}
    \mathbb{E}[f(Z)] & = \mathbb{E}[a(|Z|-\tau_a) \mathbbm{1}(|Z|\geq \tau_a)] \\
    & = a \mathbb{E}[|Z| \mathbbm{1}(|Z|\geq \tau_a)] - a \tau_a \bbP(|Z|\geq \tau_a) \\
    & = a \mathbb{E}[|Z| \mathbbm{1}(|Z|\geq \tau_a)] - \frac{\tau_a}{a} \\
    & \leq a (\mathbb{E}[Z^4])^{1/4} (\bbP(|Z|\geq \tau_a))^{3/4} - \frac{\tau_a}{a} \\
    & \leq (\mathbb{E}[Z^4])^{1/4} \frac{a}{a^{3/2}} - \frac{\tau_a}{a}
\end{align}
as well as
\begin{align}
    \mathbb{E}_{Q_a}[|Y|] = \frac{1}{a^2} a = \frac{1}{a}. 
\end{align}
Hence, 
\begin{align}
    \rho_a & \leq (\mathbb{E}[Z^4])^{1/4} \frac{1}{a^{1/2}} - \frac{\tau_a}{a} + \frac{\tau_a}{a} \\
    & = (\mathbb{E}[Z^4])^{1/4} \frac{1}{a^{1/2}},
\end{align}
which can be made arbitrarily close to zero if we take $a\to \infty$.

\subsection{Proof of Theorem \ref{thm:cascaded}}

We verify the two properties of admissible distances.

First, for any $g\in \cG$, since $\cG\subset \cG'$, 
\begin{align}
    L'(\bbP_{g}, \bbP_{g}) & = \min_{g'\in\cG'} L_1'( \bbP_{g'},\bbP_g) + \lambda L_2'(\bbP_g,\bbP_{g'}) \nonumber \\
    & = 0,
\end{align}
since we can take $g' = g$. There is $L'(\bbP_g, \bbP_g) = L_1'(\bbP_g, \bbP_g) = L_2'(\bbP_g, \bbP_g) = 0$.
Since $\cG \subset \cG' $, for $g_1, g_2 \in \cG$, we know that $g_1, g_2 \in \cG'$. Denote
\begin{align}
    g^* = \argmin_{g'\in\cG'}  L_1'(\bbP_{g'},\bbP_{g_1}) + \lambda L_2'(\bbP_{g_2},\bbP_{g'}).
\end{align}
From the fact that $L_1'$ and $L_2'$ are admissible distances, we have
\begin{align}
L'(\bbP_{g_1}, \bbP_{g_2}) - L'(\bbP_{g_2}, \bbP_{g_2})& =  L_1'(\bbP_{g^*},\bbP_{g_1}) + \lambda L_2'(\bbP_{g_2},\bbP_{g^*})\nonumber \\
& \geq \frac{1}{c_1}L_1( \bbP_{g^*},\bbP_{g_1}) - \frac{c_2}{c_1} + \frac{\lambda}{d_1}L_2(\bbP_{g_2},\bbP_{g^*}) - \frac{d_2}{d_1}.
\end{align}
By the non-negativity of the distance function, we can derive
\begin{align}
    \max(c_1, d_1)\cdot L'(\bbP_{g_1}, \bbP_{g_2}) & \geq L_1(\bbP_{g^*},\bbP_{g_1}) + \lambda L_2(\bbP_{g_2},\bbP_{g^*})  - \max(c_1, d_1)\cdot (\frac{c_2}{c_1} +\frac{d_2}{d_1}). \nonumber \\
    &\geq \min_{g'\in\cG'}  L_1(\bbP_{g'},\bbP_{g_1}) + \lambda L_2(\bbP_{g_2},\bbP_{g'})- \max(c_1, d_1)\cdot (\frac{c_2}{c_1} +\frac{d_2}{d_1}) \nonumber \\
    & = L(\bbP_{g_1}, \bbP_{g_2}) - \max(c_1, d_1)\cdot (\frac{c_2}{c_1} +\frac{d_2}{d_1}).
\end{align}
Hence the first property is satisfied.

Now we verify the second property. For any distribution $\bbP_1, \bbP_2$ and any $g\in\cG$, denote
\begin{align}
g^*_1 = \argmin_{g'\in\cG'}  L_1'( \bbP_{g'},\bbP_1) + \lambda L_2'(\bbP_{g},\bbP_{g'}), \\
g^*_2 = \argmin_{g'\in\cG'}  L_1'(\bbP_{g'},\bbP_2) + \lambda L_2'(\bbP_{g},\bbP_{g'}).
\end{align}
Without loss of generality, assume $L'(\bbP_g, \bbP_1)\geq L'(\bbP_g,\bbP_2)$, then there is
\begin{align}
     |L'(\bbP_{g}, \bbP_1) -  L'(\bbP_{g}, \bbP_2)|& = ( L_1'(\bbP_{g^*_1},\bbP_1) + \lambda L_2'(\bbP_g,\bbP_{g^*_1})) - ( L_1'(\bbP_{g^*_2},\bbP_2) + \lambda L_2'(\bbP_g,\bbP_{g^*_2})) \nonumber \\
     & \leq  ( L_1'(\bbP_{g^*_2},\bbP_1) + \lambda L_2'(\bbP_g,\bbP_{g^*_2})) - ( L_1'(\bbP_{g^*_2},\bbP_2) + \lambda L_2'(\bbP_g,\bbP_{g^*_2})) \nonumber \\
     & = |L_1'(\bbP_{g^*_2},\bbP_1) - L_1'( \bbP_{g^*_2},\bbP_2)| \\
     & = L_1''(\bbP_1, \bbP_2). 
\end{align}

\subsection{Proof of Theorem \ref{thm:strengthtesting}}
\label{pf:strengthtesting}

The right-hand side can be seen via the triangle inequality upon noting that $\hat \bbP_1$ and $\hat \bbP_2$ are independent copies of the empirical distribution $\hat \bbP_X^n$. 
\begin{align}
    \bbE[L(\hat \bbP_1, \hat \bbP_2)] & \leq \bbE[L(\hat \bbP_1, \bbP_X)] + \bbE[L(\bbP_X, \hat \bbP_2)]] \nonumber \\ 
    & = 2\bbE[L(\bbP_X, \hat \bbP_X^n)].
\end{align}
The left hand side can be proved via Jensen's inequality and convexity of $L$, 
\begin{align}
     \bbE[L(\hat \bbP_1, \hat \bbP_2)] & \geq \bbE[L( \bbP_X, \hat \bbP_2)] \nonumber \\
     & = \bbE[L(\bbP_X, \hat \bbP_X^n)].
\end{align}

\subsection{Proof of Theorem \ref{thm:equ}}
\label{appen:degenerate}

It suffices to show the following: for any two Gaussian distributions $X\sim \cN(0,A), Y\sim \cN(0,B)$, we have
\begin{align}
    W_2^2(\cN(0,A),\cN(0,B)) & = \inf_{\pi: \pi_X = \cN(0,A),\pi_Y = \cN(0,B)} \mathbb{E}_\pi \| X - Y\|^2 \\
    & = \min\left \{ \operatorname{Tr}(A) + \operatorname{Tr}(B)-2\operatorname{Tr}(\Psi): \begin{bmatrix} A & \Psi \\ \Psi^T & B \end{bmatrix} \geq 0  \right \} \label{eqn.primal} \\
    & = \sup_{S>0} \operatorname{Tr}( (I-S)A + (I-S^{-1})B) \label{eqn.dual1}\\
    & = \sup_{S\geq 0, R(B)\subset R(S)} \operatorname{Tr}( (I-S)A + (I-S^{\dagger})B) \label{eqn.dual2}  \\
    & = \operatorname{Tr}(A) + \operatorname{Tr}(B) - 2 \operatorname{Tr}((A^{1/2} B A^{1/2})^{1/2}). 
\end{align}

We first consider equation~(\ref{eqn.primal}). Consider jointly Gaussian coupling between $X$ and $Y$ such that the joint covariance matrix of the vector $(X;Y)$ is given by the positive semi-definite matrix
\begin{align}
\begin{bmatrix} A & \Psi \\ \Psi^T & B \end{bmatrix}. 
\end{align}
The corresponding value under this specific coupling is given by
\begin{align}
\operatorname{Tr}(A) + \operatorname{Tr}(B)-2\operatorname{Tr}(\Psi),
\end{align}
which by definition is an upper bound on $W_2^2(\cN(0,A), \cN(0,B))$ for any $\Psi$ such that $\begin{bmatrix} A & \Psi \\ \Psi^T & B \end{bmatrix}\geq 0$. 

We now show that for any $S\geq 0, R(B)\subset R(S)$, 
\begin{align}
  \operatorname{Tr}( (I-S)A + (I-S^{\dagger})B)\leq W_2^2(\cN(0,A), \cN(0,B)). 
\end{align}
Indeed, 
\begin{align}
    W_2^2(\cN(0,A),\cN(0,B)) & = \operatorname{Tr}(A) + \operatorname{Tr}(B) - 2 \sup_\pi \mathbb{E}_\pi[X^T Y] \\
    & \geq \operatorname{Tr}(A) + \operatorname{Tr}(B) - 2  (\mathbb{E}[f(X)] + \mathbb{E}[f^*(Y)]),
\end{align}
where we have used Young's inequality $x^Ty \leq f(x) + f^*(y)$ where $f$ is an arbitrary convex function, and $f^*$ is its convex conjugate. Taking $f(x) = \frac{1}{2}x^T S x, S\geq 0$, we have
\begin{align}
f^*(y) & = \sup_{x\in \mathbb{R}^d} \left( x^T y - f(x) \right) \\
& = \begin{cases} \frac{1}{2} y^T S^{\dagger} y & y\in R(S) \\ \infty & y \notin R(S) \end{cases}. 
\end{align}
Evaluating the expectations leads to the desired inequality. 

Now, we have shown that for any $\Psi$ such that $\begin{bmatrix} A & \Psi \\ \Psi^T & B \end{bmatrix}\geq 0$, any $S\geq 0$ such that $R(B) \subset R(S)$, we have
\begin{align}
    \operatorname{Tr}(A) + \operatorname{Tr}(B) - 2\operatorname{Tr}(\Psi) & \geq W_2^2(\mathcal{N}(0,A), \mathcal{N}(0,B)) \nonumber  \\
    & \geq \operatorname{Tr}(A) + \operatorname{Tr}(B) - \operatorname{Tr}(SA + S^\dagger B). \label{eqn.trtargerineq}
\end{align}

The following lemma constructs an explicit coupling. 
\begin{lemma}\label{lemma.psi0construction}
Define
\begin{align}
S_0 & = B^{1/2} \left( (B^{1/2} A B^{1/2})^{1/2} \right)^{\dagger} B^{1/2} \\
\Psi_0 & = A S_0.
\end{align}
Then, 
\begin{align}
    \begin{bmatrix} A & \Psi_0 \\ \Psi_0^T & B \end{bmatrix}\geq 0,
\end{align}
and 
\begin{align}
    \operatorname{Tr}(\Psi_0) & = \operatorname{Tr}((A^{1/2} B A^{1/2})^{1/2}). 
\end{align}
\end{lemma}

\begin{proof}
We first remind ourselves of the Schur complement characterization of PSD matrices. 
\begin{lemma}\cite{olkin1982distance}\label{lemma.schur}
Suppose $A\geq 0, B\geq 0$. Then, 
\begin{align}
\Sigma = \begin{bmatrix} A & \Psi \\ \Psi^T & B \end{bmatrix} \geq 0
\end{align}
if and only if
\begin{align} \label{eqn.psiomega}
\Psi \in \Omega = \{\Psi: R(\Psi^T)\subset R(B), A - \Psi B^{\dagger} \Psi^T \geq 0\}. 
\end{align}
\end{lemma}

Now we check that $\Psi_0$ is a legitimate covariance matrix using Lemma~\ref{lemma.schur}. Indeed, denoting $C = (B^{1/2} A B^{1/2})^{1/2}$, we have $S_0 = B^{1/2} C^\dagger B^{1/2}$, $\Psi_0 = A B^{1/2} C^\dagger B^{1/2}$, and it is clear that $R(\Psi_0^T) \subset R(B)$. It suffices to verify that 
\begin{align}
A - \Psi_0 B^\dagger \Psi_0^T \geq 0. 
\end{align}
It suffices to verify that
\begin{align}
I \geq A^{1/2 }B^{1/2} C^\dagger B^{1/2} B^\dagger B^{1/2} C^\dagger B^{1/2} A^{1/2}. 
\end{align}

Denote the SVD of $A^{1/2} B^{1/2} = U\Sigma V^T$, we have $B^{1/2} A^{1/2} = V\Sigma U^T$, $C = V \Sigma V^T$. Building on the observation that $B^{1/2} B^\dagger B^{1/2} \leq I$, we have
\begin{align}
A^{1/2 }B^{1/2} C^\dagger B^{1/2} B^\dagger B^{1/2} C^\dagger B^{1/2} A^{1/2} & \leq A^{1/2 }B^{1/2} C^\dagger C^\dagger B^{1/2} A^{1/2} \\
&  = U\Sigma V^T V \Sigma^\dagger V^T V \Sigma^\dagger V^T V \Sigma U^T \\
& \leq U \Sigma \Sigma^\dagger \Sigma^\dagger \Sigma U^T \\
& \leq I,
\end{align}
where in the last step we used the inequality $\Sigma \Sigma^\dagger \Sigma^\dagger \Sigma \leq I$. 
\end{proof}

\begin{lemma}\cite{anderson1978extremal} \label{thm.dualpd}
Suppose $A,B$ are PSD matrices of the same size. Then, 
\begin{align}
\inf_{S>0} \operatorname{Tr}(AS + BS^{-1}) = 2 \operatorname{Tr}((A^{1/2} B A^{1/2})^{1/2}). 
\end{align}
The infimum is achievable for some $S>0$ if and only if $r(A) = r(B) = r(A^{1/2} B A^{1/2})$. Here $r(\cdot)$ is the rank of matrix.
\end{lemma}

Now we can prove the main results. We have the following chain of inequalities and equalities as a consequence of (\ref{eqn.trtargerineq}), Lemma~\ref{lemma.psi0construction}, and Lemma~\ref{thm.dualpd}. 
\begin{align}
    2 \operatorname{Tr}((A^{1/2} B A^{1/2})^{1/2}) & = \inf_{S>0} \operatorname{Tr}(AS + BS^{-1})\\
    & \geq \inf_{R(B) \subset R(S), S\geq 0} \operatorname{Tr} (A S  + B S^{\dagger} )  \\
    & \geq \sup_{\Psi: \begin{bmatrix} A & \Psi \\ \Psi^T & B \end{bmatrix}\geq 0} 2\operatorname{Tr}(\Psi) \\
    & \geq 2 \operatorname{Tr}(\Psi_0) \\
    & = 2 \operatorname{Tr}((A^{1/2} B A^{1/2})^{1/2}).  
\end{align}

\subsection{Proof of Theorem \ref{thm:instability}}
\begin{comment}
The proof of instability relies on the following lemma.
\begin{lemma}[Chetaev instability theorem~\cite{vidyasagar2002nonlinear}]\label{lem.chetaev}
For the system $\frac{dx}{dt}=f(x)$, if there exists an equilibrium point $x^*$ and a continuously differentiable function $V(x)$ such that
\begin{enumerate}
    \item $x^*$ is a boundary point of the set $G = \{x|V(x)>0 \}$.
    \item there exists a neighborhood $U$ of $x^*$ such that $\frac{dV}{dt}(x)>0$ for all $x\in G \cap U$.
\end{enumerate}
then $x^*$ is a locally unstable equilibrium point of the system.
\end{lemma}
Then we prove the main theorem.
\end{comment}
 The optimal solutions in~(\ref{optimization1}) when $r=1$ can be written in the following form:
\begin{align}
K & = Q \Sigma Q^T \\
U^* & = \lambda_1 v_1 v_1^T \\
A^* & = v_1 v_1^T,
\end{align}
where $Q = [v_1,v_2,\ldots,v_d]$, and $\Sigma = \text{diag}(\lambda_1,\lambda_2,\ldots,\lambda_d)$ ranks the eigenvalues of $K$ in a descending order. In other words, $U^*$ is the $1$-PCA of $K$, and $A^*$ shares the linear subspace as $U^*$ with the eigenvalue being one. %It naturally leads to the following \emph{parameter sharing} GAN architecture. 

We show the optimization problem (\ref{optimization1}) is in general not locally asymptotically stable with $r(U) \leq 1$. We parametrize  $U = vv^T$, i.e. 
\begin{equation}
\label{equ:1dopt}
\min_{v} \sup_{A>0} \operatorname{Tr}( (I-A)K + (I-A^{-1})vv^T).
\end{equation}
Taking the derivative of $A$ and $v$, the gradient flow differential equation is given by
\begin{align}
\label{equ:gradientflow}
\frac{d}{dt} \begin{bmatrix}
A \\v 
\end{bmatrix} & = \begin{bmatrix}
-K + A^{-1}UA^{-1}\\-2(I-A^{-1})v
\end{bmatrix}. 
\end{align}

\begin{comment}
We construct Lyapunov function as

\begin{equation}
V(A,U) = \frac{1}{2} \|A -A_*\|_2^2 + \frac{1}{2} \|v -v_*\|_2^2
\end{equation}

Here $(A_*,v_*)$ is the optimal solution for optimization problem (\ref{equ:1dopt}).

It is obvious that 
\begin{equation}
\frac{\partial V}{\partial A} = A - A_*, 
\frac{\partial V}{\partial v} = v - v_*
\end{equation}

Without loss of generality, assume $A$ is symmetric. By gradient flow in equation (\ref{equ:gradientflow}), we know $A$ will always maintain symmetric since both $A$ and $U$ are symmetric.

Thus we have
\begin{align}
\frac{d}{dt} V & =  \operatorname{Tr}((A-A_*)(-K+A^{-1}UA^{-1})) \nonumber \\
& \quad - \operatorname{Tr}((v-v_*)^T(I-A^{-1})v)\nonumber \\
& = -\operatorname{Tr}(AK) + \operatorname{Tr}(UA^{-1}) + \operatorname{Tr}(A_*K) - \operatorname{Tr}(A_*A^{-1}UA^{-1}) \nonumber \\
\label{equ:lyapunov_instability}
& \quad  - \operatorname{Tr}((v-v_*)^T(I-A^{-1})v)
\end{align}
\end{comment}

Consider the special case when 
\begin{align}
K = \begin{bmatrix} 1 & 0 \\ 0 & 0 \end{bmatrix}. 
\end{align}
We can derive the optimal solution as 
\begin{align}
A^* = \begin{bmatrix} 1 & 0 \\ 0 & 0 \end{bmatrix}, 
v^* = \begin{bmatrix} 1 \\ 0 \end{bmatrix}
\end{align}
Note that $A^*$ is not inside the feasible region of positive definite matrices, but at the boundary of that. However, one can verify that the gradient flow is $0$ for the optimal solution if at the boundary we interpret the matrix inverse as matrix pseudo-inverse. We are interested in whether the system will leave the optimal point after being slightly perturbed.
Consider $A, U$ is searching within
\begin{align}
A = \begin{bmatrix} a_{11} & a_{12}  \\ a_{21}  & a_{22} \end{bmatrix}, 
v = \begin{bmatrix} v_1 \\ v_2\end{bmatrix}
\end{align}

Note that $A$ must be symmetric, thus we also have $a_{12}=a_{21}$. Denote $x = [vec(A), v]$ as the system state vector, and $x^* = [vec(A^*), v^*]$ as the equilibrium point of system defined in equation (\ref{equ:gradientflow}). We can compute the derivative of $a_{11}$, $a_{12}$, $a_{22}$ and $v$ as follows.
\begin{align}
    &\frac{d a_{11}}{dt} = 
    \left (\frac{a_{22}v_1-a_{12}v_2}{a_{11}a_{22}-a_{12}^2}\right )^2 -1 \\
    &\frac{d a_{12}}{dt} = \frac{(a_{22}v_{1}-a_{12}v_{2})(a_{11}v_2-a_{12}v_1)}{(a_{11}a_{22}-a_{12}^2)^2} \\
    &\frac{d a_{22}}{dt} = 
    \left (\frac{a_{11}v_2-a_{12}v_1}{a_{11}a_{22}-a_{12}^2}\right )^2  \label{equ:da22}\\
     &\frac{d v_1}{dt} = 2\frac{a_{22}v_1-a_{12}v_2}{a_{11}a_{22}-a_{12}^2} -2v_1\\
    &\frac{d v_2}{dt} = 2\frac{a_{11}v_2-a_{12}v_1}{a_{11}a_{22}-a_{12}^2} -2v_2
\end{align}
Thus for any $\delta>0$, we can intialize $v=(1,0)$, $a_{11}=1, a_{12}=0, a_{22}=\delta/2$. Under this case,  all the derivatives equal $0$.
Then $\|x(0)-x^*\|_2=\delta/2<\delta$, but $\lim_{t\rightarrow \infty}\|x(t)-x^* \|_2=\delta/2\neq0$. Since $\delta$ can be arbitrarily small, it violates the definition of local asymptotic stability in Appendix $\ref{appendix.def}$. We can  conclude that the system is not locally asymptotically stable.

\subsection{Proof of Theorem \ref{thm:stability}}

We have assumed that $r = 1$. We assume without loss of generality that $K\neq 0$. We parametrize $U = b v v^T, A = \lambda v v^T, \|v\|_2 = 1, b>0, \lambda>0$. In this case, the objective function reduces to
\begin{align}
\operatorname{Tr}(K) - \lambda v^T K v + b - \frac{b}{\lambda}. 
\end{align}

Denote 
\begin{align}
f(v,b,\lambda) = \operatorname{Tr}(K) - \lambda v^T K v + b - \frac{b}{\lambda},
\end{align}
we now analyze the alternating gradient descent flow for this problem. The gradient flow differential equation is given by
\begin{align}
\frac{d}{dt} \begin{bmatrix}
v \\ b \\ \lambda
\end{bmatrix} & = \begin{bmatrix}
-\frac{\partial f}{\partial v} \\ -\frac{\partial f}{\partial b} \\ \frac{\partial f}{\partial \lambda} 
\end{bmatrix}. 
\end{align}

Note that we have constrained $\|v\|_2 = 1$, in other words, $v$ lies in the unit sphere. Suppose the unconstrained derivative of $f$ at $v$ is $g$, then the gradient of $f$ constrained on the Stiefel manifold~\cite{tagare2011notes} $\|v\|_2 = 1$ is $g - v g^T v$. Indeed, the direction $g - v g^T v$ is orthogonal to $v$. 

Concrete computation shows when $b>0, \lambda>0$,
\begin{align}
\frac{d}{dt} \begin{bmatrix}
v \\ b \\ \lambda
\end{bmatrix} & = \begin{bmatrix}
2\lambda Kv - 2\lambda (v^T Kv)v 
\\ \frac{1}{\lambda}-1 \\ -v^T Kv + \frac{b}{\lambda^2}. 
\end{bmatrix}. 
\end{align}
when $b = 0$, $\frac{db}{dt} = \left( \frac{1}{\lambda}-1\right) \vee 0$, and when $\lambda = 0$, $\frac{d\lambda}{dt} = \left( -v^T K v +\frac{b}{\lambda^2} \right) \vee 0$. 

The saddle point we aim for is 
\begin{align}
\begin{bmatrix}
v  \\ b \\ \lambda
\end{bmatrix} & = \begin{bmatrix}
\pm v_1 \\ v_1^T K v_1 \\ 1,
\end{bmatrix} =  \begin{bmatrix}
\pm v_1 \\ \lambda_1 \\ 1,
\end{bmatrix} 
\end{align}
Here $v_1$ is an eigenvector corresponding to the largest eigenvalue of $K$, and $\lambda_1$ is the largest eigenvalue of $K$, which is assumed to be positive. We also assume that at the beginning of the optimization, we initialize $v$ such that $v_1^T v \neq 0$. 
 
Hence, we construct the Lyapunov function $L(v,b,\lambda)$ as
\begin{align}
L(v,b,\lambda) & =  \left( \frac{1}{8} + \lambda_1^2 \right) \ln \left( \frac{1}{|v^T v_1|} \right)    + \frac{1}{2}(b-\lambda_1)^2 + \frac{1}{2}(\lambda-1)^2 \nonumber \\
& \quad + \frac{1}{8}\left( \frac{1}{\lambda}-1 \right)^2 + \frac{1}{8}\left( \frac{b}{\lambda^2} - v^T K v \right)^2. 
\end{align}
Note that if $\lambda \to 0^+$, $L \to \infty$, so if we are able to prove the descent property of $L$ then $\lambda$ will not hit zero in the trajectory. 

Solving for the gradient of $L$ and denoting $v(t),b(t),\lambda(t)$ as $v,b,\lambda$ when $b>0, \lambda>0$, we have
\begin{align}
& \frac{d}{dt}L(v(t), b(t), \lambda(t)) \nonumber \\
&\quad = - \left( \frac{1}{8} + \lambda_1^2 \right)  \frac{v_1^T}{ v^T v_1} (2\lambda K v - 2\lambda v^T K v v ) \nonumber \\
& \qquad + (b - \lambda_1 )\left(\frac{1}{\lambda}-1\right) \nonumber \\
& \qquad +  (\lambda-1)\left( -v^T K v + \frac{b}{\lambda^2} \right) \nonumber \\
& \qquad + \frac{1}{4} \left( \frac{1}{\lambda}-1\right) \frac{-1}{\lambda^2} \left( \frac{b}{\lambda^2} - v^T K v \right) \nonumber \\
& \qquad + \frac{1}{4}\left( \frac{b}{\lambda^2}-v^T K v \right) \frac{1}{\lambda^2} \left( \frac{1}{\lambda}-1\right) + \frac{1}{4}\left(v^T K v - \frac{b}{\lambda^2} \right) 2 v^T K \left( 2\lambda K v - 2 \lambda (v^T K v) v \right) \nonumber \\
& \qquad + \frac{1}{4}\left( \frac{b}{\lambda^2} - v^T K v \right) \frac{-2b}{\lambda^3} \left( \frac{b}{\lambda^2} - v^T K v \right). 
\end{align}
The case of $b = 0, \lambda>0$ can also be solved and it is straightforward to show that $\frac{dL}{dt}\leq 0$ for any $v$ such that $v^T v = 1, b = 0, \lambda>0$.

Taking the part corresponding to $b$ out, we have
\begin{align}
& b\left( \frac{1}{\lambda}-1 + \frac{1}{\lambda} - \frac{1}{\lambda^2} \right) \nonumber  -\frac{b}{4\lambda^2} 2 v^T K \left( 2\lambda K v - 2\lambda (v^T K v) v \right) - \frac{b}{2\lambda^3} \left( \frac{b}{\lambda^2} - v^T K v \right)^2 \\
& = -b\left(1 - \frac{1}{\lambda} \right)^2  - \frac{b}{\lambda} \left( v^T K^2 v - (v^T K v)^2\right)  - \frac{b}{2\lambda^3} \left( \frac{b}{\lambda^2} - v^T K v \right)^2 \\
& \leq 0, 
\end{align}
where equality holds if and only if $\lambda = 1, v^T K^2 v = (v^T K v)^2, b = v^T K v$, here we have used the inequality that $v^T K^2 v \geq (v^T K v)^2$.

We decompose $v$ as 
\begin{align}
v = \sum_{i = 1}^d a_i v_i,
\end{align}
where $\{v_1,v_2,\ldots,v_d\}$ are the eigenvectors of matrix $K$. Here $\sum_i a_i^2 = 1$. 

The terms that remain in the derivative of $L$ is
\begin{align}
&\quad - \left( \frac{1}{8} + \lambda_1^2 \right)  \frac{v_1^T}{ v^T v_1} (2\lambda K v - 2\lambda v^T K v v )  + \lambda_1 \left( 1- \frac{1}{\lambda} \right) \nonumber \\
& \qquad + (1-\lambda) v^T K v + \frac{1}{4}\left(v^T K v  \right) 2 v^T K \left( 2\lambda K v - 2 \lambda (v^T K v) v \right),
\end{align}
which is equivalent to
\begin{align}
& -\left( \frac{1}{8} + \lambda_1^2 \right) 2\lambda \left( \lambda_1 - \sum_i \lambda_i a_i^2 \right)  + \lambda_1 - \frac{\lambda_1}{\lambda} + \sum_i \lambda_i a_i^2 - \lambda \sum_i \lambda_i a_i^2 \nonumber \\
& \qquad + \lambda \left( \sum_i \lambda_i a_i^2 \right) \left( \sum_i \lambda_i^2 a_i^2 - (\sum_i \lambda_i a_i^2) \right). 
\end{align}
Hence, it suffices to show that
\begin{align}
& \lambda_1 + \sum_i \lambda_i a_i^2 \nonumber \\
& \quad \leq \frac{\lambda_1}{\lambda} + \lambda \left( \sum_i \lambda_i a_i^2 + 2 \left( \frac{1}{8} + \lambda_1^2 \right)(\lambda_1 - \sum_i \lambda_i a_i^2) -  \left( \sum_i \lambda_i a_i^2 \right) \left( \sum_i \lambda_i^2 a_i^2 - (\sum_i \lambda_i a_i^2) \right) \right). 
\end{align}

To simplify notation, we introduce
\begin{align}
x & = \sum_i \lambda_i a_i^2 \\
y & = \sqrt{\sum_i \lambda_i^2 a_i^2},
\end{align}
clearly we have
\begin{align}
\lambda_1 \geq y \geq x \geq 0. 
\end{align}
The inequality that we need to prove is
\begin{align}
\lambda_1 + x \leq \frac{\lambda_1}{\lambda} + \lambda \left( x + 2\left( \frac{1}{8} + \lambda_1^2 \right)(\lambda_1-x) - x(y^2-x^2) \right). 
\end{align}

We note that 
\begin{align}
& x + 2\left( \frac{1}{8} + \lambda_1^2 \right)(\lambda_1-x) - x(y^2-x^2) \nonumber \\
& \geq x+ 2\left(\frac{1}{8} + \lambda_1^2 \right) (\lambda_1-x) - x(y+x)(\lambda_1-x) \\
& \geq x + (\lambda_1-x) \left( \frac{1}{4} + 2\lambda_1^2 -x(y+x) \right) \\
& \geq x + (\lambda_1-x) \left( \frac{1}{4} + 2\lambda_1^2 - \lambda_1(2\lambda_1) \right) \\
& \geq x + \frac{1}{4}(\lambda_1-x) \\
& >0,
\end{align}
since we have assumed that $\lambda_1>0$. 

Minimizing the right hand side with respect to $\lambda$ shows that it suffices to show that
\begin{align}
(\lambda_1+x)^2 \leq 4\lambda_1 \left( x + 2\left( \frac{1}{8} + \lambda_1^2 \right)(\lambda_1-x) - x(y^2-x^2) \right),
\end{align}
which is equivalent to
\begin{align}
(\lambda_1-x)^2 & \leq 8\lambda_1 \left( \frac{1}{8} + \lambda_1^2 \right)(\lambda_1-x) - 4\lambda_1 x(y+x)(y-x). 
\end{align}
Since $\lambda_1\geq y$, it suffices to show that
\begin{align}
(\lambda_1-x)^2 & \leq 8\lambda_1 \left( \frac{1}{8} + \lambda_1^2 \right)(\lambda_1-x) - 4\lambda_1 x(y+x)(\lambda_1-x). 
\end{align}
Clearly, this inequality holds when $\lambda_1 = x$. If $\lambda_1>x$, it suffices to show that
\begin{align}
\lambda_1-x& \leq \lambda_1 (1 + 8 \lambda_1^2) -4\lambda_1 x(y+x). 
\end{align}
It is true since 
\begin{align}
\lambda_1 (1 + 8 \lambda_1^2) -4\lambda_1 x(y+x) & \geq \lambda_1 + 8 \lambda_1^3 - 4 \lambda_1 \lambda_1(\lambda_1 + \lambda_1) \\
& = \lambda_1 \\
& \geq \lambda_1-x. 
\end{align}

The above chain of inequalities holds equality if and only if $\lambda_1 = x$. 

Hence, we have showed that the derivative of $L$ is strictly negative in the regime $v^T v = 1, \lambda>0, b\geq 0$ except when $\lambda = 1, v^T K v = \lambda_1, b \in \{0, \lambda_1\}$. The largest invariant set of $(v,b,\lambda)$ satisfying the condition above is $\{(v,b,\lambda): v^T v = 1, \lambda = 1, b = \lambda_1, v^T K v = \lambda_1\}$. 

To apply the Lasalle Theorem~\cite{lasalle1960some}, it suffices to show that over the trajectory of $(v,b,\lambda)$ is constrained in a compact set, and over that set $\frac{dL}{dt}\leq 0$ and $L(v,b,\lambda)>0$ except for $\lambda = 1, b = v^T K v = \lambda_1$. Indeed, the compact set can be chosen to be
\begin{align}
\{(v,b,\lambda): v^T v = 1, b\in [0,b_c], \lambda \in [\lambda_{c_1}, \lambda_{c_2}]\}.
\end{align}
Indeed, the constraint $b_c$ comes from the fact that $L\to \infty$ as $b\to \infty$, hence the trajectory of $b$ is bounded by a constant $b_c$ which depends on the initial condition. The constraint $\lambda_{c_1}, \lambda_{c_2}$ comes from similar considerations since $L\to \infty$ whenever $\lambda \to 0^+$ or $\lambda \to \infty$. 

Furthermore, one can check that $L$ is radially unbounded, i.e. $L(v,b,\lambda)\rightarrow \infty$ as $\|(v,b,\lambda)\| \rightarrow \infty$. 
Invoking~\cite[Corollary 5.2]{bhat2003nontangency}, we know that the dynamics converge to a single point in the set $\{(v,b,\lambda): v^T v = 1, \lambda = 1, b = \lambda_1, v^T K v = \lambda_1\}$ since every point in this set is globally asymptotically stable. 

\end{document}